\documentclass{article} % For LaTeX2e
\usepackage{iclr2022_conference_hack,times}

\usepackage{amsmath,amsfonts,bm}

\def\eqref#1{equation~\ref{#1}}
\def\ceil#1{\lceil #1 \rceil}

\def\1{\bm{1}}

\def\mA{{\bm{A}}}
\def\mB{{\bm{B}}}

\def\mD{{\bm{D}}}
\def\mE{{\bm{E}}}

\def\mH{{\bm{H}}}
\def\mI{{\bm{I}}}

\def\mL{{\bm{L}}}

\DeclareMathAlphabet{\mathsfit}{\encodingdefault}{\sfdefault}{m}{sl}
\SetMathAlphabet{\mathsfit}{bold}{\encodingdefault}{\sfdefault}{bx}{n}

\let\ab\allowbreak

\usepackage{bbm}

\newcommand{\algname}[1]{{\color{black}\small \sf #1}}
\usepackage{nicefrac}

\usepackage{hyperref}
\usepackage{url}
\usepackage{multirow}
\usepackage{graphicx}

\usepackage{colortbl}
\definecolor{bgcolor}{rgb}{0.76,0.88,0.50}

\usepackage{subfigure}
\usepackage[colorinlistoftodos,bordercolor=orange,backgroundcolor=orange!20,linecolor=orange,textsize=scriptsize]{todonotes}
 
\newcommand{\rafal}[1]{\todo[inline]{\textbf{Rafa\l: }#1}}

\newcommand{\ts}{\textsuperscript}

\newcommand{\Lpm}{L_{\pm}}
\newcommand{\mLpm}{\mL_{\pm}}
\newcommand{\ProbArg}[1]{\Prob\left(#1\right)}

\usepackage{multirow}
\usepackage[cp1250]{inputenc}
\usepackage[T1]{fontenc}
\usepackage{calligra}
\usepackage{layout}

\usepackage{amsmath}
\usepackage{amsthm}
\usepackage{amssymb}
\usepackage{amsfonts}
\usepackage{mathtools}

\usepackage{url}
\usepackage{color}
\usepackage{graphicx}
\usepackage{algorithmic}
\usepackage{algorithm}
\usepackage{verbatim}
\usepackage{thmtools} 
\usepackage{thm-restate}
\newcommand{\norm}[1]{\left\| #1 \right\|}
\newcommand{\sqnorm}[1]{\left\| #1 \right\|^2}
\newcommand{\inp}[2]{\left\langle#1,#2\right\rangle} % inner product

\newcommand{\cC}{\mathcal{C}}

\newcommand{\cL}{\mathcal{L}}

\newcommand{\cN}{\mathcal{N}}
\newcommand{\cO}{\mathcal{O}}
\newcommand{\cQ}{\mathcal{Q}}

\newcommand{\del}[1]{}

\newcommand{\R}{\mathbb{R}} % reals
\newcommand{\eqdef}{:=}

\newcommand{\Prob}{\mathbf{Prob}} % probability

\newcommand{\Exp}[1]{{\rm E}\left[#1\right]}
\newcommand{\ExpCond}[2]{{\rm E}\left[\left.#1\right\vert#2\right]}

\DeclareMathOperator{\Diag}{Diag}       % Diag(v) = diagonal matrix with v_i on the diagonal
\newtheorem{assumption}{Assumption}
\newtheorem{lemma}{Lemma}

\newtheorem{theorem}{Theorem}

\newtheorem{example}{Example}

\newtheorem{definition}{Definition}
\theoremstyle{plain}

\theoremstyle{definition}

\newcommand{\nfixt}{\nabla f_i(x^t)}
\newcommand{\nfixtpo}{\nabla f_i(x^{t+1})}
\newcommand{\sumin}{\sum\limits_{i=1}^n}
\newcommand{\suminn}{\frac{1}{n}\sum\limits_{i=1}^n}

\newcommand{\finf}{f^{\text{inf}}}

\title{Permutation Compressors for Provably Faster Distributed Nonconvex Optimization}

\author{Rafa\l{} Szlendak\thanks{The work of Rafa\l{} Szlendak was performed during a Summer research internship in the {\em Optimization and Machine Learning Lab} at KAUST led by Peter Richt\'{a}rik. Rafa\l{} Szlendak is an undergraduate student at the University of Warwick, United Kingdom.} \\
KAUST   \\
Saudi Arabia \\
\And
Alexander Tyurin\\
KAUST \\
Saudi Arabia \\
\And
Peter Richt\'{a}rik  \\
KAUST \\
Saudi Arabia \\
}

\iclrfinalcopy % Uncomment for camera-ready version, but NOT for submission.
\begin{document}

%%%%%%%%%

\maketitle

\begin{abstract}
We study the \algname{MARINA} method of \cite{MARINA} -- the current state-of-the-art distributed non-convex optimization method in terms of theoretical communication complexity. Theoretical superiority of this method can be largely attributed to two sources: the use of a carefully engineered biased stochastic gradient estimator, which leads to a reduction in the number of communication rounds, and  the reliance on
 {\em independent} stochastic communication compression operators, which leads to a reduction in the number of  transmitted bits within each communication round. In this paper we  i) extend the theory of \algname{MARINA} to support a much wider class of potentially {\em correlated} compressors, extending the reach of the method beyond the classical independent compressors setting,  ii) show that a new quantity, for which we coin the name {\em Hessian variance}, allows us to significantly refine the original analysis of \algname{MARINA} without any additional assumptions, and iii) identify a special class of correlated compressors based on the idea of {\em random  permutations}, for which we coin the term Perm$K$, the use of which leads to  $O(\sqrt{n})$ (resp.\ $O(1 + d/\sqrt{n})$) improvement in the theoretical communication complexity of \algname{MARINA} in the low Hessian variance regime when $d\geq n$ (resp.\ $d \leq n$), where $n$ is the number of workers and $d$ is the number of parameters describing the model we are learning. We corroborate our theoretical results with carefully engineered synthetic experiments with minimizing the average of nonconvex quadratics, and on autoencoder training with the MNIST dataset.
\end{abstract}

\section{Introduction}

%Distributed nonconvex optimization is the primary tool behind training modern supervised machine learning models. 

%The success of modern deep learning largely relies on the ability to train elaborate networks described by a very large number of parameters using large quantities of training data.  Such training is performed in a distributed fashion using distributed nonconvex optimization methods. Since communication cost forms a key bottleneck in distributed computing, efficient methods need to be designed with communication efficiency in mind. 

The practice of modern supervised learning relies on highly sophisticated, high dimensional and data hungry deep neural network models \citep{Transformer2017, GPT3} which need to be trained on specialized hardware providing fast distributed and parallel processing. Training of such models is typically performed using elaborate systems relying on specialized distributed stochastic gradient methods \citep{MARINA}. In distributed learning, communication among the compute nodes is typically a key bottleneck of the training system, and for this reason it is necessary to employ strategies alleviating the communication burden.

\subsection{The problem and assumptions} 

Motivated by the need to design provably communication efficient {\em distributed stochastic gradient methods} in the {\em nonconvex} regime, in this paper we consider the optimization problem
\begin{equation} \label{eq:main} \textstyle \min \limits_{x\in \R^d} \left[ f(x)\eqdef \frac{1}{n} \sum \limits_{i=1}^n f_i(x)\right],\end{equation}
where $n$ is the number of workers/machines/nodes/devices working in parallel, and $f_i:\R^d\to \R$ is a (potentially {\em nonconvex}) function representing the loss of the model parameterized by weights $x\in \R^d$ on training data stored on machine $i$. 

While we do {\em  not} assume the functions $\{f_i\}$ to be convex, we rely on their differentiability, and on the well-posedness of problem (\ref{eq:main}):
\begin{assumption} \label{ass:diff} The functions $f_1,\dots,f_n: \R^d\to \R$ are differentiable. Moreover, $f$ is lower bounded, i.e., there exists $f^{\inf} \in \R$ such that $ f(x) \geq f^{\inf}$ for all $x \in \R^d$. %Wlog we can set $f^{\inf}=\inf f$.
\end{assumption}

%Besides Assumption~\ref{ass:diff}, the analysis of \algname{MARINA} by \citet{MARINA}  depends on  the additional assumption that there exists a constant $L_+>0$ such that $\frac{1}{n} \sum_{i=1}^n \norm{\nabla f_i(x) - \nabla f_i(y)}^2 \leq L^2_{+} \norm{x-y}^2$, for all $x,y\in \R^d.$ While this is a somewhat stronger assumption than mere $L_{-}$-Lipschitz continuity of the gradient  of $f$ (the latter follows from the former by Jensen's inequality and we have $L_{-}\leq L_+$), it is weaker than $L_i$-Lipschitz continuity of the gradient of the functions $f_i$ (the former follows from the latter with $L_+^2 \leq \frac{1}{n}\sum_i L_i^2$). 

We are interested in  finding an approximately stationary point of the nonconvex problem (\ref{eq:main}). That is, we wish to identify a (random) vector $\hat{x}\in \R^d$ such that
\begin{equation}\label{eq:stationary}\Exp{\norm{ \nabla f(\hat{x})}^2} \leq \varepsilon\end{equation} while ensuring that the  volume of communication between the $n$ workers and the server is as small as possible. Without the lower boundedness assumption there might not be a point with a small gradient (e.g., think of $f$ being linear), which would render problem (\ref{eq:stationary}) unsolvable. However, lower boundedness ensures that the problem is well posed. Besides Assumption~\ref{ass:diff}, we rely on the following smoothness assumption: 
\begin{assumption} \label{as:L_+}There exists a constant $L_+>0$ such that $\frac{1}{n} \sum_{i=1}^n \norm{\nabla f_i(x) - \nabla f_i(y)}^2 \leq L^2_{+} \norm{x-y}^2$ for all $x,y\in \R^d.$ To avoid ambiguity, let $L_+$ be the smallest such number.
\end{assumption}

While this is a somewhat stronger assumption than mere $L_{-}$-Lipschitz continuity of the gradient  of $f$ (the latter follows from the former by Jensen's inequality and we have $L_{-}\leq L_+$), it is weaker than $L_i$-Lipschitz continuity of the gradient of the functions $f_i$ (the former follows from the latter with $L_+^2 \leq \frac{1}{n}\sum_i L_i^2$). So, this is still a reasonably weak assumption.

% designing a provably communication efficient (stochastic) first order method for
% 

%We are interested in solving (\ref{eq:main}) in a distributed environment with $n$ workers and one server whose role is to coordinate communication among the workers. In such an environment, communication can be much more expensive than computation. We consider iterative methods which operate as follows: all workers compute a certain vector in $\R^d$, which is then communicated to the server. The server then performs aggregation by averaging, and broadcasts the aggregated message back to the workers. Further, we consider the regime where the worker-to-server communication is the bottleneck of the system. While this is a standard model used in many prior works, we include a brief motivation and background in Appendix~\ref{section:comm_model}.

\subsection{A brief overview of the state of the art} \label{sec:overview}
To the best of our knowledge, the state-of-the-art distributed method for finding a point $\hat{x}$ satisfying (\ref{eq:stationary}) for the nonconvex problem (\ref{eq:main}) in terms of the {\em theoretical communication complexity}\footnote{For the purposes of this paper, by {\em communication complexity} we mean the product of the number of communication rounds sufficient to find $\hat{x}$ satisfying (\ref{eq:stationary}), and a suitably defined measure of the volume of communication performed in each round. As standard in the literature, we assume that the workers-to-server communication is the key bottleneck, and hence we do not count server-to-worker communication. For more details about this highly adopted and studied setup, see Appendix~\ref{section:comm_model}.} is the \algname{MARINA} method of \citet{MARINA}. \algname{MARINA} relies on {\em worker-to-server  communication compression}, and its power resides in the construction of a carefully designed sequence of  {\em biased} gradient estimators  which help the method obtain its superior communication complexity. 
%satisfying \begin{equation}\label{eq:98gfdifd}\Exp{\cC_i(v)} = v, \qquad \Exp{\norm{\cC_i(v)-v}^2}\leq \omega \norm{v}^2, \qquad \forall v\in \R^d\end{equation} for  some $\omega\geq 0$. 
The method uses  {\em randomized compression operators} $\cC_i:\R^d\to \R^d$ to compress messages (gradient differences) at the workers $i\in \{1,2,\dots,n\}$ before they are communicated to the server. It is assumed that these operators are unbiased, i.e., $\Exp{\cC_i(a)}=a$ for all $a\in \R^d$, and that their variance is bounded as $$\Exp{\norm{\cC_i(a)-a}^2} \leq \omega \norm{a}^2$$ for all $a\in \R^d$ and some $\omega\geq 0$. For convenience, let 
$\mathbb{U}(\omega)$ be the class of such compressors. A key assumption in the analysis of \algname{MARINA} is the {\em independence} of the compressors $\{\cC_i\}_{i=1}^n$. 

 In particular, \algname{MARINA} solves the problem (\ref{eq:main})--(\ref{eq:stationary}) in 
 %$$ T = 2\Delta^0 \varepsilon^{-1} ( L_{-} + L_{+} \sqrt{(p^{-1}-1) \omega/n})$$ 
  $$\textstyle T = \frac{2\Delta^0}{\varepsilon} \left( L_{-} + L_{+} \sqrt{\frac{1-p}{p} \frac{\omega}{n}} \right)$$ 
 communication rounds\footnote{\citet{MARINA} {\em present} their result with $L_{-}$ replaced by the larger quantity $L_{+}$. However, after inspecting their proof, it is clear that they proved the improved rate we attribute to them here, and merely used the bound $L_{-} \leq L_{+}$ at the end for convenience of presentation only.}, where $\Delta^0 \eqdef f(x^0)-f^{\inf}$, $x^0\in \R^d$ is the initial iterate, $p \in (0,1]$ is a parameter defining the probability with which full gradients of the local functions $\{f_i\}$ are communicated to the server, $L_{-} > 0$ is the Lipschitz constant of the gradient of $f$, and $L_{+} \geq L_{-}$ is a certain smoothness constant associated with the functions $\{f_i\}$.  
 
 In each iteration of \algname{MARINA},  all workers send (at most) $p d + (1-p)\zeta$ floats to the server in expectation, where $\zeta \eqdef \max_i \sup_{v\in \R^d}{\rm size}(\cC_i(v))$, where ${\rm size}(\cC_i(v))$ is the size of the message $v$ compressed by compressor $\cC_i$. For an uncompressed vector $v$ we have ${\rm size}(v)=d$ in the worst case, and if $\cC_i$ is the Rand$K$ sparsifier, then ${\rm size}(\cC_i(v))=K$. 
 Putting the above together, the communication complexity of \algname{MARINA} is $T (pd + (1-p)\zeta)$, i.e., the product of the number of communication rounds and the communication cost of each round. See Section~\ref{section:marina_pseudocode} for more details on the method and its theoretical properties.
 
An alternative to the application of unbiased compressors is the practice of applying contractive compressors, such as Top$K$~\citep{Alistarh-EF-NIPS2018}, together with an error feedback mechanism~\citep{Seide2014, Stich-EF-NIPS2018, beznosikov2020biased}. However, this approach is not competitive in theoretical communication complexity with \algname{MARINA}; see Appendix~\ref{section:EF} for details.

\subsection{Summary of contributions}

{\bf (a) Correlated and permutation compressors.} We {\em generalize} the analysis of \algname{MARINA} {\em beyond independence} by supporting arbitrary unbiased compressors, including compressors that are {\em correlated}.  %By doing so, we extend the reach of the method to a much larger family of compressors with various use cases and interesting properties. 
In particular, we construct new compressors based on the idea of a {\em random permutation} (we called them Perm$K$) which provably reduce the variance caused by compression beyond what independent compressors can achieve. The properties of our compressors are captured by two quantities, $A\geq B \geq 0$, through a new inequality (which we call ``AB inequality")  bounding the variance of the {\em aggregated} (as opposed to individual) compressed message.

{\bf (b) Refined analysis through the new notion of Hessian variance.} We {\em refine} the analysis of \algname{MARINA} by identifying a new quantity, for which we coin the name {\em Hessian variance}, which plays an important role in our sharper analysis. To the best of our knowledge, Hessian variance is a new quantity proposed in this work and not used in optimization before. This quantity is well defined under the same assumptions as those used in the analysis of \algname{MARINA} by \citet{MARINA}. 

{\bf (c) Improved communication complexity results.} We prove  iteration complexity and communication complexity results for \algname{MARINA}, for smooth nonconvex (Theorem~\ref{theorem:AB}) and   smooth Polyak-\L ojasiewicz\footnote{The P\L{} analysis is included in Appendix~\ref{section:PL}.} (Theorem~\ref{theorem:AB_PL}) functions. Our results hold for all unbiased compression operators, including the standard independent but also all {\em correlated} compressors. Most importantly, we show that in the low Hessian variance regime, and by using our Perm$K$ compressors, we can improve upon the current state-of-the-art communication complexity of \algname{MARINA} due to \citet{MARINA} by up to the factor  $\sqrt{n}$ in the $d\geq n$ case, and up to the factor $1 + d/\sqrt{n}$ in the $d\leq n$ case. The improvement factors degrade gracefully as Hessian variance grows, and in the worst case we recover the same complexity as those established by \citet{MARINA}.

{\bf (d) Experiments agree with our theory.} Our theoretical results lead to predictions which are corroborated through computational experiments. In particular, we perform proof-of-concept testing with carefully engineered synthetic experiments with minimizing the average of nonconvex quadratics, and also test on autoencoder training with the MNIST dataset.

% Finally, we also consider the case when $f$ satisfies the  Polyak-\L ojasiewicz (P\L) condition, and perform a similar general and refined analysis of \algname{MARINA}, obtaining vastly improved {\em linear} iteration and communication complexity convergence rates in the small Hessian variance regime. The details are deferred to Section~\ref{sec:PL}; see Theorem~\ref{theorem:AB_PL}. 

%%%%%%%%%
\section{Beyond Independence: The Power of Correlated Compressors}  \label{sec:compressors}
%%%%%%%%%

As mentioned in the introduction, \algname{MARINA} was designed and analyzed to be used with compressors $\cC_i\in \mathbb{U}(\omega)$ that are sampled {\em independently} by the workers. For example, if the Rand$K$ sparsification operator is used by all workers, then each worker chooses the $K$ random coordinates to be communicated {\em independently} from the other workers.  This independence assumption is crucial for  \algname{MARINA} to achieve its superior theoretical properties. Indeed, without independence,  the rate would depend on $\omega$ instead\footnote{This is a consequence of the more general analysis from our paper; \citet{MARINA} do not consider the case of unbiased compressors without the independence assumption. } of $\nicefrac{\omega}{n}$, which would mean no improvement as the number $n$  of workers grows, which is problematic because  $\omega$ is typically very large\footnote{For example, in the case of  the Rand$K$ sparsification operator, $\omega=\nicefrac{d}{K}-1$. Since $K$ is typically chosen to be a constant, or a small percentage of $d$, we have $\omega= \cO(d)$, which is very large, and particularly so for overparameterized models.}. For this reason, independence is assumed in the analysis of virtually all distributed methods that use unbiased communication compression, including methods  designed for convex or strongly convex problems \citep{DCGD, DIANA, ADIANA, Artemis2020}.

In our work we first {\em generalize} the analysis of \algname{MARINA} {\em beyond independence}, which provably extends its  use to a much wider array of (still unbiased) compressors, some of  which have interesting theoretical properties and are useful in practice. 

\subsection{AB inequality: a tool for a more precise control of compression variance} 

We assume that all compressors $\{\cC_i\}_{i=1}^n$ are unbiased, and that there exist constants $A,B\geq 0$ for which the compressors satisfy a certain inequality, which we call ``AB inequality'', bounding the variance of  $\frac{1}{n}\sum_i \cC_i(a_i)$ as a stochastic estimator of  $\frac{1}{n}\sum_i a_i$.

\begin{assumption}[Unbiasedness]\label{eq:unbiased_compressors} The random operators $\cC_1,\dots,\cC_n:\R^d\to \R^d$ are unbiased, i.e.,  $\Exp{\cC_i(a)} =a$ for all $i\in \{1,2,\dots,n\}$ and all $a\in \R^d$. If these conditions are satisfied, we will write $\{\cC_i\}_{i=1}^n \in \mathbb{U}$.
\end{assumption}

\begin{assumption}[AB inequality]\label{ass:AB}  There exist constants $A,B\geq 0$ such that the random operators $\cC_1,\dots,\cC_n:\R^d\to \R^d$  satisfy the inequality
\begin{equation}\textstyle \label{eq:AB}\Exp{ \norm{ \frac{1}{n} \sum \limits_{i=1}^n \cC_i(a_i) - \frac{1}{n}\sum\limits_{i=1}^n a_i }^2 } \leq A \left(\frac{1}{n}\sum\limits_{i=1}^n \norm{a_i}^2\right) - B \norm{ \frac{1}{n}\sum\limits_{i=1}^n a_i }^2\end{equation}
for all $a_1,\dots,a_n\in \R^d$. If these conditions are satisfied, we will write $\{\cC_i\}_{i=1}^n \in \mathbb{U}(A,B)$.
\end{assumption}

It is easy to observe that whenever the AB inequality holds, it must necessarily be the case that $A\geq B$.   Indeed, if we fix nonzero $a\in \R^d$ and choose $a_i = a$ for all $i$, then the right hand side of the AB inequality is equal to $A-B$ while the left hand side  is nonnegative.

Our next observation is that whenever $\cC_i\in \mathbb{U}(\omega_i)$ for all $i\in \{1,2,\dots,n\}$, the AB inequality holds  without any assumption on the independence of the compressors. Furthermore, if independence is assumed, the $A$ constant is substantially improved. 

\begin{restatable}{lemma}{unbiased}
  \label{lem:bg97fd890d} 
  If $\cC_i\in \mathbb{U}(\omega_i)$ for $i\in \{1,2,\dots, n\}$, then $\{\cC_i\}_{i=1}^n \in \mathbb{U}(\max_i \omega_i,0)$. If we further assume that the compressors  are independent, then $\{\cC_i\}_{i=1}^n \in \mathbb{U}(\frac{1}{n}\max_i \omega_i,0)$.
\end{restatable}

%\begin{proof} By Jensen's inequality,
%$ \norm{ \frac{1}{n} \sum_{i=1}^n \cC_i(a_i) - \frac{1}{n}\sum_{i=1}^n a_i}^2 \leq  \frac{1}{n} \sum_{i=1}^n \norm{ \cC_i(a_i) - a_i}^2.$
%It remains to apply expectation on both sides and then to apply Inequality (\ref{eq:U-class}).
%\end{proof}

%Whenever $\cC_i\in \mathbb{U}(\omega_i)$ for all $i$, then $\{\cC_i\}_{i=1}^n \in \mathbb{U}(\max_i \omega_i,0)$. If, in addition, independence is assumed, then  $\{\cC_i\}_{i=1}^n \in \mathbb{U}(\frac{1}{n}\max_i \omega_i,0)$. 

In Table~\ref{tbl:compressors} we provide a list of several compressors that belong to the class $\mathbb{U}(A,B)$, and give values of the associated constants $A$ and $B$.

\subsection{Why correlation may help}
While in the two examples captured by Lemma~\ref{lem:bg97fd890d} we  had $B=0$, with a carefully crafted {\em dependence} between the compressors it is possible for $B$ to be positive, and even as large as $A$. Intuitively, other things equal (e.g., fixing $A$), we should want $B$ to be positive, and as large as possible, as the AB inequality says that in such a case the variance of $\frac{1}{n}\sum_i \cC_i(a_i)$ as a stochastic estimator of  $\frac{1}{n}\sum_i a_i$ is reduced more dramatically. This is a key intuition behind the usefulness of (appropriately) correlated compressors. We now provide an alternative point of view. Note that \begin{equation}\label{eq:AB-variant} \textstyle A \left(\frac{1}{n}\sum\limits_{i=1}^n \norm{a_i}^2\right) - B \norm{ \frac{1}{n}\sum\limits_{i=1}^n a_i }^2 = A \left [\left(1 -\frac{B}{A}\right) \left(\frac{1}{n}\sum \limits_{i=1}^n \norm{a_i}^2\right) + \frac{B}{A} {\rm Var}(a_1,\dots,a_n) \right], \end{equation}
where ${\rm Var}(a_1,\dots,a_n) \eqdef \frac{1}{n}\sum_{i=1}^n \norm{a_i - \bar{a}}^2$ is the variance of the vectors $\{a_i\}_{i=1}^n$ and $\bar{a}\eqdef \frac{1}{n}\sum_{i=1}^n a_i$ is their average. So, the AB inequality upper bounds the variance of $\frac{1}{n}\sum_i \cC_i(a_i)$ as $A$ times a particular convex combination of two quantities. Since the latter quantity is always smaller or equal to the former, and can be much smaller, we should prefer compressors which put as much weight on ${\rm Var}(a_1,\dots,a_n)$ as possible. 

\subsection{Input variance compressors}

Due to the above considerations, compressors for which $A=B$ are special, and their construction and theoretical properties are a key contribution of our work. Moreover, as we shall see in Section~\ref{sec:theory}, such compressors have favorable communication complexity properties. This leads to the following definition:

\begin{definition}[Input variance compressors] We say that a collection $\{\cC_i\}_{i=1}^n$ of unbiased operators form an {\em input variance} compressor system if the variance of $\frac{1}{n}\sum_i \cC_i(a_i)$ is controlled by a multiple of the variance of the input vectors $\{a_i\}_{i=1}^n$. That is, if there exists a constant $C\geq 0$ such that
\begin{equation}\textstyle \label{eq:IV}\Exp{ \norm{ \frac{1}{n} \sum \limits_{i=1}^n \cC_i(a_i) - \frac{1}{n}\sum\limits_{i=1}^n a_i }^2 } \leq C {\rm Var}(a_1,\dots,a_n)\end{equation}
for all $a_1,\dots,a_n\in \R^d$.  If these conditions are satisfied, we will write $\{\cC_i\}_{i=1}^n \in \mathbb{IV}(C)$.

\end{definition}

In view of (\ref{eq:AB-variant}), if $\{\cC_i\}_{i=1}^n \in \mathbb{U}(A,B)$ and $A=B$, then
$\{\cC_i\}_{i=1}^n \in \mathbb{IV}(A)$.

\begin{table}[t!]
\caption{%\footnotesize
 Examples of compressors $\{\cC_i\}_{i=1}^n \in \mathbb{U}(A,B)$. See the appendix for many more.}
\label{tbl:compressors}
%\footnotesize
\begin{center}
\begin{tabular}{lcccc}
\multicolumn{1}{c}{Compressors}  & $A$ & $B$ & Calculation of $A,B$ & Reference \\ 
\hline \\
$\cC_i \in \mathbb{U}(\omega_i)$                                            & $\max_i \omega_i$ & $0$  & Lemma~\ref{lem:bg97fd890d} & standard \\
$\cC_i \in \mathbb{U}(\omega_i)$, independent                      & $\frac{1}{n}\max_i \omega_i$ & $0$ & Lemma~\ref{lem:bg97fd890d} & standard \\ 
%$\cC_i \in \mathbb{U}(\omega_i)$, orthogonal                         & $\frac{1}{n} \left(\max_i \omega_i +1\right)$ & $1$ & ??? & {\bf new} \\
 \rowcolor{bgcolor}  Perm$K$ ($d \geq n$);  Def \ref{def:PermK-1} & $1$ & $1$ & Theorem~\ref{thm:PermK-1} & {\bf new}  \\
 \rowcolor{bgcolor}  Perm$K$ ($d \leq n$);  Def \ref{def:PermK-2} & $1-\frac{n-d}{n-1}$ & $1-\frac{n-d}{n-1}$ & Theorem~\ref{thm:PermK-2} & {\bf new}  \\
% Rand${}^\top K$; Def \red{11} & $1-\frac{n(K-1)}{(n-1)K}$ & $1-\frac{n(K-1)}{(n-1)K}$  & \red{Lemma 16} & {\bf new} \\
%Block PC2;  Def \red{12} & $1-\frac{n(q-1)}{(n-1)q}$ & $1-\frac{n(q-1)}{(n-1)q}$ & \red{Lemma 17}  &  {\bf new} \\
\end{tabular}
\end{center}
\end{table}

%\begin{assumption}[Unbiased compressors] \label{ass:unbiased_compressors} 
%  Let $\cC_1,\dots,\cC_n:\R^d\to \R^d$ be a collection of (possibly correlated) random mappings such that 
%  \begin{equation*} \Exp{\frac{1}{n} \sum_{i=1}^n \cC_i(a_i)} = \frac{1}{n}\sum_{i=1}^n a_i\end{equation*}
%  for all $a_1,\dots,a_n\in \R^d$. 
%\end{assumption}

\subsection{Perm$K$: permutation based sparsifiers} 

We now define two input variance compressors based on a {\em random permutation} construction.\footnote{More examples of input variance compressors are given in the appendix.}  The first compressor handles the $d\geq n$ case, and the second handles the $d \leq n$ case. For simplicity  of exposition, we assume that $d$ is divisible by $n$ in the first case, and that $n$ is divisible by $d$ in the second case.\footnote{The general situation is handled in Appendix~\ref{sec:general_examples}.} Since both these new compressors are sparsification operators, in an analogy with the established notation Rand$K$ and Top$K$ for sparsification, we will write Perm$K$ for our permutation-based sparsifiers. To keep the notation simple, we chose to include  simple variants which do not offer freedom in choosing $K$. Having said that, these simple compressors lead to state-of-the-art communication complexity results for \algname{MARINA}, and hence not much is lost by focusing on these examples. Let $e_i$ be the $i^{\rm th}$ standard unit basis vector in $\R^d$. That is, for any $x = (x_1,\dots,x_d)\in \R^d$ we have $x = \sum_i x_i e_i$.

\begin{definition}[Perm$K$ for $d \geq n$]
  \label{def:PermK-1}
Assume that $d \geq n$ and $d  = q n$, where $q\geq 1$ is an integer. Let $\pi = (\pi_1,\dots,\pi_d)$ be a random permutation of $\{1, \dots, d\}$. Then for all $x \in \R^d$ and each $i\in \{1,2,\dots,n\}$ we define
 \begin{equation}\label{eq:PermK-1} \textstyle \cC_i(x) \eqdef n \cdot \sum \limits_{j = q (i - 1) + 1}^{q i} x_{\pi_j} e_{\pi_j}.\end{equation}
\end{definition}

Note that $\cC_i$ is a sparsifier: we have $(\cC_i(x))_l = n x_l$ if $l \in \{\pi_j \;:\;  q (i-1)+1 \leq j \leq q i\}$ and $(\cC_i(x))_l = 0$ otherwise. So, $\|\cC_i(x)\|_0 \leq q \eqdef K$, which means that $\cC_i$ offers compression by the factor $n$. Note that we do not have flexibility to choose $K$; we have $K=q=d/n$. See Appendix~\ref{sec:implementation_details} for implementation details.

\begin{restatable}{theorem}{PermK}\label{thm:PermK-1}
The Perm$K$ compressors 
%$\{\cC_i\}_{i = 1}^{n}$ 
from Definition~\ref{def:PermK-1} are unbiased and belong to $\mathbb{IV}(1)$.
\end{restatable}

In contrast with the collection of independent Rand$K$ sparsifiers, which satisfy the AB inequality with $A = \frac{d/K-1}{n}$ and $B=0$ (this follows from Lemma~\ref{lem:bg97fd890d} since $\omega_i = d/K-1$ for all $i$), Perm$K$ satisfies the AB inequality with $A=B=1$. While both are sparsifiers, the permutation construction behind Perm$K$ introduces a favorable correlation among the compressors: we have $\langle \cC_i(a_i), \cC_j(a_j)\rangle = 0$ for all $i\neq j$.

\begin{definition}[Perm$K$ for $n \geq d$]
  \label{def:PermK-2}
  Assume that $n \geq d,$ $n > 1$ and $n  = q d,$ where $q\geq 1$ is an integer. 
Define the multiset $S \eqdef \{1, \dots, 1, 2, \dots, 2, \dots, d, \dots, d\}$, where each number occurs precisely $q$ times. Let $\pi=(\pi_1,\dots,\pi_{n})$ be a random permutation of $S$. Then for all $x \in \R^d$ and each $i\in \{1,2,\dots,n\}$ we define  
 \begin{equation}\label{eq:PermK-2} \textstyle 
    \cC_i(x) \eqdef d x_{\pi_i}e_{\pi_i}.
\end{equation}
\end{definition}

Note that for each $i$, $\cC_i$ from Definition~\ref{def:PermK-2} {\em is} the Rand$1$ sparsifier, offering compression factor $d$.  However, the sparsifiers $\{\cC_i\}_{i=1}^n$ are {\em not} mutually independent. Note that, again, we do not have a choice\footnote{It is possible to provide a more general definition of Perm$K$ in the $n\geq d$ case, allowing for more freedom in choosing $K$. However, such compressors would lead to a worse communication complexity for \algname{MARINA} than the simple $K=1$ variant considered here.} of $K$ in Definition~\ref{def:PermK-2}: we have $K=1$.

\begin{restatable}{theorem}{PermKK}\label{thm:PermK-2}
The Perm$K$ compressors 
%$\{\cC_i\}_{i = 1}^{n}$ 
from Definition~\ref{def:PermK-2} are unbiased and belong to $\mathbb{IV}(A)$ with $A=1-\frac{n-d}{n-1}$.
\end{restatable}

%The main example are compressors based on the idea of  {\em random permutation.} In an analogy with the established notation Rand$K$ and Top$K$, we will write Perm$K$ when talking about a particular permutation-based compressor/sparsifier that reduces a $d$-dimensional vector down to a $K$-dimensional vector. {\em Perm$K$  is a new class of compressors, which, to the best of our knowledge, was not proposed before.} For example, if $d\geq n$, we can construct permutation compressors $\{\cC_i\}_{i=1}^n$ which compress a $d$ dimensional vector down to a $K=d/n$ dimensional vector, and for which $A=B=1$. We shall illustrate the significance of these properties in Section~\ref{sec:theory}.

{\bf Combining Perm$K$ with quantization.} It is easy to show that if  $\{\cC_i\}_{i=1}^n \in \mathbb{U}(A,B)$, and $\cQ_i \in  \mathbb{U}(\omega_i)$ are chosen independently of $\{\cC_i\}_{i=1}^n$ (we do not require mutual independence of $\{\cQ_i\}$), then $\{\cC_i \circ \cQ_i\}_{i=1}^n \in \mathbb{U}((\max_i \omega_i +1)A,B)$ (see Lemma~\ref{lem:composition}). This allows us to combine our compression techniques with quantization  \citep{alistarh2017qsgd, Cnat}.

 %%%%%%%%%%%%%%%%%%%%%%%%
\section{Hessian Variance}
 %%%%%%%%%%%%%%%%%%%%%%%%

\begin{table}
%\footnotesize
\caption{
%\footnotesize 
Value of $\Lpm^2$ in cases when $f_i(x)  = \phi(x) + \phi_i(x)$, where $\phi:\R^d\to \R$ is an arbitrary differentiable function and $\phi_i:\R^d \to \R$ is twice continuously differentiable. The matrices $\mA_i\in \R^{d\times d}$ are assumed (without loss of generality) to be symmetric. The matrix-valued function $\mL_{\pm}(x,y)$ is defined in Theorem~\ref{thm:second-order}.}
\label{tbl:D_can_be_small}
\begin{center}
\begin{tabular}{ cc | c}
 $\phi(x)$ & $\phi_i(x)$ & Hessian variance $\Lpm^2$ \\ 
\hline 
 any  & $0$ & $0$ \\
 any  & $b_i^\top x + c_i$ & $0$ \\
 \rowcolor{bgcolor}  0  & $\frac{1}{2} x^\top \mA_i x + b_i^\top x + c_i$ & $\lambda_{\max} \left( \frac{1}{n} \sum_{i=1}^n \mA_i^2  - \left(\frac{1}{n} \sum_{i=1}^n \mA_i\right)^2 \right)$ \\
 \rowcolor{bgcolor}  0  & smooth & $\sup \limits_{x, y \in \R^d, x\neq y}  \frac{(x-y)^\top  \mLpm(x,y)(x-y)}{\norm{x-y}^2}$ \\
\end{tabular}
\end{center}
\end{table} 

Working under the {\em same} assumptions on the problem  (\ref{eq:main})--(\ref{eq:stationary}) as \citet{MARINA} (i.e., Assumptions~\ref{ass:diff} and \ref{as:L_+}), in this paper we study the complexity of \algname{MARINA} under the influence of a new quantity, which we call {\em Hessian variance.} 

\begin{definition}[Hessian variance] \label{ass:HV} Let $\Lpm\geq 0$ be the smallest quantity such that
\begin{equation}\label{eq:HV} \textstyle \frac{1}{n} \sum \limits_{i=1}^n \norm{\nabla f_i(x) - \nabla f_i(y)}^2 - \norm{\nabla f(x) - \nabla f(y)}^2 \leq L^2_{\pm} \norm{x-y}^2, \quad \forall x,y\in \R^d. \end{equation}
We refer to the quantity $L_{\pm}^2$ by the name {\em Hessian variance.} 
\end{definition}

 Recall that in this paper we have so far mentioned four ``smoothness'' constants: $L_i$ (Lipschitz constant of $\nabla f_i$), $L_-$ (Lipschitz constant of $\nabla f$), $L_+$ (see Assumption~\ref{as:L_+}) and $\Lpm$ (Definition~\ref{ass:HV}). To avoid ambiguity, let all be defined as the smallest constants for which the defining inequalities hold. In case the defining inequality does not hold, the value is set to $+\infty$. This convention allows us to formulate the following result summarizing the relationships between these quantities.

\begin{restatable}{lemma}{ineq} \label{lem:08hhdf} $L_{-} \leq L_{+}$, $L_{-} \leq \frac{1}{n}\sum \limits_{i=1}^n L_i$, $L^2_{+}\leq \frac{1}{n}\sum\limits_{i=1}^n L_i^2$, and $L_{+}^2 - L_{-}^2\leq \Lpm^2 \leq L_{+}^2$.
\end{restatable}

It follows that if $L_i$ is finite for all $i$, then $L_-, L_+$ and $\Lpm$ are all finite as well. Similarly, if $L_+$ is finite (i.e., if Assumption~\ref{as:L_+} holds),  then $L_-$ and $\Lpm$ are finite, and $L_{\pm}\leq L_+$. We are not aware of any prior use of this quantity in the analysis of any optimization methods. Importantly,
there are situations when $L_-$   is large, and yet the Hessian variance $\Lpm^2$ is small, or even zero. This is important as the improvements we obtain in our analysis of \algname{MARINA} are most pronounced in the regime when the Hessian variance is small.

\subsection{Hessian variance can be zero}

We now illustrate on a few examples that there are situations when the values of $L_-$ and  $L_i$ are  large and the Hessian variance is zero. The simplest such example is the identical functions regime. 

\begin{example}[Identical functions] Assume that $f_1=f_2=\dots=f_n$. Then $\Lpm=0$. \end{example}

This follows by observing that the left hand side in (\ref{eq:HV}) is zero. Note that while $\Lpm=0$, it is possible for  $L_-$ and $L_+$ to be arbitrarily large! Note that methods based on the Top$K$ compressor (including all error feedback methods) suffer in this regime. Indeed, \algname{EF21} in this simple scenario is the same method for any value of $n$, and hence can't possibly improve as $n$ grows. This is because when $a_i=a_j$ for all $i,j$, $\frac{1}{n}\sum_i {\rm Top}K(a_i) = {\rm Top}K(a_i)$. As the next example shows, Hessian variance is zero even if we perturb the local functions via {\em arbitrary} linear functions.

\begin{example}[Identical functions + arbitrary linear perturbation] Assume that $f_i(x) = \phi(x) + b_i^\top x + c_i$, for some differentiable function $\phi:\R^d\to \R$ and arbitrary $b_i\in \R^d$ and  $c_i\in \R$. Then $\Lpm=0$.
\end{example}
This follows by observing that the left hand side in (\ref{eq:HV}) is zero in this case as well. Note that in this example it is possible for the functions $\{f_i\}$ to have {\em arbitrarily different minimizers}. So, this example does {\em not} correspond to the overparameterized machine learning regime, and is in general challenging for standard methods.

\subsection{Second order characterization}

To get an insight into when the Hessian variance may be small but not necessarily zero, we establish a useful second order characterization.

\begin{restatable}{theorem}{secondorder} \label{thm:second-order}Assume that for each $i\in \{1,2,\dots,n\}$, the function $f_i$ is twice continuously differentiable. Fix any $x,y\in \R^d$ and define\footnote{Note that $\mH_i(x,y) $ is the average of the Hessians of $f_i$ on the line segment connecting $x$ and $y$.} \begin{equation} \label{eq:hfd-0f9y8gfd9-u8fd} \textstyle \mH_i(x,y) \eqdef \int_{0}^1 \nabla^2 f_i(x + t(y-x)) \; d t, \qquad \mH(x,y) \eqdef \frac{1}{n} \sum \limits_{i=1}^n \mH_i(x,y) .\end{equation}
Then the matrices $\mL_{i}(x,y)\eqdef  \mH_i^2 (x,y)$, $\mL_{-}(x,y)\eqdef \mH^2(x,y)$,
$\mL_{+}(x,y)\eqdef \frac{1}{n} \sum_{i=1}^n \mH_i^2 (x,y)$ and 
$\mLpm(x,y) \eqdef \mL_{+}(x,y) - \mL_{-}(x,y) $
are symmetric and positive semidefinite. Moreover,  \[ \textstyle L_i^2 = \sup \limits_{x, y \in \R^d, x\neq y}  \frac{(x-y)^\top  \mL_i(x,y)(x-y)}{\norm{x-y}^2}, \quad L_{-}^2 = \sup \limits_{x, y \in \R^d, x\neq y}  \frac{(x-y)^\top  \mL_{-}(x,y)(x-y)}{\norm{x-y}^2},\]
 \[ \textstyle L_{+}^2 = \sup \limits_{x, y \in \R^d, x\neq y}  \frac{(x-y)^\top  \mL_{+}(x,y)(x-y)}{\norm{x-y}^2}  , \quad \Lpm^2 =\sup \limits_{x, y \in \R^d, x\neq y}  \frac{(x-y)^\top  \mLpm(x,y)(x-y)}{\norm{x-y}^2}. \]
\end{restatable}

While  $\Lpm^2$ is obviously well defined through Definition~\ref{ass:HV} even when the functions $\{f_i\}$ are {\em not} twice differentiable, the term ``Hessian variance'' comes from the interpretation of  $L_{\pm}^2$ in the case of quadratic functions.

\begin{example}[Quadratic functions]
\label{ex:quadratic_functions}
Let  $f_i(x)=\frac{1}{2} x^\top \mA_i x + b_i^\top x + c_i$, where $\mA_i\in \R^{d\times d}$ are symmetric.  Then $\Lpm^2 = \lambda_{\max} ( \frac{1}{n} \sum_{i=1}^n \mA_i^2  - \left(\frac{1}{n} \sum_{i=1}^n \mA_i)^2 \right)$, where $\lambda_{\max}(\cdot)$ denotes the largest eigenvalue.
\end{example}

Indeed, note that the matrix $\frac{1}{n} \sum_{i=1}^n \mA_i^2  - \left(\frac{1}{n} \sum_{i=1}^n \mA_i\right)^2$ can be interpreted as a matrix-valued variance of the Hessians $\mA_1,\dots,\mA_n$, and $\Lpm^2$ measures the size of this matrix in terms of its largest eigenvalue. 

See Table~\ref{tbl:D_can_be_small} for a summary of the examples mentioned above.  As we shall explain in Section~\ref{sec:theory}, the data/problem regime when the Hessian variance is small is of key importance to the improvements we obtain in this paper.

%In this simple estimate, we have ignored the effect of the term $- \norm{\nabla f(x) - \nabla f(y)}^2$ on the value of $L_{\pm}$. However, this effect is sometimes large, and can lead to $L_{\pm}$ being small even in situations when, for example, the constants  $\{L_i\}$, or even the Lipschitz constant of the gradient of $f$, are very large.

% Clearly, Assumption~\ref{ass:second} is satisfied, for example, when each $f_i$ has an $L_i$-Lipschitz gradient. Indeed, in this case we have $\frac{1}{n} \sum_{i=1}^n \norm{\nabla f_i(x) - \nabla f_i(y)}^2 \leq \frac{1}{n} \sum_{i=1}^n L_i^2 \norm{x - y}^2$, and hence we know that $L_{\pm}^2 \leq \frac{1}{n} \sum_{i=1}^n L_i^2$. 

%%%%%%%%%%%%%%%%%%%%%%%
\section{Improved Iteration and Communication Complexity} \label{sec:theory}
%%%%%%%%%%%%%%%%%%%%%%%

The key contribution of our paper is a more general and more refined analysis of \algname{MARINA}. In particular, we i) extend the reach of \algname{MARINA} to the general class of unbiased and possibly correlated compressors $\{\cC_i\}_{i=1}^n \in \mathbb{U}(A,B)$ while ii) providing a more refined analysis in that we  take the Hessian variance $L_{\pm}^2$ into account. 

\begin{restatable}{theorem}{ab}
  \label{theorem:AB} Let Assumptions~\ref{ass:diff}, \ref{as:L_+}, \ref{eq:unbiased_compressors}  and
\ref{ass:AB} be satisfied. Let the stepsize in \algname{MARINA} be chosen as
$
  0<  \gamma \leq \frac{1}{M}
$, where $M=L_- + \sqrt{\frac{1-p}{p}\left((A - B)L_+^2 + BL_\pm^2\right)}$. Then after $T$ iterations, \algname{MARINA} finds a random point $\hat{x}^T\in \R^d$ for which
  $$
 \textstyle \Exp{\norm{ \nabla f(\hat x^T)}^2} \leq \frac{2 \Delta^0}{\gamma T}.
  $$
\end{restatable}

 In particular, by choosing the maximum stepsize allowed by Theorem~\ref{theorem:AB}, 
 \algname{MARINA} converges in $T$ communication rounds, where $T$ is shown in the first row Table~\ref{table:orcale_complexity}.  If in this result we replace $\Lpm^2$ by the coarse estimate $L_{\pm}^2 \leq L_+^2$, and further specialize to independent compressors satisfying  $\cC_i \in \mathbb{U}(\omega)$ for all $i\in \{1,2,\dots,n\}$, then since $\{\cC_i\}_{i=1}^n\in \mathbb{U}(\omega/n,0)$ (recall Lemma~\ref{lem:bg97fd890d}), our general rate specializes to the result of \citet{MARINA}, which we show in the second row of Table~\ref{table:orcale_complexity}. 
 
 {\em However, and this is a key finding of our work, in the regime when the Hessian variance $\Lpm^2$ is very small, the original result of \citet{MARINA} can be vastly suboptimal!} To show this, in Table~\ref{table:communication_complexity} we compare the {\em communication complexity}, i.e., the \# of communication rounds multiplied by the maximum \# of floats transmitted by a worker to the sever in a single communication round. We compare the communication complexity of \algname{MARINA} with the Rand$K$ and Perm$K$ compressors, and the state-of-the-art error-feedback method \algname{EF21} of \citet{EF21} with the Top$K$  compressor. In all cases we do not consider the communication complexity of the initial step equal to $\cO(d)$. In each case we optimized over the parameters of the methods (e.g., $p$ for \algname{MARINA} and $K$ in all cases; for details see Appendix~\ref{sec:appendix:complexity_bounds}). Our results for \algname{MARINA} with Perm$K$ are better than the competing methods (recall Lemma~\ref{lem:08hhdf}).

\begin{table}[t!]
  \centering
  \caption{  %\footnotesize 
  The number of communication rounds for  solving (\ref{eq:main})--(\ref{eq:stationary}) by \algname{MARINA} and \algname{EF21}.}
  \label{table:orcale_complexity}
 % \footnotesize  
  \begin{tabular}{|c|c|c|}
    \hline
    Method + Compressors & $T =$ \# Communication Rounds \\ \hline    \hline
  \rowcolor{bgcolor}   \begin{tabular}{@{}c@{}} \algname{MARINA} $\bigcap$ $\{\cC_i\}_{i=1}^n \in \mathbb{U}(A,B)$ \\
    { (this paper, 2021)}
     \end{tabular} & 
    $\cO\left(\frac{\Delta_0}{\varepsilon}\left(L_- + \sqrt{\frac{1-p}{p}\left((A - B)L_+^2 + BL_\pm^2\right)}\right)\right)$
    \\ \hline
    \begin{tabular}{@{}c@{}}\algname{MARINA} $\bigcap$ $\cC_i \in \mathbb{U}(\omega)$ and independent \\ \citep{MARINA}\end{tabular} &
    $\cO\left(\frac{\Delta_0}{\varepsilon}\left(L_- + \sqrt{\frac{1-p}{p}\frac{\omega}{n}}L_+\right)\right)$ 
    \\ \hline
    \begin{tabular}{@{}c@{}}\algname{EF21} $\bigcap$ $\cC_i$ are $\alpha$-contractive \\ \citep{EF21}\end{tabular} & $\cO\left(\frac{\Delta_0}{\varepsilon}\left(L_- + \left(\frac{1 + \sqrt{1 - \alpha}}{\alpha} - 1\right)L_+\right)\right)$ \\ \hline
  \end{tabular}
\end{table}

\newcommand{\scaleboxparam}{0.85}

\begin{table}[t!]
  \centering
  \caption{  % \footnotesize 
  Optimized communication complexity of \algname{MARINA} and \algname{EF21} with particular compressors. }
  \label{table:communication_complexity}
 % \footnotesize  
  \begin{tabular}{|c|c|c|}
    \hline 
    &\multicolumn{2}{|c|}{Communication Complexity} \\ \hline
    Method + Compressor & $d \geq n$ (Lemma~\ref{lemma:optimal_parameters_of_methods}) & $d \leq n$ (Lemma~\ref{lemma:optimal_parameters_of_methods_n_geq_d})\\ \hline \hline
 \rowcolor{bgcolor}     \algname{MARINA} $\bigcap$ Perm$K$  & 
    \scalebox{\scaleboxparam}{$\cO\left(\frac{\Delta_0}{\varepsilon}\min\left\{dL_{-},\frac{d}{n}L_- + \frac{d}{\sqrt{n}}L_\pm\right\}\right)$}
    & \scalebox{\scaleboxparam}{$\cO\left(\frac{\Delta_0}{\varepsilon}\min\left\{dL_{-},L_- + \frac{d}{\sqrt{n}}L_\pm\right\}\right)$}
    \\ \hline
    \algname{MARINA} $\bigcap$ Rand$K$  & 
    \scalebox{\scaleboxparam}{$\cO\left(\frac{\Delta_0}{\varepsilon}\min\left\{dL_{-},\frac{d}{\sqrt{n}}L_{+}\right\}\right)$}
    & \scalebox{\scaleboxparam}{$\cO\left(\frac{\Delta_0}{\varepsilon}\min\left\{dL_{-},L_- + \frac{d}{\sqrt{n}}L_{+}\right\}\right)$}
    \\ \hline
    \algname{EF21} $\bigcap$ Top$K$ & \scalebox{\scaleboxparam}{$\cO\left(\frac{\Delta_0}{\varepsilon}dL_-\right)$}
    & \scalebox{\scaleboxparam}{$\cO\left(\frac{\Delta_0}{\varepsilon}dL_-\right)$}  \\[0.5ex] \hline
  \end{tabular}
\end{table}

% \begin{equation} \label{eq:marina-our-rate-contributions} \textstyle 
% T = \frac{2 \Delta^0}{\varepsilon}\left(L_- + \sqrt{\frac{1-p}{p}\left((A - B)L_+^2  + B \Lpm^2\right)}\right)
%  \end{equation}
%  communication rounds. 

\subsection{Improvements in the ideal zero-Hessian-variance regime} To better understand the  improvements our analysis provides, let us consider the ideal regime characterized by zero Hessian variance: $\Lpm^2=0$. If we now use compressors $\{\cC_i\}_{i=1}^n \in \mathbb{U}(A,B)$ for which $A=B$, which is the case for Perm$K$, then the dependence on the potentially very large quantity $L_+^2$ is eliminated completely.

{\bf Big model case ($d\geq n$).}  In this case, and using the Perm$K$ compressor, \algname{MARINA} has communication complexity $\cO( L_- \Delta^0 \varepsilon^{-1}\nicefrac{d}{n} )$, while using the  Rand$K$ compressor, the communication complexity of \algname{MARINA} is no better than $\cO(L_- \Delta^0 \varepsilon^{-1} \nicefrac{d}{\sqrt{n}} )$. Hence, we get an improvement by {\em at least} the factor $\sqrt{n}$. Moreover, note that this is an $n\times$ improvement over  gradient descent (\algname{GD}) \citep{khaled2020better} and \algname{EF21}, both of which have communication complexity $\cO( L_-\Delta^0 \varepsilon^{-1}  d) $. In Appendix~\ref{sec:group_hessian_variance}, we discuss how we can get the same theoretical improvement even if $\Lpm^2 > 0.$

%\begin{quote} \em This means that in the zero Hessian variance regime, we can solve the problem in the same \# of communication rounds as \algname{GD}, while communicating compressed messages only. In other words, we get compressed communication for free! \end{quote}

{\bf Big data case ($d\leq n$).}  In this case, and using the Perm$K$ compressor, \algname{MARINA} achieves communication complexity $\cO(L_- \Delta^0 \varepsilon^{-1})$, while using the  Rand$K$ compressor, the communication complexity of \algname{MARINA} is no better than $\cO(L_- \Delta^0 \varepsilon^{-1} (1+\nicefrac{d}{\sqrt{n}}) )$.  Hence, we get an improvement by {\em at least} the factor $1 + d/\sqrt{n}$. Moreover, note that this is a $d\times$ improvement over  gradient descent (\algname{GD}) and \algname{EF21}, both of which have communication complexity $\cO( L_-  \Delta^0\varepsilon^{-1}  d) $.

\section{Experiments} \label{sec:experiments}
%%%%%%%%%%%%%%%%%%%%%%%%%

% \newcommand{\smothnessvariance}{$\Lpm^2$}
% \newcommand{\ltilde}{$L_{+}^2$}
% \newcommand{\blockpermutation}{Perm$K$}
% \newcommand{\randk}{Rand$K$}

We compare \algname{MARINA} using Rand$K$ and Perm$K$, and \algname{EF21} with Top$K$,  in two experiments. In the first experiment, we construct quadratic optimization tasks with different $L_\pm$ to capture the dependencies that our theory predicts. In the second experiment, we consider practical machine learning task MNIST \citep{lecun2010mnist} to support our assertions.
Each plot represents the dependence between the norm of gradient (or function value) and the total number of transmitted bits by a node.

\subsection{Testing theoretical predictions on a synthetic quadratic problem}

\label{section:experiment:quadratic}

To test the predictive power of our theory in a controlled environment, we first consider a synthetic (strongly convex) quadratic function $f=\frac{1}{n}\sum f_i$ composed of nonconvex quadratics $$ \textstyle f_i(x) \eqdef \frac{1}{2}x^\top \mA_i x - x^\top b_i,
$$
where $b_i \in \R^d,$ $\mA_i \in \R^{d \times d},$ and $\mA_i = \mA_i^\top$. We enforced that $f$ is $\lambda$--strongly convex, i.e., $\frac{1}{n}\sum_{i=1}^n \mA_i\succcurlyeq \lambda \mI $ for $\lambda > 0.$ We fix $\lambda = 1\mathrm{e}{-6}$, and dimension $d = 1000$ (see Figure~\ref{fig:project-permutation-compression-marina-quadratic-with-norms_best}). We then generated optimization tasks with the number of nodes $n \in \{10, 1000, 10000\}$ and $L_\pm \in \{0, 0.05, 0.1, 0.21, 0.91\}$. We take \algname{MARINA}'s and \algname{EF21}'s parameters prescribed by the theory and performed a grid search for the step sizes for each compressor by multiplying the theoretical ones with powers of two. For simplicity, we provide one plot for each compressor with the best convergence rate. 
First, we see that Perm$K$ outperforms Rand$K$, and their differences in the plots reproduce dependencies from Table~\ref{table:communication_complexity}. Moreover,
when $n \in \{1000, 10000\}$ and $\Lpm \leq 0.21$,  
\algname{EF21} with Top$K$ has worse performance than \algname{MARINA} with Perm$K$, while in heterogeneous regime, when $L_\pm = 0.91$, Top$K$ is superior except when $n = 10000$. 
% Detailed experiments and explanations are presented in Appendix~\ref{section:extra_experiments}. 
See Appendix~\ref{section:extra_experiments} for detailed experiments.

\newcommand{\experimetscaption}[1]{Comparison of #1on synthetic quadratic optimization tasks. 
Each row corresponds to a fixed number of nodes; each column corresponds to a fixed noise scale. 
In the legends, we provide compressor names and fine-tuned multiplicity factors of step sizes relative to theoretical ones. 
Abbreviations: NS = noise scale. Axis $x$ represents the number of bits that every node has sent. Dimension $d = 1000.$}

\begin{figure}[h!]
 \centering
 \includegraphics[width=1.0\linewidth]{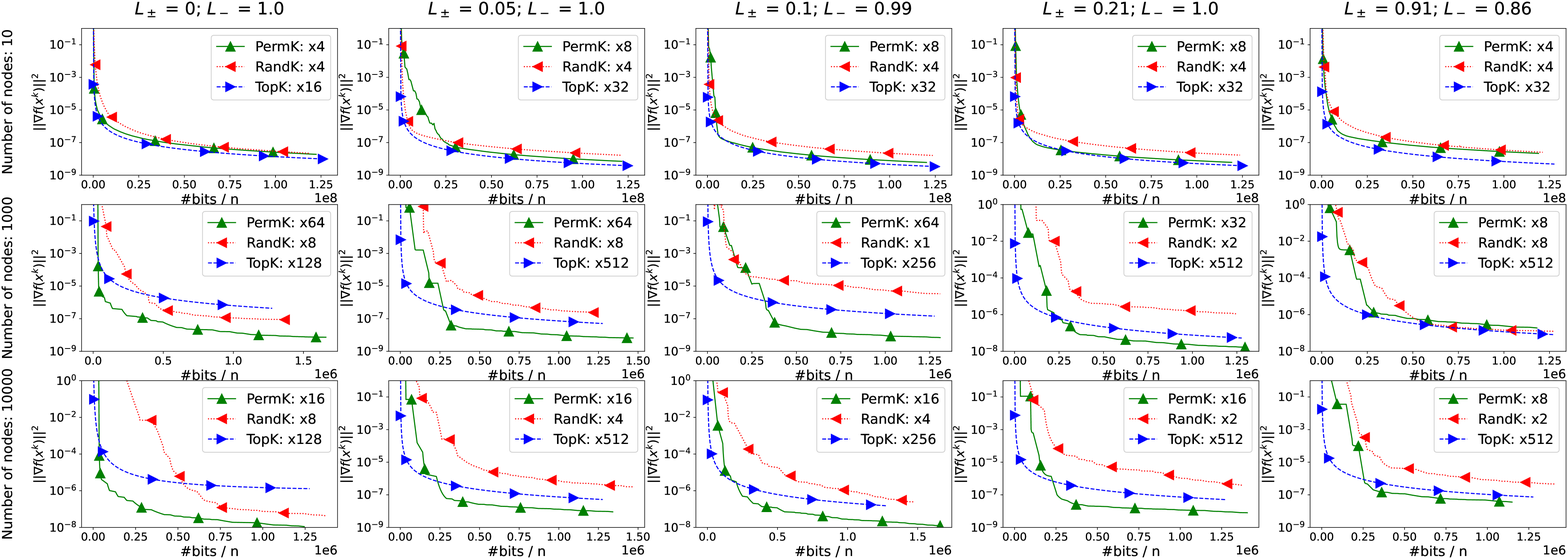}
\vspace*{-0.4cm}
\caption{Comparison of algorithms on synthetic quadratic optimization tasks with nonconvex $\{f_i\}$.}
 \label{fig:project-permutation-compression-marina-quadratic-with-norms_best}
\end{figure}

\subsection{Training an autoencoder with MNIST}

\newcommand{\experimetscaptionmnist}[1]{Comparison of #1on the encoding learning task for the MNIST dataset.
Each row corresponds to a fixed regularization parameter $\lambda$; 
each column corresponds to a fixed probability $\widehat{p}$.
In the legends, we provide compressor names and fine-tuned step sizes. 
Axis $x$ represents the number of bits that every node has sent. } 

\label{section:experiment:mnist}

Now we compare compressors from Section \ref{section:experiment:quadratic} on the  MNIST dataset \citep{lecun2010mnist}. 
%MNIST is a dataset of 60000 handwritten image digits between 0 and 9. 
Our current goal is to learn the linear autoencoder,
\begin{eqnarray*}
  \textstyle \min \limits_{\mD \in \R^{d_f \times d_e}, \mE \in \R^{d_e \times d_f}} \left[ f(\mD, \mE) \eqdef \frac{1}{N}\sum \limits_{i = 1}^N \norm{\mD \mE a_i - a_i}^2 \right] ,
\end{eqnarray*}
where $a_i \in \R^{d_f}$ are MNIST images, $d_f = 784$ is the number of features, 
$d_e = 16$ is the size of encoding space.
Thus the dimension of the problem $d = 25088,$
and compressors send at most $26$ floats in each communication round since we take $n = 1000.$
We use parameter $\widehat{p}$ to control the homogeneity of MNIST split among $n$ nodes:
if $\widehat{p} = 1$, then all nodes store the same data, and if $\widehat{p} = 0$, then nodes store different splits 
(see Appendix~\ref{section:extra_experiments:mnist}). 
In Figure~\ref{fig:project-permutation-compression-marina-mnist-auto_encoder-1000-nodes-prob-homog-P-reg-P-ef21-init-grad-xavier-normal-with-randk_best},
one plot for each compressor with the best convergence rate is provided for $\widehat{p} \in \{0, 0.5, 0.75, 0.9, 1.0\}.$
We choose parameters of algorithms prescribed by the theory except for the step sizes, where we performed a grid search as before. In all experiments, Perm$K$ outperforms Rand$K$. Moreover, we see that in  the more homogeneous regimes, when $\widehat{p} \in \{0.9, 1.0\}$, Perm$K$ converges faster than Top$K$. When $\widehat{p} = 0.75$, both compressors have almost the same performance. In the heterogenous regime, when $\widehat{p} \in \{0, 0.5\}$, Top$K$ is faster than Perm$K$; however, the difference between them is tiny compared to Rand$K$.

\vspace*{-0.2cm}
\begin{figure}[!h]
  \centering
  \includegraphics[width=0.95\linewidth]{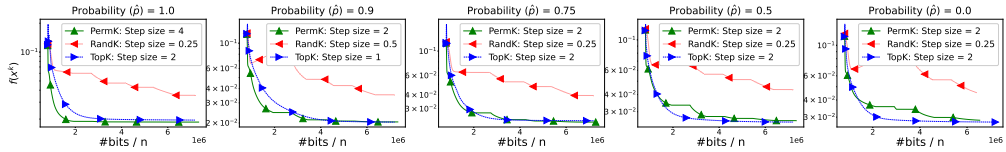}
  \vspace*{-0.3cm}
  \caption{Comparison of algorithms on the encoding learning task for the MNIST dataset.}
  \label{fig:project-permutation-compression-marina-mnist-auto_encoder-1000-nodes-prob-homog-P-reg-P-ef21-init-grad-xavier-normal-with-randk_best}
\end{figure}

\bibliography{iclr2022_conference}
\bibliographystyle{iclr2022_conference}

\newpage

\appendix

\part*{Appendix}

\tableofcontents

%%%%%%%%%%%%%%%%%%%%%%%%%%%%
\newpage
\section{Extra Experiments}
\label{section:extra_experiments}
%%%%%%%%%%%%%%%%%%%%%%%%%%%%

In this section, we provide more detailed experiments and explanations.

\subsection{Experiments setup}

\label{section:extra_experiments:setup}

All methods are implemented in Python 3.6 and run on a machine with 24 Intel(R) Xeon(R) Gold 6146 CPU @ 3.20GHz cores with 32-bit precision. Communication between master and nodes is emulated in one machine.

In all experiments, we compare \algname{MARINA} algorithm with Rand$K$ compressor and Perm$K$ compressor 
{and \algname{EF21} with Top$K$}
side-by-side. In Rand$K$ and Top$K$, we take $K = \ceil{d / n}$; we show in Lemma~\ref{lemma:optimal_parameters_of_methods}  that $K = \ceil{d / n}$ is optimal for Rand$K$. 
For Top$K$, the optimal rate predicted by the current state-of-the-art theory is obtained when $K = d$ (however, in practice, Top$K$ works much better when $K \ll d$).
Lastly, we have the pessimistic assumption that 
$\Lpm^2$ and $L_{+}^2$ are equal to their upper bound $\frac{1}{n}\sum_{i=1}^n L_i^2.$

\subsection{Experiment with quadratic optimization tasks: full description}

First, we present Algorithm~\ref{algorithm:matrix_generation} which is used in the experiments of Section~\ref{section:experiment:quadratic}. 
The algorithm is designed to generate sparse quadratic optimization tasks where 
we can control $\Lpm$ using the noise scale.
Furthermore, it can be seen that the procedure generates strongly convex quadratic optimization tasks; 
thus, all assumptions from this paper are fulfilled to use theoretical results.

\begin{algorithm}[!h]
  \caption{Quadratic optimization task generation}
  \begin{algorithmic}[1]
  \label{algorithm:matrix_generation}
  \STATE \textbf{Parameters:} number nodes $n$, dimension $d$, regularizer $\lambda$, and noise scale $s$.
  \FOR{$i = 1, \dots, n$}
  \STATE Generate random noises $\nu_i^s = 1 + s \xi_i^s$ and $\nu_i^b = s \xi_i^b,$ i.i.d. $\xi_i^s, \xi_i^b \sim \cN(0, 1)$
  \STATE Take vector $b_i = \frac{\nu_i^s}{4}(-1 + \nu_i^b, 0, \cdots, 0) \in \R^{d}$
  \STATE Take the initial tridiagonal matrix
  \[\mA_i = \frac{\nu_i^s}{4}\left( \begin{array}{cccc}
    2 & -1 & & 0\\
    -1 & \ddots & \ddots & \\
    & \ddots & \ddots & -1 \\
    0 & & -1 & 2 \end{array} \right) \in \R^{d \times d}\]
  \ENDFOR
  \STATE Take the mean of matrices $\mA = \frac{1}{n}\sum_{i=1}^n \mA_i$
  \STATE Find the minimum eigenvalue $\lambda_{\min}(\mA)$
  \FOR{$i = 1, \dots, n$}
  \STATE Update matrix $\mA_i = \mA_i + (\lambda - \lambda_{\min}(\mA)) \mI$
  \ENDFOR
  \STATE Take starting point $x^0 = (\sqrt{d}, 0, \cdots, 0)$
  \STATE \textbf{Output:} matrices $\mA_1, \cdots, \mA_n$, vectors $b_1, \cdots, b_n$, starting point $x^0$
  \end{algorithmic}
\end{algorithm}

Homogeneity of optimizations tasks is controlled by noise scale $s$; 
indeed, with noise scale equal to zero, all matrices are equal, and, 
by increasing noise scale, functions become less ``similar'' and $\Lpm^2$ grows. In Section~\ref{section:experiment:quadratic}, we take noise scales $s \in \{0, 0.05, 0.1, 0.2, 0.8\}.$

In Figure \ref{fig:project-permutation-compression-marina-quadratic-with-norms_best_pl}, 
we provide the same experiments as in Section~\ref{section:experiment:quadratic} but with $\lambda = 0.0001$ to capture dependencies under P\L\,condition.
Here, we also see that Perm$K$ has better performance when the number of nodes $n \geq 1000$ and $L_\pm \leq 0.21$.

\subsection{Comparison of \algname{MARINA} with Rand$K$ and \algname{MARINA} with Perm$K$ on quadratic optimization problems}

\label{section:extra_experiments:quad_prog:rand_k}

In this section, we provide detailed experiments from Section~\ref{section:experiment:quadratic} and comparisons of Rand$K$ and 
Perm$K$ with different step sizes 
(see Figure~\ref{fig:project-permutation-compression-marina-quadratic-with-norms} and 
Figure~\ref{fig:project-permutation-compression-marina-quadratic-with-norms-pl}). 
We omitted plots where algorithms diverged. 
We can see that in all experiments, Perm$K$ behaves better than Rand$K$ and tolerates larger step sizes. The improvement becomes more significant when $n$ increases.

\subsection{Comparison of \algname{EF21} with Top$K$ and \algname{MARINA} with Perm$K$ on quadratic optimization problems}

In this section, we provide detailed experiments from Section~\ref{section:experiment:quadratic} 
and comparisons of \algname{EF21} with Top$K$ and 
\algname{MARINA} with Perm$K$ with different step sizes 
(see Figure~\ref{fig:project-permutation-compression-marina-quadratic-with-norms_2} 
and Figure~\ref{fig:project-permutation-compression-marina-quadratic-with-norms_2-pl}).
We omitted plots where algorithms diverged. 
As we can see, when $L_{\pm} \leq 0.21$ and $n \geq 10000$, Perm$K$ converges faster than Top$K$. 
While in heterogeneous regimes, when $L_{\pm}$ is large, Top$K$ has better performance except when $n = 10000$.
When $n > d$, we see that Perm$K$ converges faster in all experiments.

\begin{figure}
  \centering
  \includegraphics[width=1.0\linewidth]{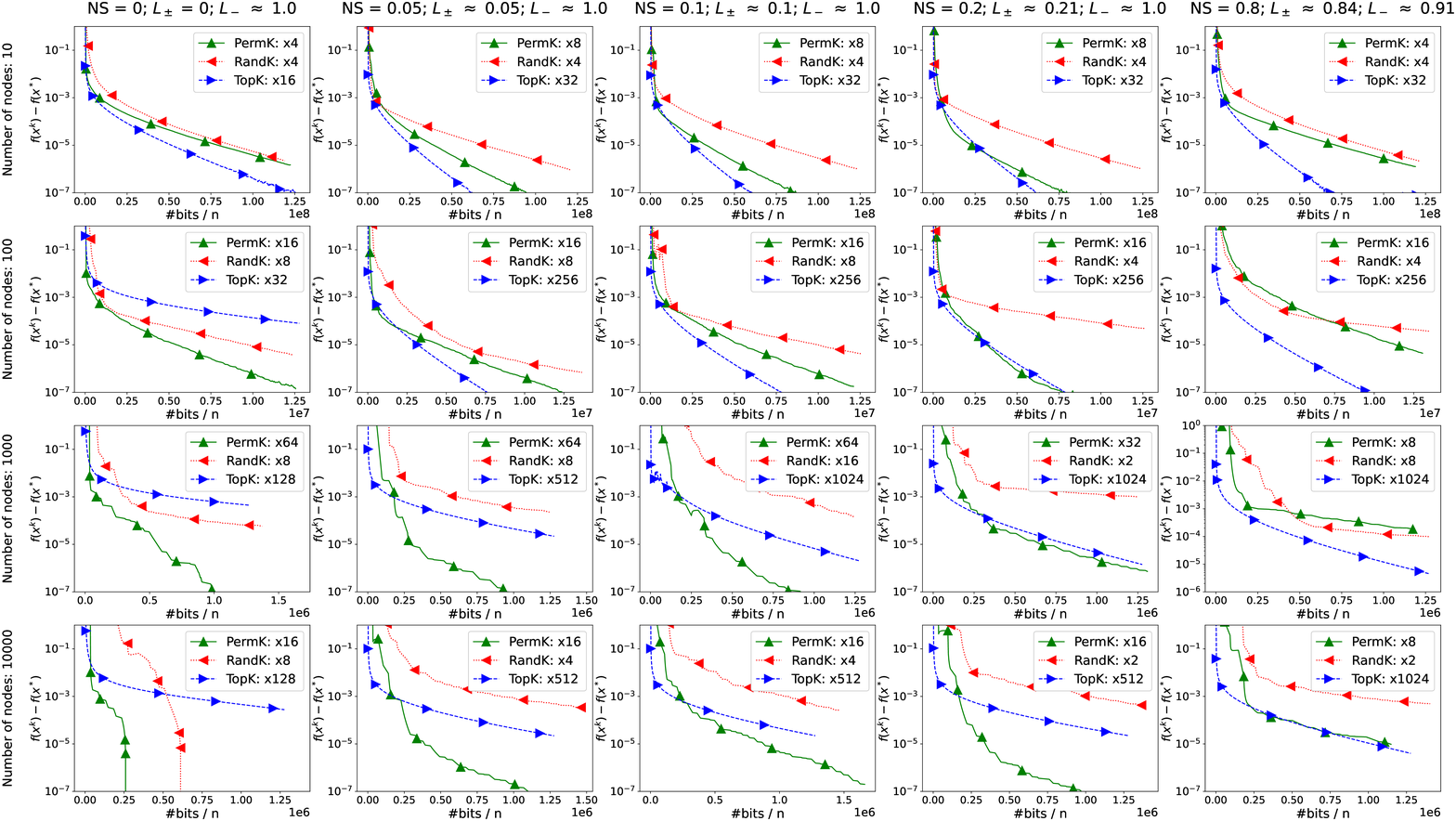}
  \vspace*{-1.0cm}
  \caption{\experimetscaption{algorithms under P\L\,condition }}
  \label{fig:project-permutation-compression-marina-quadratic-with-norms_best_pl}
\end{figure}

\begin{figure}
  \centering
  \includegraphics[width=0.99\linewidth]{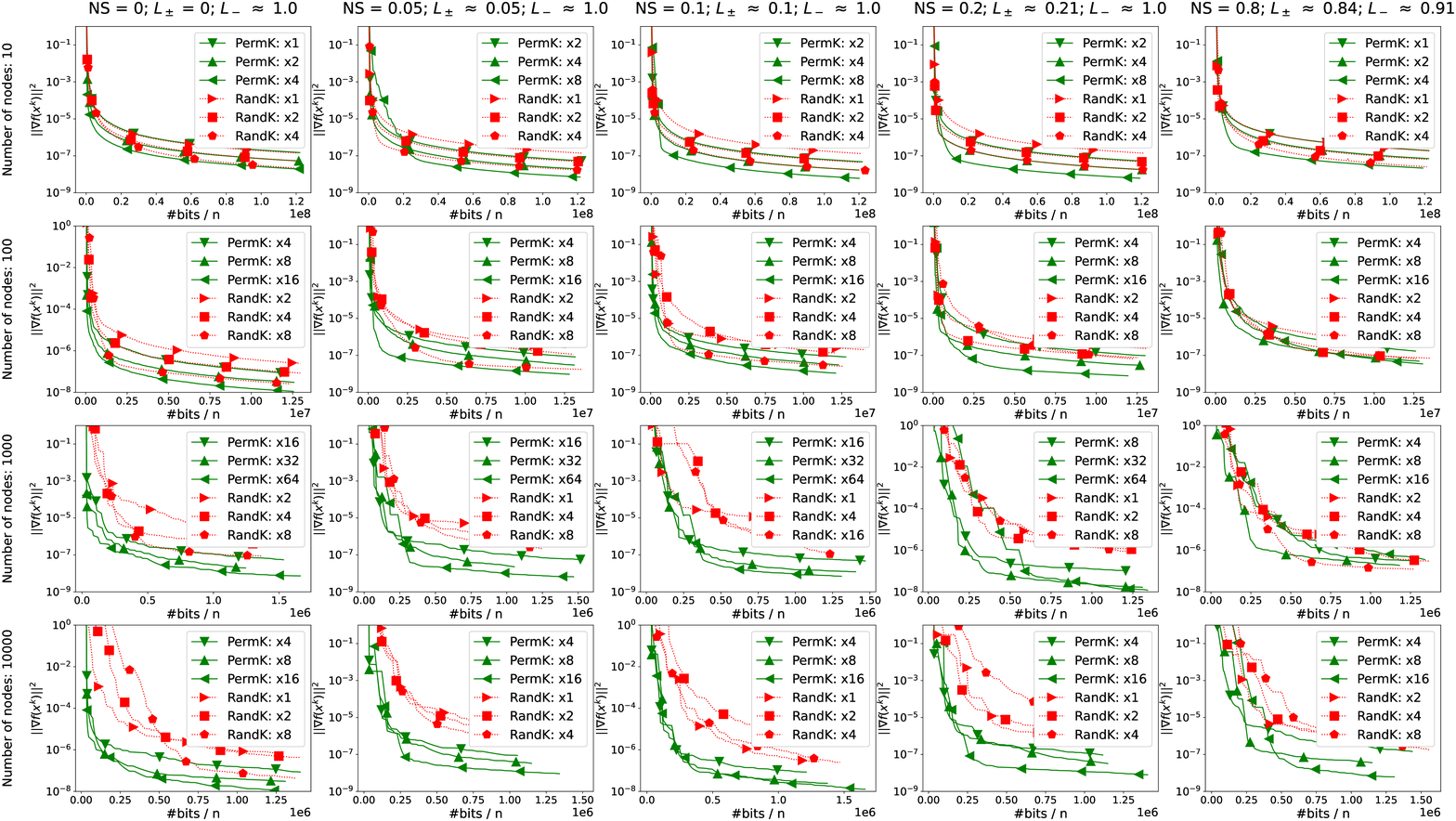}
  \vspace*{-0.5cm}
  \caption{\experimetscaption{Rand$K$ and Perm$K$}}
  \label{fig:project-permutation-compression-marina-quadratic-with-norms}

  \vspace{\floatsep}

  \centering
  \includegraphics[width=0.99\linewidth]{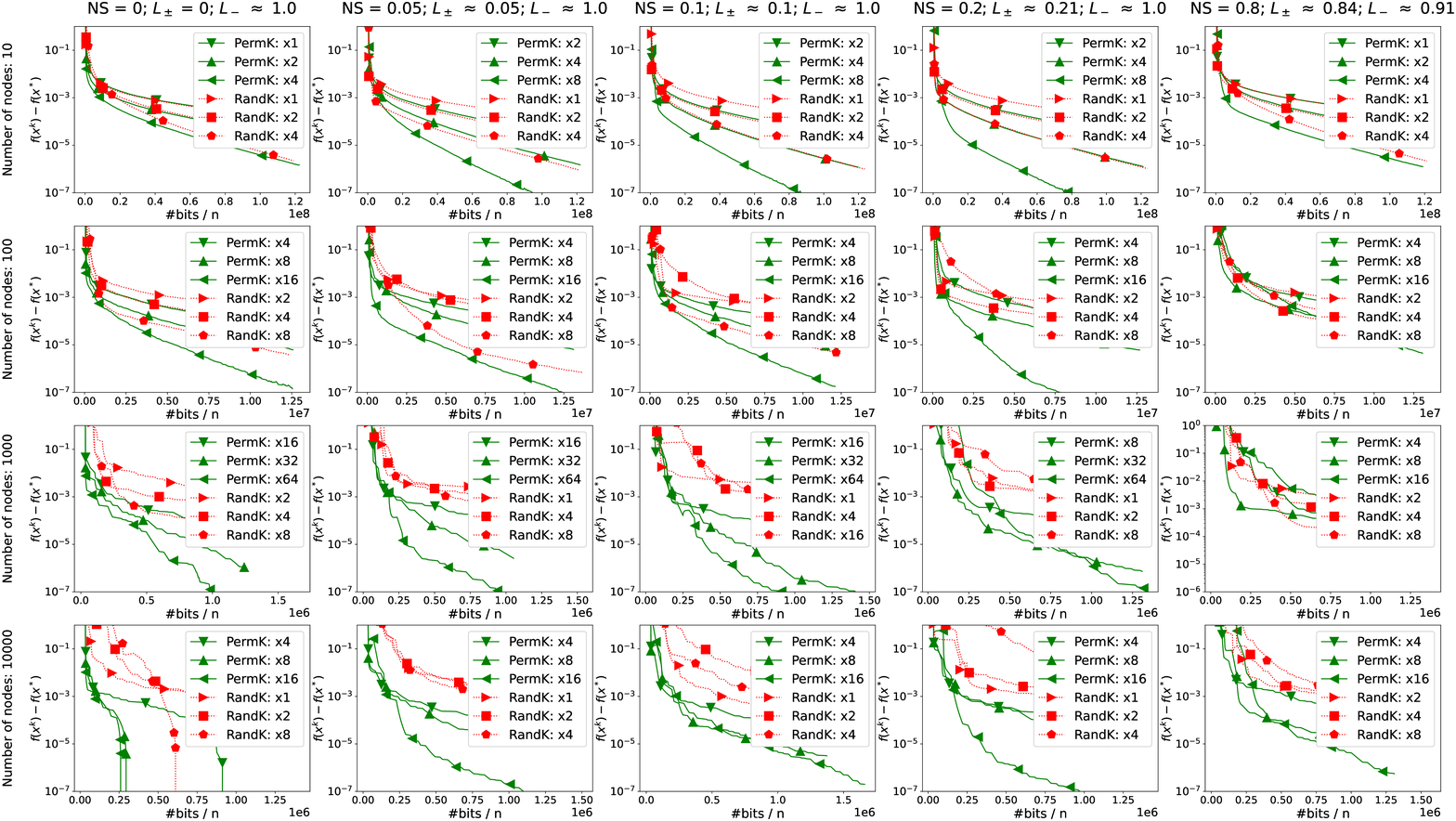}
  \vspace*{-0.5cm}
  \caption{\experimetscaption{Rand$K$ and Perm$K$ under P\L\,condition }}
  \label{fig:project-permutation-compression-marina-quadratic-with-norms-pl}
\end{figure}

\begin{figure}
  \centering
  \includegraphics[width=0.99\linewidth]{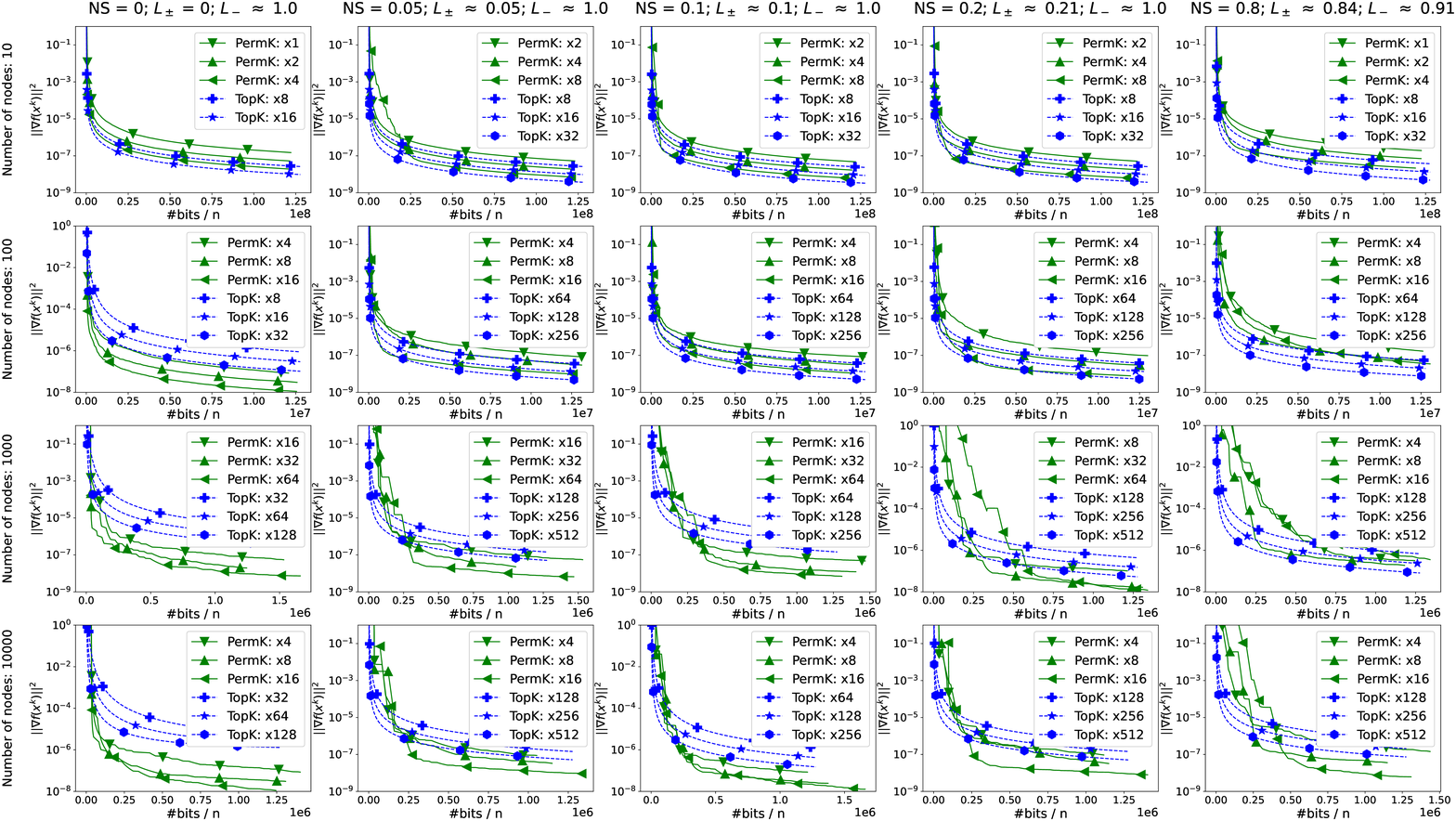}
  \vspace*{-0.5cm}
  \caption{\experimetscaption{Top$K$ and Perm$K$}}
  \label{fig:project-permutation-compression-marina-quadratic-with-norms_2}

  \vspace{\floatsep}

  \centering
  \includegraphics[width=0.99\linewidth]{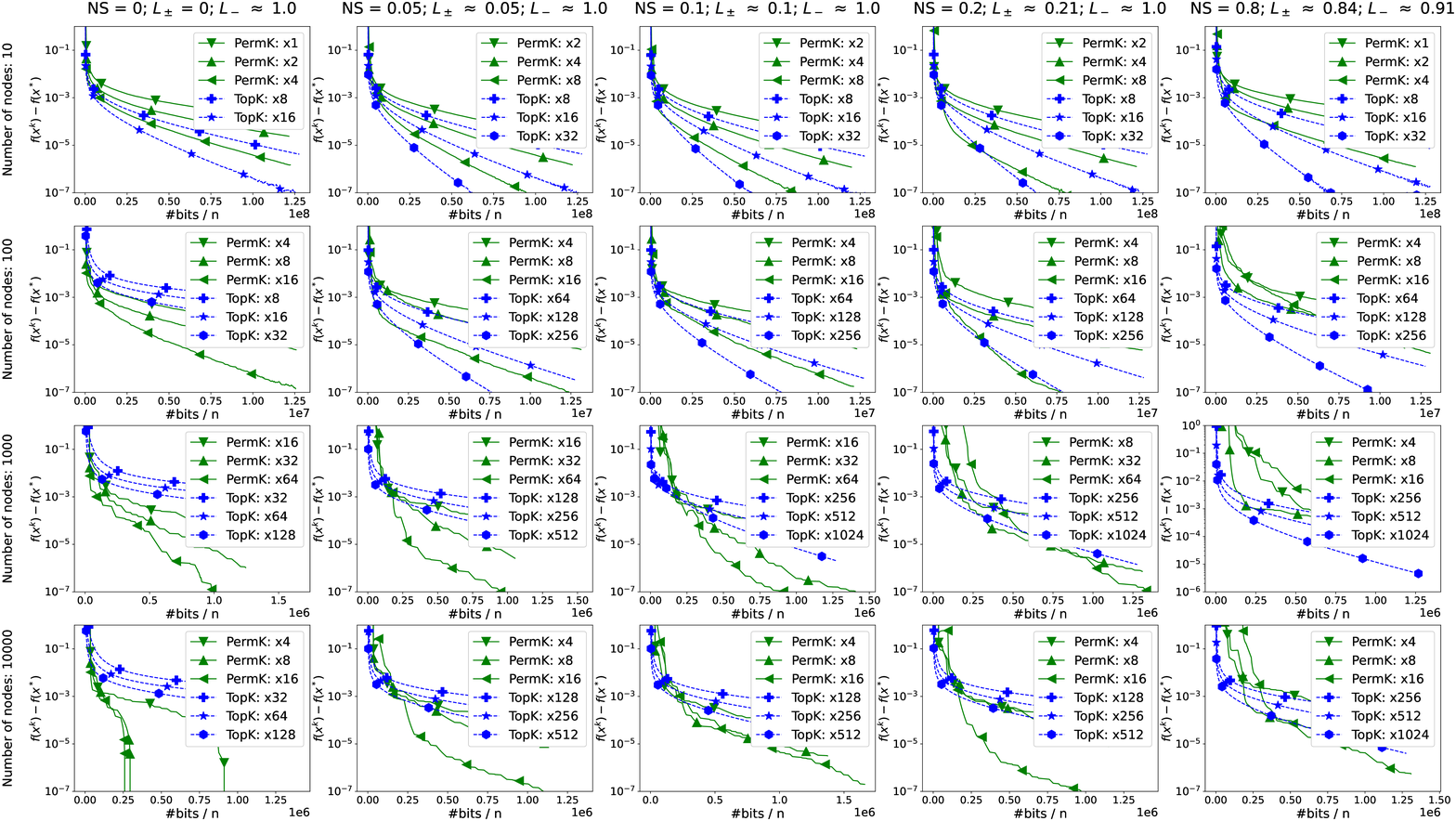}
  \vspace*{-0.5cm}
  \caption{\experimetscaption{Top$K$ and Perm$K$ under P\L\,condition }}
  \label{fig:project-permutation-compression-marina-quadratic-with-norms_2-pl}
\end{figure}

\clearpage

\subsection{Experiment with MNIST: full description}
\label{section:extra_experiments:mnist}

We introduce parameter $\widehat{p}$.
Initially, we randomly split MNIST into $n + 1$ parts: $D_0, D_1, \cdots, D_n$, where $n = 1000$ is the number of nodes. 
Then, for all $i \in \{1,\dots,n\}$, the $i$\ts{th} node takes split $D_0$ with probability $\widehat{p}$, 
or split $D_i$ with probability $1 - \widehat{p}$. We define the chosen split as $\widehat{D_i}$. 
Using probability $\widehat{p}$, we control the homogeneity of our distribution optimization task.
Note that if $\widehat{p} = 1$, all nodes store the same data $D_0$, 
and if $\widehat{p} = 0$, nodes store different splits $D_i$. 

Let us consider the more general optimization problem than in Section~\ref{section:experiment:mnist}.
We optimize the following non-convex loss with regularization:
\begin{eqnarray*}
  \textstyle \min \limits_{\mD \in \R^{d_f \times d_e}, \mE \in \R^{d_e \times d_f}} \left[ f(\mD, \mE) \eqdef \frac{1}{N}\sum \limits_{i = 1}^N \norm{\mD \mE a_i - a_i}^2 + \frac{\lambda}{2}\norm{\mD \mE - \mI}_F^2 \right] ,
\end{eqnarray*}
where $a_i \in \R^{d_f}$ are MNIST images, $d_f = 784$ is the number of features, 
$d_e = 16$ is the size of encoding space.
regularizer $\lambda \geq 0.$

Each node stores function
\begin{eqnarray*}
  f_i(\mD, \mE) \eqdef \frac{1}{|\widehat{D_i}|}\sum_{j \in \widehat{D_i}} \norm{\mD \mE a_j - a_j}^2 + \frac{\lambda}{2}\norm{\mD \mE - \mI}_F^2, \quad \forall i \in \{1, \dots, n\}.
\end{eqnarray*}

In Figure~\ref{fig:project-permutation-compression-marina-mnist-auto_encoder-1000-nodes-prob-homog-P-reg-P-ef21-init-grad-xavier-normal-with-randk_best_full}, one plot for each compressor with the best convergence rate is provided for $\lambda = \{0, 0.00001, 0.001\}$ and $\widehat{p} = \{0, 0.5, 0.75, 0.9, 1.0\}.$ 

We see that in homogeneous regimes, when $\widehat{p} \in \{0.9, 1.0\}$, 
Perm$K$ outperforms other compressors for any $\lambda$. And the larger the regularization parameter $\lambda$, the faster Perm$K$ convergences compared to rivals.

\begin{figure}[!h]
  \centering
  \includegraphics[width=0.95\linewidth]{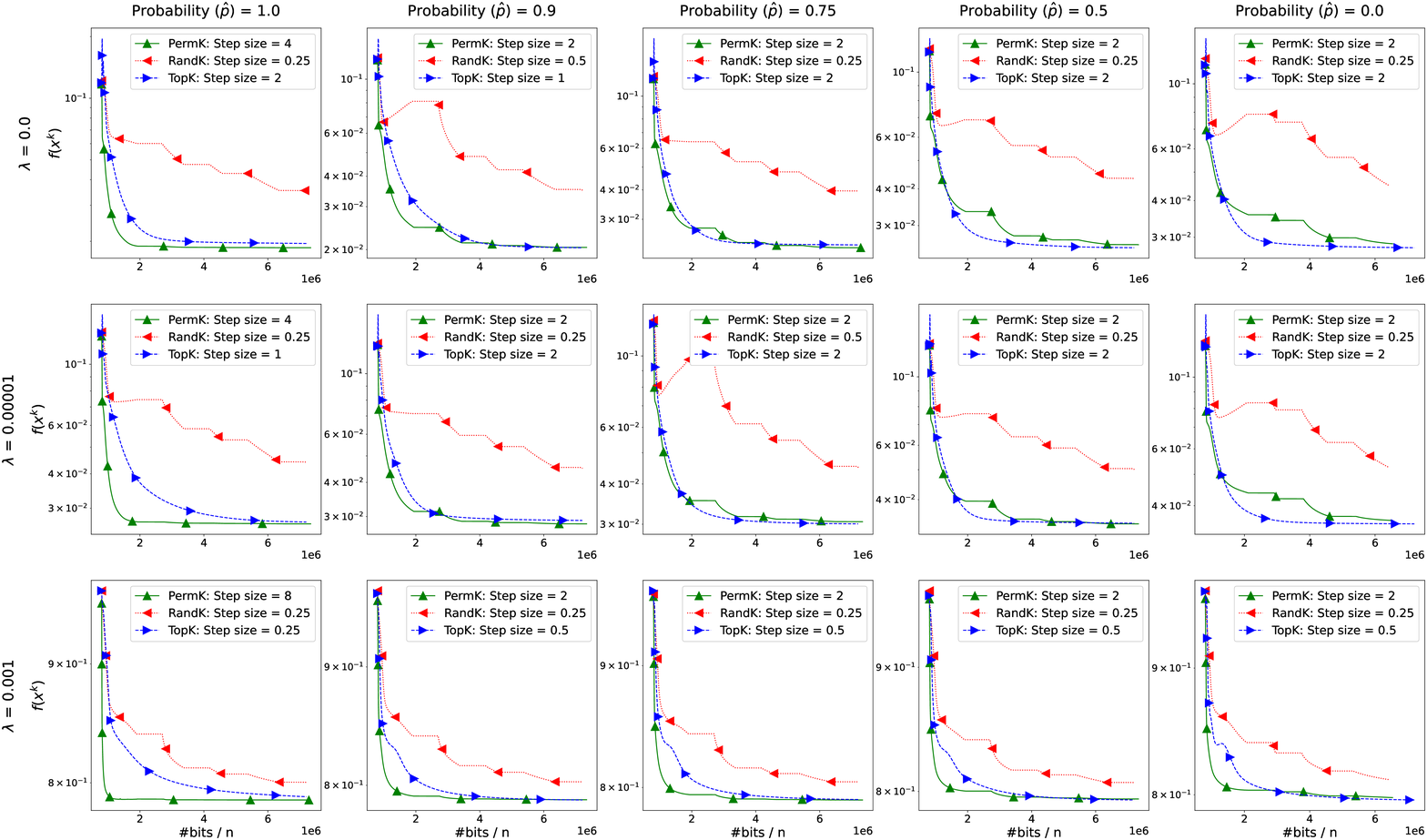}
  \vspace*{-0.5cm}
  \caption{\experimetscaptionmnist{algorithms }}
  \label{fig:project-permutation-compression-marina-mnist-auto_encoder-1000-nodes-prob-homog-P-reg-P-ef21-init-grad-xavier-normal-with-randk_best_full}
\end{figure}

\subsection{Comparison of \algname{MARINA} with Rand$K$ and \algname{MARINA} with Perm$K$ on MNIST dataset}

In this section, we provide detailed experiments from Section~\ref{section:experiment:mnist} and comparisons of Rand$K$ and 
Perm$K$ with different step sizes 
(see Figure~\ref{fig:project-permutation-compression-marina-mnist-auto_encoder-1000-nodes-prob-homog-P-reg-P-ef21-init-grad-xavier-normal-with-randk}).
We omitted plots where algorithms diverged. 
We see that in all experiments, Perm$K$ is better than Rand$K$.
Practical experiments on MNIST fully reproduce dependencies from our theory and experiments 
with synthetic quadratic optimization tasks from Section~\ref{section:experiment:quadratic}.

\begin{figure}[!h]
  \centering
  \includegraphics[width=1.0\linewidth]{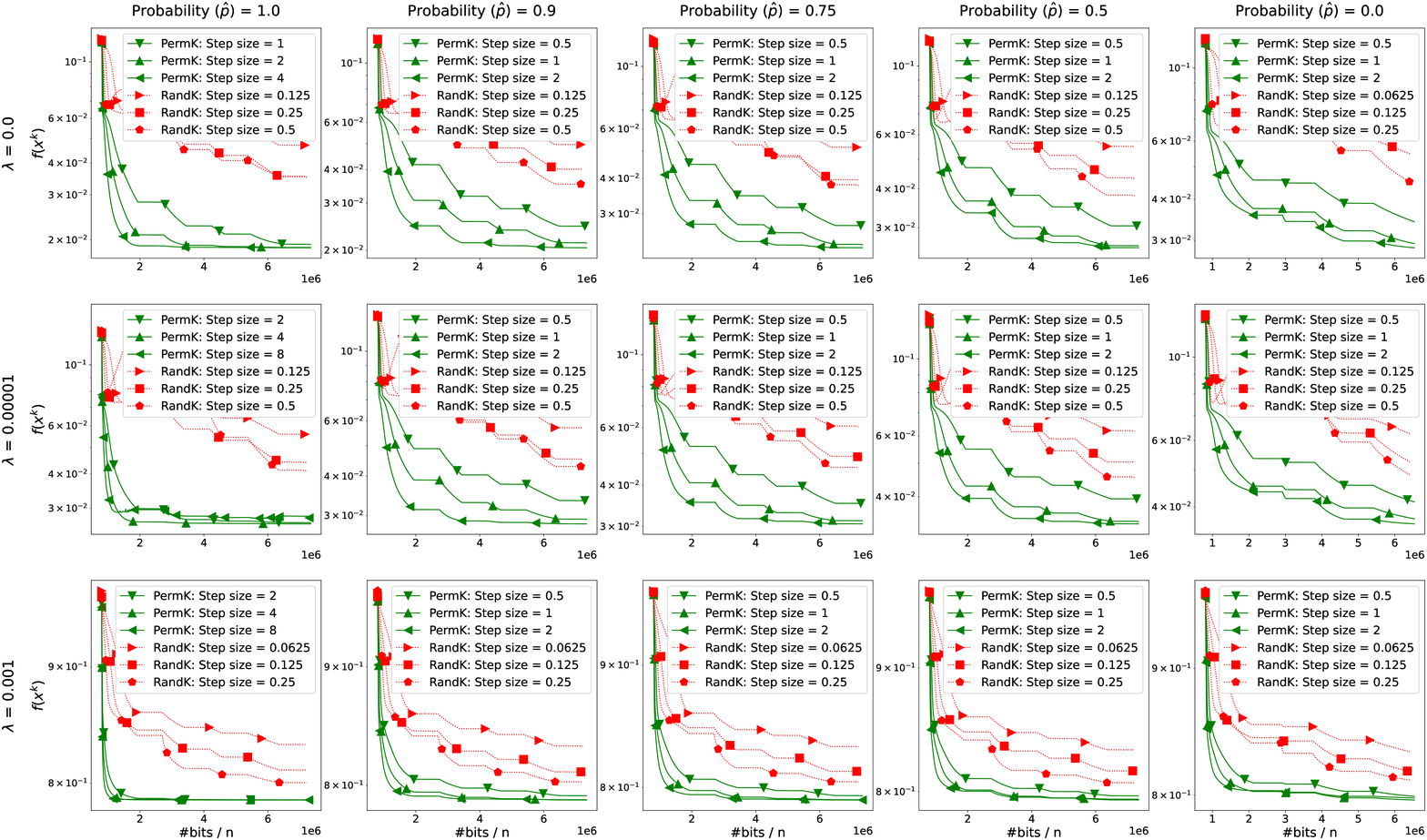}
  \vspace*{-1.0cm}
  \caption{\experimetscaptionmnist{Rand$K$ and Perm$K$}}
  \label{fig:project-permutation-compression-marina-mnist-auto_encoder-1000-nodes-prob-homog-P-reg-P-ef21-init-grad-xavier-normal-with-randk}
\end{figure}

\subsection{Comparison of \algname{EF21} with Top$K$ and \algname{MARINA} with Perm$K$ on MNIST dataset}

In this section, we provide detailed experiments from Section~\ref{section:experiment:mnist} and comparisons of Rand$K$ and 
Perm$K$ with different step sizes 
(see Figure~\ref{fig:project-permutation-compression-marina-mnist-auto_encoder-1000-nodes-prob-homog-P-reg-P-ef21-init-grad-xavier-normal-with-randk_2}).
We omitted plots where algorithms diverged.
We see that, when $\widehat{p} \in \{0.9, 1.0\}$,
Perm$K$ tolerates larger step sizes and convergences faster than Top$K$. 
When a probability $\widehat{p} \in \{0, 0.5\}$, 
both compressors approximately tolerate the same step sizes, but
Top$K$ has a better performance when $\lambda \in \{0, 0.00001\}$.

\begin{figure}[!h]
  \centering
  \includegraphics[width=1.0\linewidth]{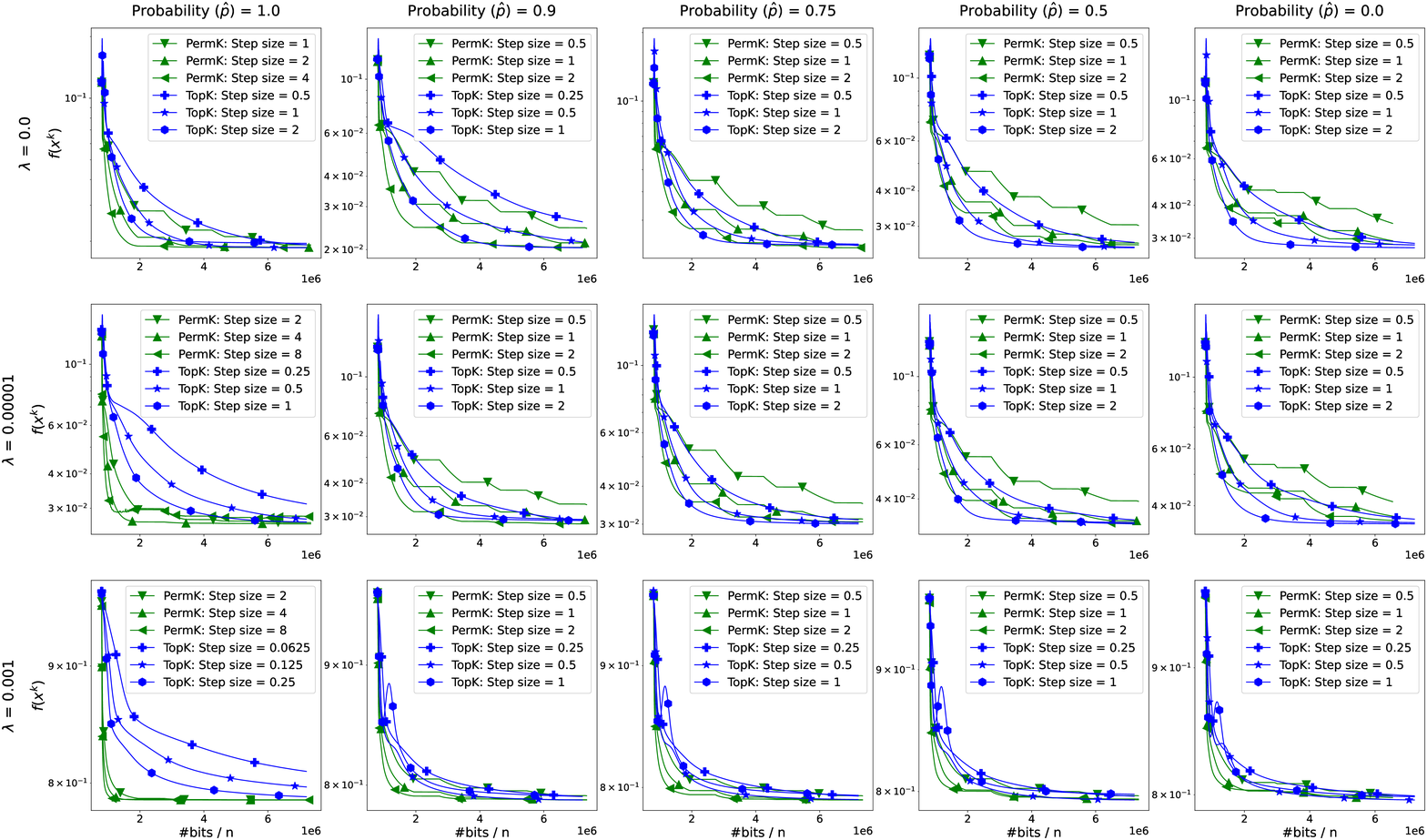}
  \vspace*{-1.0cm}
  \caption{\experimetscaptionmnist{Top$K$ and Perm$K$}}
  \label{fig:project-permutation-compression-marina-mnist-auto_encoder-1000-nodes-prob-homog-P-reg-P-ef21-init-grad-xavier-normal-with-randk_2}
\end{figure}

%%%%%%%%%%%%%%%%%%%%%%%%%
%%%%%%%%%%%%%%%%%%%%%%%%%
\clearpage
\section{\algname{MARINA} Algorithm}
%%%%%%%%%%%%%%%%%%%%%%%%%
%%%%%%%%%%%%%%%%%%%%%%%%%
\label{section:marina_pseudocode}
To the best of our knowledge, the state-of-the-art method for solving the nonconvex problem (\ref{eq:main}) in terms of the theoretical communication efficiency is \algname{MARINA} \citep{MARINA}. In its simplest variant, \algname{MARINA} performs iterations of the form
\begin{equation} \label{eq:MARINA-generic} x^{k+1} = x^k - \gamma g^k, \qquad g^k = \frac{1}{n}\sum_{i=1}^n g_i^k,\end{equation}
where $g_i^k$ is a carefully designed {\em biased} estimator of the gradient $\nabla f_i(x^k)$, and $\gamma>0$ is a learning rate. The gradient estimators used in \algname{MARINA} are initialized to the full gradients, i.e., $g_i^0 = \nabla f_i(x^0)$, for $i \in \{1,\dots,n\}$, and subsequently updated as
\[g_i^{k+1} = \begin{cases} \nabla f_i(x^{k+1}) & \text{if} \quad \theta_k=1 \\
g^k + \cC_i^k( \nabla f_i(x^{k+1}) - \nabla f_i(x^k)) & \text{if} \quad \theta_k=0 
\end{cases},\]
where $\theta_k$ is a Bernoulli random variable sampled at iteration $k$ (equal to $1$ with probability $p \in (0, 1]$, and equal to $0$ with probability $1-p$), and 
$\cC_i:\R^d\to \R^d$ is a randomized compression operator sampled at iteration $k$ on node $i$ {\em independently} from other nodes. In particular, \citet{MARINA} assume that the compression operators $\cC_i$ are {\em unbiased}, and that their variance is proportional to squared norm of the input vector:
\[ \Exp{\cC_i(x)} = x, \qquad \Exp{\norm{\cC_i(x)-x}^2} \leq \omega_i \norm{x}^2, \qquad \forall x\in \R^d. \]

%It is clear that $g^k$ is not necessarily an unbiased estimator of the gradient $\nabla f(x^k)$.

In each iteration of \algname{MARINA}, the gradient estimator is reset to the true gradient with (small) probability $p$. Otherwise, each worker $i$ compresses the difference of the last two local gradients, and communicates the compressed message $$m_i^k =  \cC_i^k( \nabla f_i(x^{k+1}) - \nabla f_i(x^k))$$ to the server. These messages are then aggregated by the server to form the new gradient estimator via
\[g^{k+1} = g^k + \frac{1}{n}\sum_{i=1}^n m_i^k.\]
Note that i) this preserves the second relation in (\ref{eq:MARINA-generic}), ii) the server {\em can} compute $g^{k+1}$ since it has access to $g^k$, which is the case (via a recursive argument) if  $g^0$ is known by the server at the start of the iterative process\footnote{This is done by each worker sending the full gradient $g_i^0 = \nabla f_i(x^0)$ to the server at initialization.}.

\begin{algorithm}
  \caption{\algname{MARINA}}
  \label{alg:marina}
  \begin{algorithmic}[1]
  \STATE \textbf{Input:} starting point $x^0$, stepsize $\gamma$, probability $p \in (0, 1]$, number of iterations $T$
  \STATE Initialize $g^0 = \nabla f(x^0)$
  \FOR{$k = 0, 1, \dots, T - 1$}
  \STATE Sample $\theta_t \sim \textnormal{Be}(p)$
  \STATE Broadcast $g^t$ to all workers
  \FOR{$i = 1, \dots, n$ in parallel}
  \STATE $x^{t+1} = x^t - \gamma g^t$
  \STATE Set $g^{t+1}_i = \nabla f_i(x^{t+1})$ if $\theta_t = 1,$ 
  and $g^{t+1}_i = g^t + \cC_i\left(\nabla f_i(x^{t+1}) - \nabla f_i(x^{t})\right)$ otherwise
  \label{alg:gradient_estimate_definition}
  \ENDFOR
  \STATE $g^{t+1} = \frac{1}{n} \sum_{i=1}^{n} g^{t+1}_i$
  \ENDFOR
  \STATE \textbf{Output:} $\hat{x}^T$ chosen uniformly at random from $\{x^t\}_{k=0}^{T-1}$
  \end{algorithmic}
\end{algorithm}

Further, note that the {\em expected communication cost} in each iteration of \algname{MARINA} is equal to
\[\ {\rm Comm} = p d + (1-p) \zeta, \qquad \zeta =\max_i \zeta_i,
\]
where $d$ is the cost of communicating a (possibly dense) vector in $\R^d$, and $\zeta_i \leq d$ is the expected cost of communicating a vector compressed by $\cC_i$.

\algname{MARINA} one of the very few examples in stochastic optimization where the use of a biased estimator leads to a better theoretical complexity than the use of an unbiased estimator, with the other example being optimal \algname{SGD} methods for single-node problems \algname{SARAH}~\citep{SARAH}, \algname{SPIDER}~\citep{SPIDER},  \algname{PAGE}~\citep{PAGE}.

%%%%%%%%%%%%%%%%%%%%%
%%%%%%%%%%%%%%%%%%%%%
\clearpage
\section{Missing Proofs}
%%%%%%%%%%%%%%%%%%%%%
%%%%%%%%%%%%%%%%%%%%%
\label{section:missing_proofs}

%%%%%%%%%%%%%%%%%%%%%%%%%%%%%%%%%%%%%
\subsection{Proof of Lemma~\ref{lem:bg97fd890d}}
%%%%%%%%%%%%%%%%%%%%%%%%%%%%%%%%%%%%%

\unbiased*
\begin{proof} 
  Let us first assume unbiasedness only. By Jensen's inequality,
  $$\norm{ \frac{1}{n} \sum_{i=1}^n \cC_i(a_i) - \frac{1}{n}\sum_{i=1}^n a_i}^2 \leq  \frac{1}{n} \sum_{i=1}^n \norm{ \cC_i(a_i) - a_i}^2.$$
  It remains to apply expectation on both sides and then use inequality
  $$\Exp{\norm{ \cC_i(a_i) - a_i}^2} \leq \omega_i \norm{a_i}^2, \forall i \in \{1, \dots, n\},$$
  to conclude that $\{\cC_i\}_{i = 1}^n \in \mathbb{U}(\max_i\omega_i, 0)$.

  Let us now add the assumption of independence.
  \begin{align*}
    &\Exp{\norm{\frac{1}{n}\sum\limits_{i=1}^n\cC_i(a_i) - \frac{1}{n}\sum\limits_{i=1}^n a_i}^2} \\
    &= \Exp{\norm{\frac{1}{n}\sum\limits_{i=1}^n\left(\cC_i(a_i) - a_i\right)}^2}\\
    &= \frac{1}{n^2}\sum\limits_{i=1}^n\Exp{\norm{\cC_i(a_i) - a_i}^2} + \frac{1}{n^2}\sum\limits_{i\neq j}\Exp{\inp{\cC_i(a_i) - a_i}{\cC_j(a_j) - a_j}}\\
    &\leq \frac{\max_i \omega_i}{n^2} \sum\limits_{i=1}^n\norm{a_i}^2,
  \end{align*}
  by independence, thus, $A = \max_i \omega_i / n$, $B = 0$.
\end{proof}

%%%%%%%%%%%%%%%%%%%%%%%%%%%%%%%%%%%%%
\subsection{Proof of Lemma~\ref{lem:08hhdf}}
%%%%%%%%%%%%%%%%%%%%%%%%%%%%%%%%%%%%%

\ineq*
\begin{proof} 
  Let us define 
  \begin{eqnarray*}
    \cL_{-}(x,y) &\eqdef& \norm{\nabla f(x) - \nabla f(y)}^2, \\
    \cL_{+}(x,y) &\eqdef& \frac{1}{n} \sum_{i=1}^n \norm{\nabla f_i(x) - \nabla f_i(y)}^2,\\
    \cL_{\pm}(x,y) &\eqdef& \cL_{+}(x,y) - \cL_{-}(x,y).
  \end{eqnarray*}
  The inequalities are now established as follows:
\begin{enumerate}
  \item By Jensen's inequality and the definition of $L_{+}$,
  \begin{eqnarray*}
    \cL_{-}(x,y) \leq \cL_{+}(x,y) \leq L_{+}^2 \norm{x-y}^2,
  \end{eqnarray*}
  thus, $L_{-}$ is at most $L_{+}$.
  \item By the triangle inequality, we have
  \begin{eqnarray*}
    \norm{\nabla f(x) - \nabla f(y)} \leq \frac{1}{n} \sum_{i=1}^n \norm{\nabla f_i(x) - \nabla f_i(y)} 
    \leq \frac{1}{n} \sum_{i=1}^n L_i \norm{x-y}, 
  \end{eqnarray*}
  thus $L_{-}$ is at most $\frac{1}{n} \sum_{i=1}^n L_i$.
  \item From the definition of $L_i$, we have
  \begin{eqnarray*}
    \cL_{+}(x,y) = \frac{1}{n} \sum_{i=1}^n \norm{\nabla f_i(x) - \nabla f_i(y)}^2 
    \leq \frac{1}{n} \sum_{i=1}^n L_i^2 \norm{x-y}^2,
  \end{eqnarray*}
  and $L_{+}^2$ is at most $\frac{1}{n} \sum_{i=1}^n L_i^2$.
  \item The right inequality follows from $\cL_{-}(x,y) \geq 0$ and $$\cL_{\pm}(x,y) \leq \cL_{+}(x,y) \leq L_{+}^2\norm{x-y}^2.$$
  Now, we prove the left inequality. From the definition of $L_{\pm}$, we have
  \begin{eqnarray*}
    \cL_{\pm}(x,y) \leq L_{\pm}^2 \norm{x-y}^2,
  \end{eqnarray*}
  and 
  \begin{eqnarray*}
    \cL_{\pm}(x,y) = \cL_{+}(x,y) - \cL_{-}(x,y),
  \end{eqnarray*}
  hence,
  \begin{eqnarray*}
    \cL_{+}(x,y) \leq L_{\pm}^2 \norm{x-y}^2 + \cL_{-}(x,y) \leq (L_{-}^2 + L_{\pm}^2) \norm{x-y}^2,
  \end{eqnarray*}
  thus $L_{+}^2 \leq L_{-}^2 + L_{\pm}^2$.
\end{enumerate}
\end{proof}

%%%%%%%%%%%%%%%%%%%%%%%%%%%%%%%%%%%%%
\subsection{Proof of Theorem~\ref{thm:PermK-1}}
%%%%%%%%%%%%%%%%%%%%%%%%%%%%%%%%%%%%%
\PermK*

\begin{proof}
  $ $\newline
  
  We fix any $x \in \R^d$ and prove unbiasedness:
  \begin{eqnarray*}
    \Exp{\cC_i(x)} = n\sum_{j = q (i - 1) + 1}^{q i} \Exp{x_{\pi_j} e_{\pi_j}}
    = n\left(\sum_{j = q (i - 1) + 1}^{q i} \frac{1}{d}\sum\limits_{i = 1}^d x_ie_i\right)
    = \frac{nq}{d}x
    = x.
  \end{eqnarray*}
  
  Next, we find the second moment:
  \begin{eqnarray*}
    \Exp{\norm{\cC_i(x)}^2} = n^2 \sum_{j = q (i - 1) + 1}^{q i} \Exp{|x_{\pi_j}|^2}
    = n^2 \sum_{j = q (i - 1) + 1}^{q i} \frac{1}{d}\sum_{i = 1}^d |x_i|^2 
    = n^2 \frac{q}{d}\norm{x}^2 
    = n\norm{x}^2.
  \end{eqnarray*}
  For all $a_1, \dots, a_n \in \R^d,$ the following inequality holds:
  \begin{eqnarray*}
    \Exp{\norm{\frac{1}{n}\sum\limits_{i = 1}\cC_i(a_i)}^2} &=& \frac{1}{n^2}\sum\limits_{i = 1}^n \Exp{\norm{\cC_i(a_i)}^2} + \sum\limits_{i\neq j}\Exp{\inp{\cC_i(a_i)}{C_j(a_j)}} \\
    &=& \frac{1}{n^2}\sum\limits_{i = 1}^n \Exp{\norm{\cC_i(a_i)}^2} \\
    &=& \frac{1}{n}\sum\limits_{i = 1}^n \norm{a_i}^2.
  \end{eqnarray*}
  Hence, Assumption~\ref{ass:AB} is fulfilled with $A = B = 1$. 
  \end{proof}

%%%%%%%%%%%%%%%%%%%%%%%%%%%%%%%%%%%%%
\subsection{Proof of Theorem~\ref{thm:PermK-2}}
%%%%%%%%%%%%%%%%%%%%%%%%%%%%%%%%%%%%%
\PermKK*
\begin{proof}
  $ $\newline

  We fix any $x \in \R^d$ and prove unbiasedness:
  \begin{eqnarray*}
    \Exp{\cC_i(x)} = d \Exp{x_{\pi_i}e_{\pi_i}}
    = d \frac{1}{d}\sum\limits_{i = 1}^d x_i e_i
    = x.
  \end{eqnarray*}
  
  Next, we find the second moment:
  \begin{eqnarray*}
    \Exp{\norm{\cC_i(x)}^2} = \frac{1}{d} \sum\limits_{i = 1}^d d^2 |x_i|^2
    = d \norm{x}^2.
  \end{eqnarray*}
  
  For all $i \neq j \in \{1, 2, \dots, n\}, x, y \in \R^d,$ we have
\begin{eqnarray*}
  \Exp{\inp{\cC_i(x)}{\cC_j(y)}} &=& \ExpCond{\inp{\cC_i(x)}{\cC_j(y)}}{\pi_i = \pi_j} \Prob\left(\pi_i = \pi_j\right) \\
  &=& \frac{(q - 1)}{(n - 1)d} \sum\limits_{q = 1}^d\ExpCond{\inp{C_i(x)}{C_j(y)}}{\pi_i = d, \pi_j = d} \\
  &=& \frac{(q - 1)}{(n - 1)d} \sum\limits_{q = 1}^d d^2 x_q y_q \\
  &=& \frac{(q - 1)d}{n - 1}\inp{x}{y}.
\end{eqnarray*}
For all $a_1, \dots, a_n \in \R^d,$ the following inequality holds:
\begin{eqnarray*}
  \Exp{\norm{\frac{1}{n}\sum\limits_{i = 1}^n\cC_i(a_i)}^2} &=&  \frac{1}{n^2}\sum\limits_{i = 1}^n \Exp{\norm{\cC_i(a_i)}^2} + \frac{1}{n^2}\sum\limits_{i\neq j}\Exp{\inp{C_i(a_i)}{C_j(a_j)}} \\
  &=& \frac{d}{n^2}\sum\limits_{i = 1}^n \norm{a_i}^2 + \frac{1}{n^2}\sum\limits_{i\neq j}\Exp{\inp{\cC_i(a_i)}{\cC_j(a_j)}} \\
  &=& \frac{d}{n^2}\sum\limits_{i = 1}^n \norm{a_i}^2 + \frac{(q - 1)d}{n^2(n - 1)}\sum\limits_{i\neq j}\inp{a_i}{a_j} \\
  &=& \left(\frac{d}{n} - \frac{(q - 1)d}{n(n - 1)}\right)\frac{1}{n}\sum\limits_{i = 1}^n \norm{a_i}^2 + \frac{(q - 1)d}{n - 1}\norm{\frac{1}{n}\sum\limits_{i = 1}^n a_i}^2 \\
  &=& \left(1 - \frac{n - d}{n - 1}\right)\frac{1}{n}\sum\limits_{i = 1}^n \norm{a_i}^2 + \frac{n - d}{n - 1}\norm{\frac{1}{n}\sum\limits_{i = 1}^n a_i}^2.
\end{eqnarray*}
Hence, Assumption~\ref{ass:AB} is fulfilled with $A = B = 1 - \frac{n - d}{n - 1}$.
\end{proof}
%%%%%%%%%%%%%%%%%%%%%%%%%%
\subsection{Proof of Theorem~\ref{thm:second-order}}
%%%%%%%%%%%%%%%%%%%%%%%%%%%%%%
\secondorder*
\begin{proof}
 The fundamental theorem of calculus says that for any continuously differentiable function $\psi:\R \to \R$ we have
\[\psi(1) - \psi(0) =  \int_0^1 \psi'(t) dt .\]
Choose $i\in \{1,2,\dots,n\}$, $j\in \{1,2,\dots,d\}$, distinct vectors $x,y\in \R^d$, and let $$\psi_{ij}(t)\eqdef \langle \nabla f_i(x+ t(y-x)), e_j \rangle,$$ where $e_j\in \R^d$ is the $j$th standard unit basis vector. Since $f_i$ is twice continuously differentiable, $\psi_{ij}$ is continuously differentiable, and  by the chain rule, $$\psi_{ij}'(t) = \langle \nabla^2 f_i(x+t(y-x))(y-x), e_j \rangle.$$ Applying the fundamental theorem of calculus, we get 
\begin{equation}\label{eq:iubv98fd9fd09}\psi_{ij}(1) - \psi_{ij}(0) = \int_0^1 \langle \nabla^2 f_i(x+t(y-x))(y-x), e_j \rangle dt.\end{equation}
Let $\psi_i:\R \to \R^d$ be defined by $\psi_i(t) \eqdef \nabla f_i(x + t(y-x)) = (\psi_{i1}(t), \dots, \psi_{id}(t)) $.
Combining equations \ref{eq:iubv98fd9fd09} for $j=1,2,\dots,d$ into a vector form using the fact that
\[ \int_0^1 \langle \nabla^2 f_i(x+t(y-x))(y-x), e_j \rangle dt  =  \left\langle \left(\int_0^1 \nabla^2 f_i(x+t(y-x))dt \right) (y-x), e_j \right\rangle  \]
we arrive at the identity
\begin{eqnarray}\nabla f_i(y) - \nabla f_i(x) &=& \psi_{i}(1) - \psi_{i}(0) \notag \\
&\overset{(\ref{eq:iubv98fd9fd09})}{=}& \left(\int_0^1  \nabla^2 f_i(x+t(y-x))dt \right)(y-x) \notag \\
&\overset{(\ref{eq:hfd-0f9y8gfd9-u8fd})}{=}& \mH_i(x,y) (y-x).\label{eq:98h98dh998-jkkjNUDB}\end{eqnarray}
Next, since $\nabla^2 f_i(x+t(y-x))$ is symmetric for all $t$, so is $\mH_i(x,y)$, and hence $\mL_i(x,y) \eqdef\mH_i^2(x,y) = \mH_i^\top(x,y)\mH_i(x,y)$, which also means that $\mL_i(x,y)$ is symmetric and positive semidefinite. Combining these observations, we obtain
\begin{equation}\label{eq:n098h0fd-097bfd}\norm{\nabla f_i(x) - \nabla f_i(y)}^2 \overset{(\ref{eq:98h98dh998-jkkjNUDB})}{=} (x-y)^\top \mL_i(x,y) (x-y).\end{equation}
Clearly,  
\[L_i^2 = \sup \limits_{x,y\in \R^d, x\neq y} \frac{\norm{\nabla f_i(x) - \nabla f_i(y)}^2}{\norm{x-y}^2} \overset{(\ref{eq:n098h0fd-097bfd})}{=} \sup \limits_{x,y\in \R^d, x\neq y} \frac{(x-y)^\top \mL_i(x,y) (x-y)}{\norm{x-y}^2}.\]
Using the same reasoning, we have $\nabla f(y) - \nabla f(x) = \mH(x,y) (y-x),$ and
\begin{align*}
  L_-^2 = \sup \limits_{x,y\in \R^d, x\neq y} \frac{\norm{\nabla f(x) - \nabla f(y)}^2}{\norm{x-y}^2} = 
  \sup \limits_{x,y\in \R^d, x\neq y} \frac{(x-y)^\top \mL_-(x,y) (x-y)}{\norm{x-y}^2},
\end{align*}
where $\mL(x,y) \eqdef\mH^2(x,y) = \mH^\top(x,y)\mH(x,y)$ is symmetric and positive semidefinite, 
since $\mH_i(x,y)$ are symmetric and positive semidefinite.
Finally,
\begin{eqnarray*}
  L_+^2 &=& \sup \limits_{x,y\in \R^d, x\neq y} \frac{\frac{1}{n}\sum_{i=1}^n\norm{\nabla f_i(x) - \nabla f_i(y)}^2}{\norm{x-y}^2} \\
  &=& \sup \limits_{x,y\in \R^d, x\neq y} \frac{(x-y)^\top \left(\frac{1}{n}\sum_{i=1}^n\mH_i^2(x,y)\right) (x-y)}{\norm{x-y}^2} \\
  &=& \sup \limits_{x,y\in \R^d, x\neq y} \frac{(x-y)^\top \mL_+(x,y) (x-y)}{\norm{x-y}^2},
\end{eqnarray*}
and
\begin{eqnarray*}
  L_\pm^2 &=& \sup \limits_{x,y\in \R^d, x\neq y} \frac{\frac{1}{n}\sum_{i=1}^n\norm{\nabla f_i(x) - \nabla f_i(y)}^2 - \norm{\nabla f(x) - \nabla f(y)}^2}{\norm{x-y}^2} \\
  &=& \sup \limits_{x,y\in \R^d, x\neq y} \frac{(x-y)^\top \left(\frac{1}{n}\sum_{i=1}^n\mH_i^2(x,y) - \mH^2(x,y)\right) (x-y)}{\norm{x-y}^2} \\
  &=& \sup \limits_{x,y\in \R^d, x\neq y} \frac{(x-y)^\top \mL_\pm(x,y) (x-y)}{\norm{x-y}^2}.
\end{eqnarray*}
Note, that $\mL_+(x,y)$ inherits symmetry and positive semidefiniteness from $\mH_i^2(x,y).$
Symmetry of $\mL_\pm(x,y)$ is trivial. To prove positive semidefiniteness of $\mL_\pm(x,y)$, note that
\begin{eqnarray*}
  \mL_\pm(x,y) &=& \frac{1}{n}\sum_{i=1}^n \mH_i^2(x,y) - \mH^2(x,y) \\
  &=& \frac{1}{n}\sum_{i=1}^n\left(\mH_i(x,y) - \mH(x,y) + \mH(x,y)\right)^2 - \mH^2(x,y)\\
  &=& \frac{1}{n}\sum_{i=1}^n\left(\mH_i(x,y) - \mH(x,y)\right)^2 + \frac{1}{n}\mH(x,y)\sum_{i=1}^n\left(\mH_i(x,y) - \mH(x,y) \right)\\ 
  &\quad +& \frac{1}{n}\sum_{i=1}^n\left(\mH_i(x,y) - \mH(x,y)\right)\mH(x,y) \\
  &=& \frac{1}{n}\sum_{i=1}^n\left(\mH_i(x,y) - \mH(x,y)\right)^2,
\end{eqnarray*}
which is positive semidefinite.
\end{proof}
%%%%%%%%%%%%%%%%%%%%%%%%%%%%%%
\subsection{Proof of Theorem~\ref{theorem:AB}}
\label{sec:proof_theorem_ab}
%%%%%%%%%%%%%%%%%%%%%%%%%%%%%%

\ab*

\begin{proof}
In the proof, we follow closely the analysis of \cite{MARINA} and adapt it to utilize the power of Hessian variance (Definition~\ref{ass:HV}) and AB assumption (Assumption~\ref{ass:AB}). 
We bound the term $\Exp{\norm{g^{t+1}-\nabla f(x^{t+1})}^2}$ in a similar fashion to \cite{MARINA}, but make use of the AB assumption. Other steps
are essentially identical, but refine the existing analysis through Hessian variance.

First, we recall the following lemmas.

\begin{lemma}[\cite{PAGE}]
  \label{lemma:page_lemma}
  Suppose that $L_{-}$ is finite and let $x^{t+1} = x^{t} - \gamma g^{t}$. Then for any $g^{t} \in \R^d$ and $\gamma > 0$, we have
  \begin{eqnarray}
    \label{eq:page_lemma}
    f(x^{t + 1}) \leq f(x^t) - \frac{\gamma}{2}\norm{\nabla f(x^t)}^2 - \left(\frac{1}{2\gamma} - \frac{L_-}{2}\right)
    \norm{x^{t+1} - x^t}^2 + \frac{\gamma}{2}\norm{g^{t} - x^t}^2.
  \end{eqnarray}
\end{lemma}

\begin{lemma}[\cite{EF21}]
  \label{lemma:stepsize_page}
  Let $a,b>0$. If $0 \leq \gamma \leq \frac{1}{\sqrt{a}+b}$, then $a \gamma^{2}+b \gamma \leq 1$. Moreover, the bound is tight up to the factor of 2 since $\frac{1}{\sqrt{a}+b} \leq \min \left\{\frac{1}{\sqrt{a}}, \frac{1}{b}\right\} \leq \frac{2}{\sqrt{a}+b}$
\end{lemma}

Next, we get an upper bound of $\ExpCond{\norm{g^{t+1}-\nabla f(x^{t+1})}^2}{x^{t+1}}.$

\begin{lemma}
  \label{lemma:upper_bound_variance}
  Let us consider $g^{t + 1}$ from Line \ref{alg:gradient_estimate_definition} of Algorithm~\ref{alg:marina} and 
  assume, that Assumptions \ref{as:L_+}, \ref{eq:unbiased_compressors} and \ref{ass:AB} hold, then
  \begin{eqnarray}
    \label{eq:lemma_gradient_estimate_bound}
    \ExpCond{\norm{g^{t+1}-\nabla f(x^{t+1})}^2}{x^{t+1}}
    &\leq& (1 - p)\left(\left(A - B\right)L_{+}^2 + BL_{\pm}^2\right)\norm{x^{t+1} - x^{t}}^2 \notag\\
      && \qquad + (1 - p)\norm{g^t - \nabla f(x^{t})}^2.
  \end{eqnarray}
\end{lemma}

\begin{proof}
  In the view of definition of $g^{t + 1}$, we get
  \begin{align*}
    &\ExpCond{\norm{g^{t+1}-\nabla f(x^{t+1})}^2}{x^{t+1}} \\
    &= (1-p) \ExpCond{\norm{g^t + \frac{1}{n}\sum\limits_{i=1}^n \cC_i\left(\nabla f_{i}(x^{t+1}) - \nabla f_{i}(x^t)\right) - \nabla f(x^{t+1})}^2}{x^{t+1}}\\
    &= (1-p)\ExpCond{\norm{\frac{1}{n}\sum\limits_{i=1}^n \cC_i\left(\nabla f_{i}(x^{t+1}) - \nabla f_{i}(x^t)\right) - \nabla f(x^{t+1}) + \nabla f(x^t)}^2}{x^{t+1}}\\
    &\quad + (1-p)\norm{g^t - \nabla f(x^t)}^2.\\
  \end{align*}
  In the last inequality we used unbiasedness of $\cC_i.$ Next, from AB inequality, we have
  \begin{align*}
    &\ExpCond{\norm{g^{t+1}-\nabla f(x^{t+1})}^2}{x^{t+1}} \\
    &\leq (1-p)\ExpCond{\norm{\frac{1}{n}\sum\limits_{i=1}^n \cC_i\left(\nabla f_{i}(x^{t+1}) - \nabla f_{i}(x^t)\right) - \nabla f(x^{t+1}) + \nabla f(x^t)}^2}{x^{t+1}}\\
    &\quad + (1-p)\norm{g^t - \nabla f(x^t)}^2.\\
    &\leq (1 - p)\left(A\left(\frac{1}{n}\sum\limits_{i=1}^n \norm{\nabla f_i(x^{t+1}) - \nabla f_i(x^{t})}^2\right) - B\norm{\nabla f(x^{t + 1}) - \nabla f(x^t)}^2\right)\\
    &\quad + (1 - p)\norm{g^t - \nabla f(x^{t})}^2\\
    &= (1 - p)\Bigg(\left(A - B\right)\left(\frac{1}{n}\sum\limits_{i=1}^n \norm{\nabla f_i(x^{t+1}) - \nabla f_i(x^{t})}^2\right) \\
    &\quad + B\left(\frac{1}{n}\sum\limits_{i=1}^n \norm{\nabla f_i(x^{t+1}) - \nabla f_i(x^{t})}^2 - \norm{\nabla f(x^{t + 1}) - \nabla f(x^t)}^2\right)\Bigg) \\
    &\quad + (1 - p)\norm{g^t - \nabla f(x^{t})}^2.\\
  \end{align*}
  Using Assumption~\ref{as:L_+} and Definition~\ref{ass:HV}, we obtain (\ref{eq:lemma_gradient_estimate_bound}).
\end{proof}

We are ready to prove Theorem~\ref{theorem:AB}. Defining
\begin{align*}
  &\Phi^t \eqdef f(x^t) - f^{\inf} + \frac{\gamma}{2p}\norm{g^t - \nabla f(x^t)}^2,\\
  &\widehat{L}^2 \eqdef \left(A - B\right)L_{+}^2 + BL_{\pm}^2,
\end{align*}
and using inequalities (\ref{eq:page_lemma}) and (\ref{eq:lemma_gradient_estimate_bound}), we get
\begin{align*}
  &\Exp{\Phi^{t+1}} \\
  &\leq \Exp{f(x^t) - f^{\inf} - \frac{\gamma}{2}\norm{\nabla f(x^t)}^2 - \left(\frac{1}{2\gamma} - \frac{L_-}{2}\right)\norm{x^{t+1}-x^t}^2 + \frac{\gamma}{2}\norm{g^t - \nabla f(x^t)}^2} \\
  &\quad + \frac{\gamma}{2p}\Exp{(1-p)\widehat{L}^2\norm{x^{t+1}-x^t}^2 + (1-p)\norm{g^t - \nabla f(x^t)}^2} \\
  &= \Exp{\Phi^t} - \frac{\gamma}{2}\Exp{\norm{\nabla f(x^t)}^2} \\
  &\quad + \left(\frac{\gamma(1-p)\widehat{L}^2}{2p} - \frac{1}{2\gamma} + \frac{L_{-}}{2}\right) \Exp{\norm{x^{t+1}-x^t}^2} \\
  &\leq \Exp{\Phi^t} - \frac{\gamma}{2}\Exp{\norm{\nabla f(x^t)}^2}
\end{align*}
where in the last inequality we use 
$$\frac{\gamma(1-p)\widehat{L}^2}{2p} - \frac{1}{2\gamma} + \frac{L}{2} \leq 0$$ following from the stepsize choice and Lemma \ref{lemma:stepsize_page}.

Summing up inequalities $\Exp{\Phi^{t+1}} \leq \Exp{\Phi^t} - \frac{\gamma}{2}\Exp{\norm{\nabla f(x^t)}^2}$ for $t=0,1,\ldots,T-1$ and rearranging the terms, we get
\begin{eqnarray*}
  \frac{1}{T}\sum\limits_{t=0}^{T-1}\Exp{\norm{\nabla f(x^t)}^2} &\le& \frac{2}{\gamma T}\sum\limits_{t=0}^{T-1}\left(\Exp{\Phi^t}-\Exp{\Phi^{t+1}}\right) = \frac{2\left(\Exp{\Phi^0}-\Exp{\Phi^{T}}\right)}{\gamma T} \leq \frac{2\Delta_0}{\gamma T},
\end{eqnarray*}
since $g^0 = \nabla f(x^0)$ and $\Phi^{T} \geq 0$. Finally, using the tower property and the definition of $\hat x^T$ (see Section~\ref{section:marina_pseudocode}), we obtain the desired result.
\end{proof}

%%%%%%%%%%%%%%%%%%%%%%%%%%%%
\newpage
\section{Polyak-\L ojasiewicz Analysis}
\label{section:PL}
%%%%%%%%%%%%%%%%%%%%%%%%%%%%

In this section, we analyze the algorithm under Polyak-\L ojasiewicz (P\L) condition. 
We show that \algname{MARINA} algorithm with Assumption~\ref{ass:pl_condition} enjoys a linear convergence rate.
Now, we state the assumption and the convergence rate theorem.

\begin{assumption}[P\L\,condition] \label{ass:pl_condition} 
  Function $f$ satisfies Polyak-\L ojasiewicz (P\L) condition, i.e.,
  \begin{align}
    \label{ass:pl}
    \norm{\nabla f(x)}^2 \geq 2\mu (f(x) - f^{\star}), \quad \forall x \in \R^d,
  \end{align}
  where $\mu > 0$ and $f^{\star} \eqdef \inf_{x} f(x)$.
\end{assumption}

\begin{restatable}{lemma}{Lmurelation}
  \label{lemma:relation_between_L_and_mu}
  For $L_{-} > 0$ and $\mu$ from Assumption~\ref{ass:pl_condition} holds that $L_{-} \geq \mu.$
\end{restatable}

\begin{restatable}{theorem}{TheoremABPL}
  \label{theorem:AB_PL}
 Let Assumptions~\ref{ass:diff}, \ref{as:L_+}, \ref{eq:unbiased_compressors}, 
\ref{ass:AB} and \ref{ass:pl_condition} be satisfied and
  \begin{align}
    \label{theorem:AB_PL:gamma}
    \gamma \leq \min\left\{\left(L_- + \sqrt{\frac{2\left(1-p\right)}{p}\left((A - B)L_+^2 + BL_\pm^2\right)}\right)^{-1}, \frac{p}{2\mu}\right\},
  \end{align}
  then for ${x}^T$ from \algname{MARINA} algorithm the following inequality holds:
  $$
  \Exp{f(x^T) - f^{\star}} \leq \left(1 - \gamma \mu\right)^T \Delta^0.
  $$
\end{restatable}
We provide the proof to Theorem~\ref{theorem:AB_PL} in Section~\ref{sec:proof_theorem_ab_pl}.

In Table~\ref{table:communication_complexity_pl}, we provide communication complexity of \algname{MARINA} with Perm$K$ and Rand$K$, and \algname{EF21} with Top$K$, optimized w.r.t. parameters of the methods.
As in Section~\ref{sec:theory}, we see that \algname{MARINA} with Perm$K$ is not worse than \algname{MARINA} with Rand$K$ (recall Lemma~\ref{lem:08hhdf}). 

Let us consider zero Hessian variance regime: $L_\pm = 0.$ 
When $d \geq n$, Perm$K$ compressor has communication complexity $\cO\left(\max\left\{\nicefrac{dL_-}{n\mu}, d\right\}\right),$ while Rand$K$ compressor has communication complexity $\cO\left(\max\left\{\nicefrac{dL_-}{\sqrt{n}\mu}, d\right\}\right)$. And the communication complexity of Perm$K$ is strictly better when $\nicefrac{dL_-}{\sqrt{n}\mu} > d.$ 
Moreover, if $d \leq n$ and $\left(1 + \nicefrac{d}{\sqrt{n}}\right) \nicefrac{L_{-}}{\mu} > d$, then we get the strict improvement of the communication complexity from $\cO\left(\max\left\{\left(1 + \nicefrac{d}{\sqrt{n}}\right) \nicefrac{L_{-}}{\mu}, d\right\}\right)$ to $\cO\left(\max\left\{L_- / \mu, d\right\}\right)$ over \algname{MARINA} with Rand$K$.

\newcommand{\scaleboxparampl}{0.75}

\begin{table}[h!]
  \centering
  \caption{\footnotesize Optimized communication complexity of \algname{MARINA} and \algname{EF21} with particular compressors under P\L\,condition (up to a logarithmic factor).}
  \label{table:communication_complexity_pl}
  \begin{tabular}{|c|c|c|}
    \hline 
    &\multicolumn{2}{|c|}{Communication complexity} \\ \hline
    Method & $d \geq n$ (Lemma~\ref{lemma:optimal_parameters_of_methods_pl}) & $d \leq n$ (Lemma~\ref{lemma:optimal_parameters_of_methods_pl:n_geq_d}) \\ \hline
    \algname{MARINA} $\bigcap$ Perm$K$  & 
    \scalebox{\scaleboxparampl}{$\cO\left(\max\left\{\frac{1}{\mu}\min\left\{dL_{-},\frac{d}{n}L_- + \frac{d}{\sqrt{n}}L_\pm\right\}, d\right\}\right)$}
    & \scalebox{\scaleboxparampl}{$\cO\left(\max\left\{\frac{1}{\mu}\min\left\{dL_{-},L_- + \frac{d}{\sqrt{n}}L_\pm\right\}, d\right\}\right)$}
    \\ \hline
    \algname{MARINA} $\bigcap$ Rand$K$  & 
    \scalebox{\scaleboxparampl}{$\cO\left(\max\left\{\frac{1}{\mu}\min\left\{dL_{-},\frac{d}{\sqrt{n}}L_{+}\right\}, d\right\}\right)$}
    & \scalebox{\scaleboxparampl}{$\cO\left(\max\left\{\frac{1}{\mu}\min\left\{dL_{-},L_- + \frac{d}{\sqrt{n}}L_{+}\right\}, d\right\}\right)$}
    \\ \hline
    \algname{EF21} $\bigcap$ Top$K$ & \scalebox{\scaleboxparampl}{$\cO\left(\frac{dL_-}{\mu}\right)$}
    & \scalebox{\scaleboxparampl}{$\cO\left(\frac{dL_-}{\mu}\right)$} \\[0.5ex] \hline
  \end{tabular}
\end{table}

\subsection{Proof of Lemma~\ref{lemma:relation_between_L_and_mu}}

\Lmurelation*

\begin{proof}
We can define $L_{-}$ using the following inequality:
$$f(x) \leq f(y) + \inp{\nabla f(y)}{x - y} + \frac{L_{-}}{2}\norm{x - y}^2, \quad \forall x, y, \in \R^d.$$
Let us take $x = y - 1 /L_{-}\nabla f(y)$. Then,
$$f\left(y - \frac{1}{L_{-}}\nabla f(y)\right) \leq f(y) - \frac{1}{2L_{-}}\norm{\nabla f(y)}^2, \quad \forall y \in \R^d.$$
Rearranging the terms and using Definition~\ref{ass:pl}, we have
$$\frac{1}{2L_{-}}\norm{\nabla f(y)}^2 \leq f(y) - f\left(y - \frac{1}{L_{-}}\nabla f(y)\right) \leq f(y) - f^{\inf} \leq \frac{1}{2\mu}\norm{\nabla f(y)}^2,$$
thus, $\mu \leq L_{-}.$
\end{proof}

%%%%%%%%%%%%%%%%%%%%%%%%%%%%%%%
\subsection{Proof of Theorem~\ref{theorem:AB_PL}}
\label{sec:proof_theorem_ab_pl}
%%%%%%%%%%%%%%%%%%%%%%%%%%%%%%%

\TheoremABPL*

\begin{proof}
The analysis is almost the same as in \cite{MARINA}, but we include it for completeness. Let us define
\begin{eqnarray*}
  \Phi^t \eqdef f(x^t) - f^{\inf} + \frac{\gamma}{p}\norm{g^t - \nabla f(x^t)}^2,
  &\widehat{L}^2 \eqdef \left(\left(A - B\right)L_{+}^2 + BL_{\pm}^2\right).
\end{eqnarray*}
As in Appendix~\ref{sec:proof_theorem_ab}, we use (\ref{eq:page_lemma}) and (\ref{eq:lemma_gradient_estimate_bound}) to get that
\begin{align*}
  &\Exp{\Phi^{t+1}} \\
  &\leq \Exp{f(x^t) - f^{\inf} - \frac{\gamma}{2}\norm{\nabla f(x^t)}^2 - \left(\frac{1}{2\gamma} - \frac{L_-}{2}\right)\norm{x^{t+1}-x^t}^2 + \frac{\gamma}{2}\norm{g^t - \nabla f(x^t)}^2} \\
  &\quad + \frac{\gamma}{p}\Exp{(1-p)\widehat{L}^2\norm{x^{t+1}-x^t}^2 + (1-p)\norm{g^t - \nabla f(x^t)}^2} \\
  &\stackrel{(\ref{ass:pl})}{\leq} \Exp{(1 - \gamma \mu)(f(x^t) - f^{\inf}) - \left(\frac{1}{2\gamma} - \frac{L_-}{2}\right)\norm{x^{t+1}-x^t}^2 + \frac{\gamma}{2}\norm{g^t - \nabla f(x^t)}^2} \\
  &\quad + \frac{\gamma}{p}\Exp{(1-p)\widehat{L}^2\norm{x^{t+1}-x^t}^2 + (1-p)\norm{g^t - \nabla f(x^t)}^2} \\
  &= \Exp{(1 - \gamma \mu)(f(x^t) - f^{\inf}) + \left(\frac{\gamma}{2} + \frac{\gamma}{p}(1-p)\right)\norm{g^t - \nabla f(x^t)}^2} \\
  &\quad + \Exp{\left(\frac{\gamma}{p}(1-p)\widehat{L}^2 - \frac{1}{2\gamma} + \frac{L_-}{2}\right)\norm{x^{t+1}-x^t}^2} \\
  &\leq (1 - \gamma \mu) \Exp{\Phi^{t}}.
\end{align*}

In the last inequality, we used $\frac{\gamma}{p}(1-p)\widehat{L}^2 - \frac{1}{2\gamma} + \frac{L_-}{2} \leq 0$ and $\frac{\gamma}{2} + \frac{\gamma}{p}(1-p) \leq (1 - \gamma \mu) \frac{\gamma}{p},$ that follow from (\ref{theorem:AB_PL:gamma}) and Lemma \ref{lemma:stepsize_page}. Unrolling $\Exp{\Phi^{t+1}} \leq (1 - \gamma \mu) \Exp{\Phi^{t}}$ and using $g^0 = \nabla f(x^0)$, we have

\begin{align*}
  \Exp{f(x^T) - f^{\inf}} \leq \Exp{\Phi^{T}} \leq (1 - \gamma \mu)^T \Phi^{0} = (1 - \gamma \mu)^T \left(f(x^0) - f^{\inf}\right).
\end{align*}

This concludes the proof.
\end{proof}

\clearpage
\section{\algname{EF21} Analysis}
\label{section:ef21_analysis}

We provide convergence proofs of \algname{EF21} algorithm from \cite{EF21} for non-convex and P\L{} regimes. They will be almost identical to the one by \cite{EF21} (indeed, the only change is the constant $L_+$ instead of $\widetilde L$), but we have decided to include it for the sake of clarity.

\subsection{\algname{EF21} rate in the non-convex regime}
%%%%%%%%%%%%%%%%%%%%%%%%%%%%%%%%%%%

We will be using the following lemmas, the proofs of which are in their corresponding papers.
\begin{lemma}[\cite{EF21}]
  \label{lem:theta-beta} Let $\cC$ to be $\alpha$-contractive for $0<\alpha\leq 1$. Define
	$G_i^t \eqdef  \sqnorm{ g_i^t - \nabla f_i(x^t) } $ and $W^t \eqdef \{g_1^t, \dots, g_n^t, x^t, x^{t+1}\}$. 		For any $s>0$ we have
		\begin{equation}\label{eq:90y0yfhdf} \Exp{ G_i^{t+1} \;|\; W^t} \leq (1-\theta(s))   G_i^t + \beta(s)  \sqnorm{\nabla f_i(x^{t+1}) - \nabla f_i(x^t)} ,\end{equation}
		where
		\begin{equation}\label{eq:theta-beta-def}\theta(s) \eqdef 1- (1- \alpha )(1+s), \qquad \text{and} \qquad \beta(s)\eqdef (1- \alpha ) \left(1+ s^{-1} \right).\end{equation}
\end{lemma}

\begin{lemma}[\cite{EF21}]
  \label{le:optimal_t-Peter} Let $0<\alpha\leq 1$ and for $s>0$ let $\theta(s)$ and $\beta(s)$ be as in \eqref{eq:theta-beta-def}. Then the solution of the optimization problem
	\begin{equation}\label{eq:98g_(89fd8gf9d} \min_{s} \left\{ \frac{\beta(s)}{\theta(s)} \;:\; 0<s<\frac{\alpha}{1-\alpha}\right\}\end{equation}
	is given by $s^* = \frac{1}{\sqrt{1-\alpha}}-1$. Furthermore, $\theta(s^*) = 1-\sqrt{1-\alpha}$, $\beta(s^*) = \frac{1-\alpha}{1-\sqrt{1-\alpha}}$ and
	\begin{equation}\label{eq:n98fhgd98hfd}\sqrt{\frac{\beta(s^*)}{\theta(s^*)}} = \frac{1}{\sqrt{1-\alpha}} -1 = \frac{1}{\alpha} + \frac{\sqrt{1-\alpha}}{\alpha} - 1 \leq \frac{2}{\alpha}-1.\end{equation}
\end{lemma}

We are now ready to conduct the proof.
\begin{theorem}\label{thm:main-distrib}
  Let Assumptions~\ref{ass:diff} and \ref{as:L_+} hold, and let the stepsize be set as
  \begin{equation} \label{eq:manin-nonconvex-stepsize}
  0<\gamma \leq \left(L_{-} + L_+\sqrt{\frac{\beta}{\theta}}\right)^{-1}.
\end{equation}
  Fix $T \geq 1$ and let $\hat{x}^{T}$ be chosen from the iterates $x^{0}, x^{1}, \ldots, x^{T-1}$  uniformly at random. Then
\begin{equation} \label{eq:main-nonconvex}		
\Exp{\sqnorm{\nabla f(\hat{x}^{T})} } \leq \frac{2\left(f(x^{0})-f^{\text {inf }}\right)}{\gamma T} + \frac{\Exp{G^0}}{\theta T} . 
\end{equation}		
\end{theorem}
\begin{proof}

{\bf STEP 1.} Recall that Lemma~\ref{lem:theta-beta} says that			\begin{equation}\label{eq:n89fg9d08hfbdi_8f}
				\Exp{\sqnorm{g_i^{t+1} -  \nabla f_i(x^{t+1})}\mid W^t } \leq (1 - \theta)\sqnorm{g_i^t - \nfixt} + \beta  \sqnorm {\nfixtpo - \nfixt} ,
			\end{equation}
			where 
$\theta = \theta(s^*)$ and $\beta = \beta(s^*)$ are given by Lemma~\ref{le:optimal_t-Peter}.			 Averaging inequalities \eqref{eq:n89fg9d08hfbdi_8f} over $i \in \{1,2,\dots,n\}$ gives
			\begin{eqnarray}
				\Exp{G^{t+1} \mid W^t } &=&  \frac{1}{n}\sumin\Exp{\sqnorm{ g_i^{t+1} - \nfixtpo} \mid W^t} \notag \\
				&\leq& \left(1 - \theta\right)\suminn \sqnorm{g_i^t - \nfixt} + \beta \suminn \sqnorm {\nfixtpo - \nfixt} \notag \\
				&=& (1 - \theta)G^t+ \beta \suminn \sqnorm {\nfixtpo - \nfixt} \notag \\
				&\le& \left(1 - \theta\right)G^t+ \beta L_+\sqnorm{x^{t+1} - x^t} .\label{eq:jbiu-9u0df9}
			\end{eqnarray}
			Using Tower property and $L$-smoothness in \eqref{eq:jbiu-9u0df9}, we proceed to 
			\begin{eqnarray}\label{eq:main_recursion_distrib}
				\Exp{G^{t+1}} = \Exp{\Exp{G^{t+1} \mid W^t}} \leq \left(1 - \theta\right)\Exp{G^t}+ \beta L_+^2 \Exp{ \sqnorm{x^{t+1} - x^t}}.
			\end{eqnarray}
			
{\bf STEP 2.} Next, using Lemma~\ref{lemma:page_lemma} and Jensen's inequality applied to the function $x\mapsto \sqnorm{x}$, we obtain the bound 
			\begin{eqnarray}
				f(x^{t+1}) &\leq& f(x^{t})-\frac{\gamma}{2}\sqnorm{\nabla f(x^{t})}-\left(\frac{1}{2 \gamma}-\frac{L_-}{2}\right)\sqnorm{x^{t+1}-x^{t}}+\frac{\gamma}{2}\sqnorm{\suminn \left(g_i^t-\nabla f_i(x^{t}) \right)}\notag \\
				&\leq&
				f(x^{t})-\frac{\gamma}{2}\sqnorm{\nabla f(x^{t})}-\left(\frac{1}{2 \gamma}-\frac{L_-}{2}\right)\sqnorm{x^{t+1}-x^{t}}+\frac{\gamma}{2} G^t. \label{eq:aux_smooth_lemma_distrib}
			\end{eqnarray}

			Subtracting $f^{\text {inf }}$ from both sides of \eqref{eq:aux_smooth_lemma_distrib} and taking expectation, we get
			\begin{eqnarray}\label{eq:func_diff_distrib}
				\Exp{f(x^{t+1})-\finf} &\leq& \quad \Exp{f(x^{t})-\finf}-\frac{\gamma}{2} \Exp{\sqnorm{\nabla f(x^{t})}} \notag \\
				&& \qquad -\left(\frac{1}{2 \gamma}-\frac{L_-}{2}\right) \Exp{\sqnorm{x^{t+1}-x^{t}}}+ \frac{\gamma}{2}\Exp{G^t}.
			\end{eqnarray}

{\bf COMBINING STEP 1 AND STEP 2.}	
	 Let $\delta^{t} \eqdef \Exp{f(x^{t})-\finf}$, $s^{t} \eqdef \Exp{G^t }$ and $r^{t} \eqdef\Exp{\sqnorm{x^{t+1}-x^{t}}}.
			$
			Then by adding \eqref{eq:func_diff_distrib} with a $\frac{\gamma}{2 \theta}$ multiple of \eqref{eq:main_recursion_distrib} we obtain
			\begin{eqnarray*}
				\delta^{t+1}+\frac{\gamma}{2 \theta} s^{t+1} &\leq& \delta^{t}-\frac{\gamma}{2}\sqnorm{\nabla f(x^{t})}-\left(\frac{1}{2 \gamma}-\frac{L_-}{2}\right) r^{t}+\frac{\gamma}{2} s^{t}+\frac{\gamma}{2 \theta}\left(\beta L_+^2 r^t + (1 - \theta) s^{t}\right) \\
				&=&\delta^{t}+\frac{\gamma}{2\theta} s^{t}-\frac{\gamma}{2}\sqnorm{\nabla f(x^{t})}-\left(\frac{1}{2\gamma} -\frac{L_-}{2} - \frac{\gamma}{2\theta}\beta L_+^2 \right) r^{t} \\
				& \leq& \delta^{t}+\frac{\gamma}{2\theta} s^{t} -\frac{\gamma}{2}\sqnorm{\nabla f(x^{t})}.
			\end{eqnarray*}
			The last inequality follows from the bound $\gamma ^2\frac{\beta L_+^2}{\theta} + L_-\gamma \leq 1,$ which holds from our assumption on the stepsize and Lemma \ref{lemma:stepsize_page}.
			By summing up inequalities for $t =0, \ldots, T-1,$ we get
			$$
			0 \leq \delta^{T}+\frac{\gamma}{2 \theta} s^{T} \leq \delta^{0}+\frac{\gamma}{2 \theta} s^{0}-\frac{\gamma}{2} \sum_{t=0}^{T-1} \Exp{\sqnorm{\nabla f(x^{t})}}.
			$$
			Multiplying both sides by $\frac{2}{\gamma T}$, after rearranging we get
			$$
			\sum_{t=0}^{T-1} \frac{1}{T} \Exp{\sqnorm{\nabla f (x^{t})}} \leq \frac{2 \delta^{0}}{\gamma T} + \frac{s^0}{\theta T}.
			$$
			It remains to notice that the left hand side can be interpreted as $\mathrm{E}\left[\sqnorm{\nabla f(\hat{x}^{T})}\right]$, where $\hat{x}^{T}$ is chosen from $x^{0}, x^{1}, \ldots, x^{T-1}$ uniformly at random.
\end{proof}

%%%%%%%%%%%%%%%%%%%%%%%%%%%%%%%
\subsection{\algname{EF21} in P\L{} regime}
%%%%%%%%%%%%%%%%%%%%%%%%%%%%%%%
\begin{theorem} \label{thm:PL-main} Let Assumptions~\ref{ass:diff}, \ref{as:L_+} and~\ref{ass:pl_condition} hold, and 	 let the stepsize in \algname{EF21} be set as 
  \begin{equation} \label{eq:PL-stepsize}
    0<\gamma \leq \min \left\{ \left(L_- + L_+\sqrt{\frac{2 \beta}{\theta}}\right)^{-1} , \frac{\theta}{2 \mu}\right\}.
  \end{equation}
      Let $\Psi^t\eqdef f(x^{t})-f^{\inf}+ \frac{\gamma}{\theta}G^t$. Then for any $T\geq 0$, we have
  \begin{equation} \label{PL-main}
      \Exp{ \Psi^T } \leq (1- \gamma \mu )^T \Exp{\Psi^0}.
  \end{equation}
\end{theorem}

\begin{proof}
  Again, this follows \cite{EF21} almost verbatim.

  We proceed as in the previous proof, but use the P\L{} inequality and subtract $f^{\inf}$ from both sides of  \eqref{eq:aux_smooth_lemma_distrib} to get
  \begin{eqnarray*}
    \Exp{f(x^{t+1})- f^{\inf}}  &\leq&  \Exp{ f(x^{t})- f^{\inf} }-\frac{\gamma}{2}\sqnorm{\nabla f(x^{t})}-\left(\frac{1}{2 \gamma}-\frac{L_-}{2}\right)\sqnorm{x^{t+1}-x^{t}}+\frac{\gamma}{2} G^t \\ 
    & \leq & (1-\gamma \mu) \Exp{ f(x^{t})- f^{\inf} } - \left(\frac{1}{2 \gamma}-\frac{L_-}{2}\right)\sqnorm{x^{t+1}-x^{t}}+\frac{\gamma}{2} G^t .
  \end{eqnarray*}
  
  Let $\delta^{t} \eqdef \Exp{f(x^{t})-f^{\inf}}$, $s^{t} \eqdef \Exp{G^t }$ and $ r^{t} \eqdef \Exp{\sqnorm{x^{t+1}-x^{t}}}$. Then by adding the above inequality with a $\frac{\gamma}{\theta}$ multiple of \eqref{eq:main_recursion_distrib}, we obtain
  \begin{eqnarray*}
    \delta^{t+1}+\frac{\gamma}{ \theta} s^{t+1} &\leq& (1-\gamma \mu)\delta^{t} -\left(\frac{1}{2 \gamma} - \frac{L_-}{2}\right) r^{t} + \frac{\gamma}{2} s^{t} + \frac{\gamma}{ \theta}\left((1-\theta) s^t + \beta L_+^2 r^t\right) \\
    &=&  (1-\gamma \mu) \delta^{t}  + \frac{\gamma}{\theta} \left(1-\frac{\theta}{2}\right) s^t  -\left(\frac{1}{2 \gamma} - \frac{L_-}{2} - \frac{\beta L_+^2 \gamma}{\theta}\right) r^{t}  	.
  \end{eqnarray*}
  Note that our assumption on the stepsize implies that 	$ 1 - \frac{\theta}{2} \leq 1 -\gamma \mu$ and $\frac{1}{2 \gamma} - \frac{L_-}{2} - \frac{\beta L_+^2 \gamma}{\theta} \geq 0$. The last inequality follows from the bound $\gamma^2\frac{2\beta L_+^2}{\theta} + \gamma L_- \leq 1,$ which holds because of Lemma \ref{lemma:stepsize_page} and our assumption on the stepsize. Thus,
  \begin{eqnarray*}
    \delta^{t+1}+\frac{\gamma}{ \theta} s^{t+1} &\leq& (1-\gamma \mu) \left(\delta^{t}+\frac{\gamma}{ \theta} s^{t} \right).
  \end{eqnarray*}
  It remains to unroll the recurrence.
\end{proof}

%%%%%%%%%%%%%%%%%%%%%%%%%%%%
\newpage
\section{Communication Model}
\label{section:comm_model}
%%%%%%%%%%%%%%%%%%%%%%%%%%%%

As mentioned in the introduction, 
{\em we consider the regime where the worker-to-server communication is the bottleneck of the system so that the server-to-workers communication can be neglected}. While this is a standard model used in many prior works, we include a brief explanation of why and when this regime is useful.

\begin{enumerate}
\item {\bf Peer-to-peer communication.} First, this regime makes sense when the server is merely an abstraction, and does not exist physically.  Indeed,  from the point of view of each worker, the server may merely represent ``all other nodes'' combined. In this model, ``a worker sending a message to the server''  should be  {\em interpreted} as this worker sending the message to all other workers. Clearly, in this model there is {\em no need} for the ``server'' to communicate the aggregated message back to the workers since aggregation is performed on all workers independently, and the aggregated message is immediately available to all workers without the need for any additional communication. 
\item {\bf Fast broadcast.} Second, the above regime makes sense in situations where the server exists physically, but is able to broadcast to the workers at a much higher speed compared to the worker-to-server communication. This happens in several distributed systems, e.g., on certain supercomputers \citep{DIANA}.  Virtually all theoretical works on communication efficient distributed algorithms assume that the server-to-worker communication is cheap, and in this work we follow in their footsteps. 
\end{enumerate}

Having said that, our work can be extended to the more difficult regime where the server-to-worker communication is also costly \citep{Cnat, DoubleSqueeze2019, Artemis2020, EC-SGD}. However, for simplicity, we do not explore this extension in this work.

%%%%%%%%%%%%%%%%%%%%%%%%%%%%
\section{On Contractive Compressors and Error Feedback}
\label{section:EF}
%%%%%%%%%%%%%%%%%%%%%%%%%%%%

\subsection{On Contractive Compressors}

The most successful algorithmic solutions to solving the nonconvex distributed optimization problem (\ref{eq:main}) in a communication-efficient manner under the communication model described in Appendix~\ref{section:comm_model}
 involve stochastic gradient descent (\algname{SGD}) methods with {\em communication compression}. There are two large classes of such methods, depending on the type of compression operator involved: (i) methods that work with contractive (and possibly biased stochastic) compression operators, such as Top$K$ or Rank$K$, and (ii) methods that work with unbiased  and independent (across the workers) stochastic compression operators, such as Rand$K$. 

A (randomized) compression operator $\cC:\R^d\to \R^d$ is $\alpha$-contractive (we write $\cC\in \mathbb{C}(\alpha)$), where $0<\alpha\leq 1$, if \begin{equation}\label{eq:contractive}\Exp{\norm{\cC(x) -x}^2} \leq (1-\alpha)\norm{x}^2, \quad x\in \R^d.\end{equation}  A canonical example is the (deterministic) Top$K$ compressor, which outputs the  $K$ largest  (in absolute value) entries of the input vector $x$, and zeroes out the rest. Top$K$ is $\alpha$-contractive with $\alpha = \nicefrac{K}{d}$. Another example is the Rank$K$ compressor based on the best rank-$K$ approximation of $x$ represented as an $a\times b=d$ matrix. It can be shown that Rank$K$ is $\alpha$-contractive with $\alpha=\nicefrac{K}{\min\{a,b\}}$ \citep[Section A.3.2]{FedNL}. 
We refer to the work of \cite{PowerSGD} for a practical communication-efficient method \algname{PowerSGD} based on low-rank approximations.  

Of special importance are $\alpha$-compressors arising from {\em unbiased} compressors via appropriate scaling. Let $\cQ:\R^d\to \R^d$ be an unbiased operator with variance proportional to the square norm of the input vector. That is, assume that $\Exp{\cQ(x)}=x$ for all $x\in \R^d$ and that there exists $\omega\geq 0$ such that
\begin{equation} \label{eq:omega-compressor} \Exp{\norm{\cQ(x)-x}^2} \leq \omega \norm{x}^2, \quad \forall x\in \R^d.\end{equation}
We will write $\cQ\in \mathbb{U}(\omega)$ for brevity.  It is well known that the operator $\cC = (\omega+1)^{-1}\cQ$ is $\alpha$-contractive with $\alpha = (\omega+1)^{-1}$. An example of  a contractive compressor arising this way is  $(\omega+1)^{-1}$Rand$K$, which keeps a subset of $K$ entries of the input vector $x$ chosen uniformly at random, and zeroes out the rest. As Top$K$, $(\omega+1)^{-1}$Rand$K$ is $\alpha$-contractive, with $\alpha=\nicefrac{K}{d}$.
 
Distributed \algname{SGD} methods relying on general contractive compressors, i.e., on contractive which do {\em not} arise from unbiased compressors from scaling,  need to rely on the error-feedback / error-compensation mechanism to avoid divergence.

\subsection{On Error Feedback}

An alternative approach to the one represented by \algname{MARINA} is to  seek more aggressive compression, even at the cost of abandoning unbiasedness, in the hope that this will lead to better communication complexity in practice.  This is the idea behind the class of {\em contractive compressors}, defined in (\ref{eq:contractive}), which have studied at least since the work of \citet{Seide2014}. Example of such compressors are the 
 Top$K$  \citep{Alistarh-EF-NIPS2018} and Rank$K$ \citep{PowerSGD, FedNL} compressors.

 %A (possibly) randomized mapping $\cC:\R^d \to \R^d$ is contractive if there exists $0<\alpha\leq 1$ such that $${\rm E}\norm{\cC(x)-x}^2 \leq (1-\alpha)\norm{x}^2$$ for all $x\in \R^d$. 
 
 While such compressors are indeed often very successful in practice, their theoretical impact on the methods using them is dramatically less understood than is the case with unbiased compressors. One of the key reasons for this that a naive use of biased compressors may lead to (exponential) divergence, even in simple problems \citep{beznosikov2020biased}. Because of this, \citet{Seide2014} proposed the {\em error feedback} framework for controlling the error introduced by compression, and thus taming the method to convergence. While it has been successfully used by practitioners for many years, error feedback  yielded the first convergence results only relatively recently \citep{Stich-EF-NIPS2018, Stich2019TheEF, errorSGD, Koloskova2019-DecentralizedEC-2019, DoubleSqueeze2019, Karimireddy_SignSGD, EC-LSVRG, beznosikov2020biased, EC-SGD}.

 The current best theoretical communication complexity results for error feedback belong to the \algname{EF21} method of \citet{EF21} who achieved their improvements by redesigning the original error feedback mechanism using the construction of a Markov compressor. However, even  \algname{EF21} currently enjoys substantially weaker iteration and communication complexity than  \algname{MARINA}. For instance, we show in Appendix~\ref{sec:appendix:complexity_bounds} that \algname{EF21} with Top$K$ is only proved to have the communication complexity of the gradient descent without any compression.

%%%%
\newpage
\section{Composition of Compressors with AB Assumption and Unbiased Compressors}
%%%%

\begin{lemma}
  \label{lem:composition} 
  If $\{\cC_i\}_{i=1}^n \in \mathbb{U}(A,B)$ and $\cQ_i\in \mathbb{U}(\omega_i)$ for $i\in \{1,2,\dots, n\}$, then $\{\cC_i \circ \cQ_i\}_{i=1}^n \in \mathbb{U}(\left(\max_i \omega_i + 1\right)A,B)$.
\end{lemma}

\begin{proof}
    By the tower property, for all $a_1, \dots, a_n \in \R^d,$ we have
    \begin{align*}
      &\Exp{\norm{\frac{1}{n}\sum\limits_{i=1}^n\cC_i(\cQ_i(a_i)) -  \frac{1}{n}\sum\limits_{i=1}^n \cQ_i(a_i)}^2} \\
      &=\Exp{\ExpCond{\norm{\frac{1}{n}\sum\limits_{i=1}^n\cC_i(\cQ_i(a_i))- \frac{1}{n}\sum\limits_{i=1}^n \cQ_i(a_i)}^2}{\cQ_1(a_1), \cdots, \cQ_n(a_n)}} \\
      &\leq \Exp{A\frac{1}{n}\sum\limits_{i=1}^n\norm{\cQ_i(a_i)}^2 - B\norm{\frac{1}{n}\sum\limits_{i=1}^n\cQ_i(a_i)}^2}.
    \end{align*}
    Since $\cQ_i\in \mathbb{U}(\omega_i)$ for $i\in \{1,2,\dots, n\}$, we get
    \begin{align*}
      \Exp{\norm{\frac{1}{n}\sum\limits_{i=1}^n\cC_i(\cQ_i(a_i)) -  \frac{1}{n}\sum\limits_{i=1}^n \cQ_i(a_i)}^2} 
      &\leq\left(\max_i \omega_i + 1\right)A\frac{1}{n}\sum\limits_{i=1}^n\norm{a_i}^2 - B\Exp{\norm{\frac{1}{n}\sum\limits_{i=1}^n\cQ_i(a_i)}^2}.
    \end{align*}
    Using Jensen's inequality, we derive inequalities:
    $$\Exp{\norm{\frac{1}{n}\sum\limits_{i=1}^n\cQ_i(a_i)}^2} \geq \norm{\frac{1}{n}\sum\limits_{i=1}^n\Exp{\cQ_i(a_i)}}^2 = \norm{\frac{1}{n}\sum\limits_{i=1}^n a_i}^2,$$
    and 
    \begin{align*}
      \Exp{\norm{\frac{1}{n}\sum\limits_{i=1}^n\cC_i(\cQ_i(a_i)) -  \frac{1}{n}\sum\limits_{i=1}^n \cQ_i(a_i)}^2} 
      &\leq\left(\max_i \omega_i + 1\right)A\frac{1}{n}\sum\limits_{i=1}^n\norm{a_i}^2 - B\norm{\frac{1}{n}\sum\limits_{i=1}^n a_i}^2.
    \end{align*}
    The last inequality completes the proof.
\end{proof}

%%%%
\clearpage
\section{General Examples of Perm$K$}
%%%%
\label{sec:general_examples}

For the sake of clarity, in the main part of our paper, we assumed that $n \mid d$ or $d \mid n$, and provided corresponding examples of Perm$K$ (see Definition~\ref{def:PermK-1} and Definition~\ref{def:PermK-2}). Now, we provide two examples of Perm$K$ that work with any $n$ and $d$ and generalize the previous examples.

\subsection{Case $d \geq n$}
The following example generalizes for the case when $n$ does not divide $d$.
Let us assume that $d = kn + r$ and $0 \leq r < n.$
As in Definition~\ref{def:PermK-1}, we permute coordinates and split them into the blocks of sizes 
$\{k, \dots, k, r\}.$ The first $n$ block of size $k$ we assign to nodes. 
Next, we take the last block of size $r$ and randomly assign each coordinate from this block to one node.
As the size of the last block of size $r$ is less than $n$, some nodes will send one coordinate less.

\begin{definition}[Perm$K$ \, ($d \geq n$)]
  \label{def:permutation_coordinates_with_permuted_nodes}
  Let us assume that $d \geq n,$ $d = kn + r$, $0 \leq r < n,$ $\pi^d=(\pi_1^d,\dots,\pi_d^d)$ is a random permutation of $\{1, \cdots, d\},$
  and $\pi^n=(\pi_1^n,\dots,\pi_n^n)$ is a random permutation of $\{1, \cdots, n\}.$
  We define the tuple of vectors $S(x) = (x_{\pi_{kn + 1}^d} e_{\pi_{kn + 1}^d}, \dots, x_{\pi_{kn + r}^d} e_{\pi_{kn + r}^d}, 0, \dots, 0)$ of size $n$. Then,
  $$\cC_i(x) \eqdef n\left(\sum_{j = k (i - 1) + 1}^{k i} x_{\pi_j^d} e_{\pi_j^d} + \left(S(x)\right)_{\pi^n_i}\right).$$
\end{definition}

\begin{theorem}
  Compressors $\{\cC_i\}_{i = 1}^{n}$ 
  from Definition~\ref{def:permutation_coordinates_with_permuted_nodes} belong to $\mathbb{IV}(1)$.
\end{theorem}

\begin{proof}

We fix any $x \in \R^d$ and prove unbiasedness:
\begin{eqnarray*}
  \Exp{\cC_i(x)} &=& n\left(\sum_{j = k (i - 1) + 1}^{k i} \Exp{x_{\pi_j^d} e_{\pi_j^d}} + \Exp{\left(S(x)\right)_{\pi^n_i}}\right) \\
  &=& n\left(\frac{k}{d}x + \left(1 - \frac{r}{n}\right)0 + \frac{r}{n} \frac{1}{d}x\right) \\
  &=& n\left(\frac{kn + r}{nd}\right)x \\
  &=& x,
\end{eqnarray*}
for all $i \in \{1, \dots, n\}.$

Next, we derive the second moment:
\begin{eqnarray*}
  \Exp{\norm{\cC_i(x)}^2} &=& n^2\left(\sum_{j = k (i - 1) + 1}^{k i} \Exp{\norm{x_{\pi_j^d} e_{\pi_j^d}}^2} + \Exp{\norm{\left(S(x)\right)_{\pi^n_i}}^2}\right) \\
  &=& n^2\left(\frac{k}{d}\norm{x}^2 + \left(1 - \frac{r}{n}\right)\norm{0}^2 + \frac{r}{n} \frac{1}{d}\norm{x}^2\right) \\
  &=& n^2\left(\frac{kn + r}{nd}\right)\norm{x}^2 \\
  &=& n\norm{x}^2,
\end{eqnarray*}

We fix $x, y \in \R^d.$ For all $i \neq l \in \{1, \dots, n\},$ we have
  \begin{align*}
    &\Exp{\inp{\cC_i(x)}{\cC_l(y)}} \\
    &= n^2\inp{\sum_{j = k (i - 1) + 1}^{k i} x_{\pi_j^d} e_{\pi_j^d} + \left(S(x)\right)_{\pi^n_i}}{\sum_{j = k (l - 1) + 1}^{k l} y_{\pi_j^d} e_{\pi_j^d} + \left(S(y)\right)_{\pi^n_l}} = 0,
  \end{align*}
due to orthogonality of vectors $e_{p},$ for all $p \in \{1, \dots, d\}$, and the fact that $i \neq l$.

Thus, for all $a_1, \dots, a_n \in \R^d,$ the following equality holds:
\begin{eqnarray*}
  \Exp{\norm{\frac{1}{n}\sum\limits_{i = 1}^n\cC_i(a_i)}^2} = \frac{1}{n^2}\sum\limits_{i = 1}^n \Exp{\norm{\cC_i(a_i)}^2} + \frac{1}{n^2}\sum\limits_{i\neq j}\Exp{\inp{C_i(a_i)}{C_j(a_j)}}
  = \frac{1}{n}\sum\limits_{i = 1}^n \norm{a_i}^2.
\end{eqnarray*}
Hence, Assumption~\ref{ass:AB} is fulfilled with $A = B = 1$.

\end{proof}

\subsection{Case $n \geq d$}
The following definition generalizes Definition~\ref{def:PermK-2} for the case when $d$ does not divide $n$.
Let us assume that $n = qd + r$ and $0 \leq r < d.$ As in Definition~\ref{def:PermK-2}, we permute the multiset, where each coordinate occures $q$ times. Then, we randomly assign each element from the multiset of size $qd$ to one node. 
Note that $r$ randomly chosen nodes are idle.

\begin{definition}[Perm$K$, ($n \geq d$)]
  \label{def:permutation_multiset}
  Let us assume that $n \geq d,$ $n = qd + r$, $0 \leq r < d.$ Let us fix point $x \in \R^d$, that we want to compress.
  Define the tuple of vectors $\widehat{S}(x) = (x_1 e_1, \dots, x_1 e_1, x_2 e_2, \dots, x_2 e_2, \dots, x_d e_d, \dots, x_d e_d)$,
  where each vector occurs $q$ times. 
  Concat $r$ zero vectors to $\widehat{S}(x)$: $S(x) = \widehat{S}(x) \oplus (0, \dots, 0)$.
  Let $\pi=(\pi_1,\dots,\pi_n)$ be a random permutation of $\{1, \dots, n\}$. Define 
  $$
    \cC_i(x) \eqdef \frac{n}{q} \left(S(x)\right)_{\pi_i}.
  $$
\end{definition}

\begin{theorem}
  Compressors $\{\cC_i\}_{i = 1}^{n}$ 
  from Definition~\ref{def:permutation_multiset} belong to $\mathbb{IV}\left(A\right)$ with $A = 1 - \frac{n(q - 1)}{(n - 1)q}$.
\end{theorem}

\begin{proof}
We start with proving the unbiasedness:
\begin{eqnarray*}
  \Exp{\cC_i(x)} = \frac{n}{q} \sum_{j = 1}^d x_j e_j \ProbArg{\pi_i = j}
  = \frac{n}{q} \sum_{j = 1}^d x_j e_j \frac{q}{n}
  = x,
\end{eqnarray*}
for all $i \in \{1, \dots, n\}, x \in \R^d.$

Next, we find the second moment:
\begin{eqnarray*}
  \Exp{\norm{\cC_i(x)}^2} = \frac{n^2}{q^2} \sum\limits_{i = 1}^d x_i^2 \ProbArg{\pi_i = j}
  = \frac{n}{q} \norm{x}^2,
\end{eqnarray*}
for all $i \in \{1, \dots, n\}, x \in \R^d.$

For all $i \neq j \in \{1, \dots, n\}$ and $x, y \in \R^d,$ we have
\begin{align*}
  &\Exp{\inp{\cC_i(x)}{\cC_j(y)}} = \frac{n^2}{q^2} \sum_{k = 1}^d x_k y_k \ProbArg{\pi_i = k, \pi_j = k} \\
  & = \frac{n^2}{q^2} \sum_{k = 1}^d x_k y_k \frac{q(q - 1)}{n(n - 1)} = \frac{n(q - 1)}{(n - 1)q} \inp{x}{y}.
\end{align*}

Thus, for all $a_1, \dots, a_n \in \R^d,$ the following equality holds:
\begin{eqnarray*}
  \Exp{\norm{\frac{1}{n}\sum\limits_{i = 1}^n\cC_i(a_i)}^2} 
  &=& \frac{1}{n^2}\sum\limits_{i = 1}^n \Exp{\norm{\cC_i(a_i)}^2} + \frac{1}{n^2}\sum\limits_{i\neq j}\Exp{\inp{C_i(a_i)}{C_j(a_j)}} \\
  &=& \frac{1}{nq}\sum\limits_{i = 1}^n \norm{a_i}^2 + \frac{n(q - 1)}{(n - 1)q}\left(\frac{1}{n^2}\sum\limits_{i\neq j}\inp{a_i}{a_j}\right) \\
  &=& \left(\frac{1}{q} - \frac{(q - 1)}{(n - 1)q}\right)\frac{1}{n}\sum\limits_{i = 1}^n \norm{a_i}^2 + \frac{n(q - 1)}{(n - 1)q}\norm{\frac{1}{n} \sum\limits_{i = 1}^n a_i}^2. \\
\end{eqnarray*}

Hence, Assumption~\ref{ass:AB} is fulfilled with $A = B = 1 - \frac{n(q - 1)}{(n - 1)q}$.

\end{proof}

%%%%%%%%%%%%%%%%%%%
\section{Implementation Details of Perm$K$}
\label{sec:implementation_details}
%%%%%%%%%%%%%%%%%%%

Now, we discuss the implementation details of Perm$K$ from Definition~\ref{def:PermK-1}. Unlike Rand$K$ and Top$K$ compressors, Perm$K$ compressors are statistically dependent. We provide a simple idea of how to manage dependence between nodes. First of all, note that the samples of random permutation $\pi$ are the only source of randomness. By Definition~\ref{def:PermK-1}, they are shared between nodes and generated in each communication round. However, instead of sharing the samples, we can generate these samples in each node regardless of other nodes. Almost all random generation libraries and frameworks are deterministic (or pseudorandom) and only depend on the initial random seed. Thus, at the beginning of the optimization procedure, all nodes should set the same initial random seed and then call the same function that generates samples of a random permutation. The computation complexity of generating a sample from a random permutation $\pi = (\pi_1,\dots,\pi_d)$ is $\cO\left(d\right)$ using the Fisher-Yates shuffle algorithm~~\citep{fisher1938statistical, knuth1997Art}. All other examples of compressors can be implemented in the same fashion.

%%%%
%%%%
\section{More Examples of Permutation-Based Compressors}
%%%%
%%%%

%%%%
\subsection{Block permutation compressor}
%%%%

 In block permutation compressor, we partition the set $\{1,\dots,d\}$ into $m\leq n$ disjoint blocks. For each block $P_i$, $\left\lfloor \frac{n}{m}\right\rfloor$ devices sparsify their vectors to coordinates with indices in $P_i$ only.

\begin{definition}
  \label{def:block_pc2}
  Let $P$ to be a partition of the set $\{1,\dots,d\}$ into $m\leq n$ non-empty subsets, and $n = mq + r,$ where $0 \leq r < m.$ 
  Define matrices $\mA_1, \dots, \mA_n$ as follows: put $\mA_i \eqdef 0$ if $i > mq$. Denote the subsets in $P$ as $P_1, \cdots, P_m$. Next, for any $P_i \in P$, we set $\mA_{(i-1)q + 1}, \mA_{(i-1)q + 2}, \cdots, \mA_{iq}$ to 
  $\frac{n}{q}\Diag(P_i)$. Here by $\Diag(S)$ we mean the diagonal matrix where each $i^{\text{th}}$ diagonal entry is equal to $1$ if $i\in S$ and 0 otherwise. Let $\pi=(\pi_1,\dots,\pi_n)$ be a random permutation
  of set $\{1,\dots,n\}$. We define $\cC_i(x) \eqdef \mA_{\pi_i}x$. We call the set $\{\cC_i\}_{i = 1}^n$ the \emph{block permutation compressor}.
\end{definition}

\begin{lemma}
  Compressors $\{\cC_i\}_{i = 1}^{n}$ belong to $\mathbb{IV}(A)$ with $A = 1 - \frac{n(q-1)}{(n-1)q}.$
\end{lemma}

\begin{proof}
  We start with the proof of unbiasedness:
  \begin{eqnarray*}
    \Exp{\cC_i(x)} = \frac{1}{n}\sum\limits_{i=1}^n\mA_{\pi_i}x
    = \left(\frac{1}{n}\cdot\frac{n}{q}\sum\limits_{i = 1}^m q\Diag(P_i)\right)x
    = \mI x
    = x,
  \end{eqnarray*}
  for all $i \in \{1, \dots, n\}, x \in \R^d$.

  Next, we establish the second moment:
  \begin{eqnarray*}
    \Exp{\norm{\cC_i(x)}^2} = \frac{r}{n}\cdot 0 + \sum\limits_{l = 1}^m\frac{q}{n}\sum\limits_{j\in P_l}\left|\frac{n}{q}x_j\right|^2
    = \frac{q}{n}\sum\limits_{j=1}^d\left|\frac{n}{q}x_j\right|^2
    = \frac{n}{q}\norm{x}^2,
  \end{eqnarray*}
  for all $i \in \{1, \dots, n\}, x \in \R^d$.

  The following equality will be useful for the AB assumption:
  \begin{eqnarray*}
    \Exp{\inp{\cC_i(x)}{\cC_j(y)}} &=& \Exp{\inp{\mA_{\pi_i}x}{\mA_{\pi_j}y}} \\
    &=& \sum_{k = 1}^m \left(\sum_{l \in P_k} \frac{n^2}{q^2}x_l y_l\right) \ProbArg{\pi_i \in P_k, \pi_j \in P_k} \\
    &=& \sum_{k = 1}^m \left(\sum_{l \in P_k} \frac{n^2}{q^2}x_l y_l\right) \frac{q(q - 1)}{n(n - 1)} \\
    &=& \frac{n(q - 1)}{q(n - 1)} \inp{x}{y},
  \end{eqnarray*}
  for all $i \neq j \in \{1, \dots, n\}, x, y \in \R^d$.
  Thus,
  \begin{eqnarray*}
    \Exp{\norm{\frac{1}{n}\sum\limits_{i = 1}^n\cC_i(a_i)}^2} &=& \frac{1}{qn}\sum\limits_{i=1}^n\norm{a_i}^2 + \frac{n(q-1)}{(n-1)qn^2}\sum\limits_{i\neq j}\inp{a_i}{a_j} \\
    &=& \left(\frac{1}{qn} - \frac{n(q-1)}{(n - 1)qn^2}\right)\sum\limits_{i=1}^n\norm{a_i}^2 + \frac{n(q-1)}{(n-1)q}\norm{\frac{1}{n}\sum\limits_{i = 1}^n a_i}^2 \\
    &=& \left(\frac{(n-1) - q + 1}{(n - 1)q}\right)\frac{1}{n}\sum\limits_{i=1}^n\norm{a_i}^2 + \frac{n(q-1)}{(n-1)q}\norm{\frac{1}{n}\sum\limits_{i = 1}^n a_i}^2 \\
    &=& \frac{n-q}{(n-1)q}\cdot\frac{1}{n}\sum\limits_{i=1}^n\norm{a_i}^2 + \frac{n(q-1)}{(n-1)q}\norm{\frac{1}{n}\sum\limits_{i = 1}^n a_i}^2 \\
    &=& \left(1 - \frac{n(q-1)}{(n-1)q}\right)\frac{1}{n}\sum\limits_{i=1}^n\norm{a_i}^2 + \frac{n(q-1)}{(n-1)q}\norm{\frac{1}{n}\sum\limits_{i = 1}^n a_i}^2,
  \end{eqnarray*}
  for all $a_1, \dots, a_n \in \R^d.$
  Hence, Assumption~\ref{ass:AB} is fulfilled with $A = B = 1 - \frac{n(q-1)}{(n-1)q}$.
\end{proof}

\subsection{Permutation of mappings}

\begin{definition}
  Let $Q_1, \dots, Q_n:\mathbb{R}^d\rightarrow\mathbb{R}^d$ be a collection of \emph{deterministic} mappings $\mathbb{R}^d\rightarrow \mathbb{R}^d$. Let $\pi = \left(\pi_1, \dots, \pi_n\right)$ be a random permutation of set $\{1, \dots, n\},$ where $n > 1$. Define $\cC_i \eqdef Q_{\pi_i}$.
  Assume that the following conditions hold:
  \begin{enumerate} 
    \item \label{ass1} There exists $\omega \geq 0$ such that $\Exp{\norm{\cC_i(x)}^2} \leq (\omega + 1)\norm{x}^2$ for all $i\in \{1, \dots, n\}$, ${x, y\in\mathbb{R}^d}$.
    \item \label{ass2} There exists $\theta \in\mathbb{R}$ such that $\sum\limits_{i = 1}^n\inp{Q_i(x)}{Q_i(y)} = \theta\inp{x}{y}$ for all $x, y\in\mathbb{R}^d$.
    \item \label{ass3} $\frac{1}{n}\sum\limits_{i = 1}^n Q_i(x) = x$ for all $x\in\mathbb{R}^d$.
  \end{enumerate} 
  We call the collection $\{\cC_i\}_{i = 1}^n$ the \emph{permutation of mappings.}
\end{definition}
\begin{lemma}
  Permutation of mappings belongs to $\mathbb{U}(A, B)$ with $A = \frac{\omega + 1}{n} - \frac{1}{n - 1}\left(1 - \frac{\theta}{n^2}\right)$ and $B = 1 - \frac{n}{n - 1}\left(1 - \frac{\theta}{n^2}\right)$.
\end{lemma}
\begin{proof}
  Unbiasedness follows directly from Condition~\ref{ass3}. Let us fix $a_1, \dots, a_n \in \R^d.$
  We shall now establish the AB assumption.
  \begin{align*}
    &\Exp{\norm{\frac{1}{n}\sum\limits_{i = 1}^n\cC_i(a_i) - \frac{1}{n}\sum\limits_{i = 1}^n a_i}^2} \\
    &= \Exp{\norm{\frac{1}{n}\sum\limits_{i = 1}^n(Q_{\pi_i}(a_i) - a_i)}^2} \\
    &= \frac{1}{n^2} \Exp{\sum\limits_{i = 1}^n \norm{Q_{\pi_i}(a_i) - a_i}^2 + \sum\limits_{i\neq j}\inp{Q_{\pi_i}(a_i) - a_i}{Q_{\pi_j}(a_j) - a_j}} \\
    &\leq \frac{1}{n^2} \Exp{\omega\sum\limits_{i = 1}^n \norm{a_i}^2 + \sum\limits_{i\neq j}\inp{Q_{\pi_i}(a_i) - a_i}{Q_{\pi_j}(a_j) - a_j}} \\
    &= \frac{1}{n^2}\left( \omega\sum\limits_{i = 1}^n \norm{a_i}^2 + \Exp{\sum\limits_{i\neq j}\inp{Q_{\pi_i}(a_i)}{Q_{\pi_j}(a_j)}} - \sum\limits_{i\neq j} \inp{a_i}{a_j} \right) \\
    &= \frac{1}{n^2} \left(\omega\sum\limits_{i = 1}^n \norm{a_i}^2 - \norm{\sum\limits_{i = 1}^na_i}^2 + \sum\limits_{i = 1}^n \norm{a_i}^2 + \sum\limits_{i\neq j}\Exp{\inp{Q_{\pi_i}(a_i)}{Q_{\pi_j}(a_j)}}\right)\\
    &= \frac{(\omega + 1)}{n}\frac{1}{n} \sum\limits_{i = 1}^n \norm{a_i}^2 - \norm{\frac{1}{n}\sum\limits_{i = 1}^na_i}^2 + \frac{1}{n^2}\sum\limits_{i\neq j} \Exp{\inp{Q_{\pi_i}(a_i)}{Q_{\pi_j}(a_j)}}\\
  \end{align*}
  Let us now compute $\frac{1}{n^2}\sum\limits_{i\neq j} \Exp{\inp{Q_{\pi_i}(a_i)}{Q_{\pi_j}(a_j)}}$.
  \begin{eqnarray*}
    \frac{1}{n^2}\sum\limits_{i\neq j} \Exp{\inp{Q_{\pi_i}(a_i)}{Q_{\pi_j}(a_j)}} &=& \frac{1}{n^2}\sum\limits_{i \neq j}\frac{1}{n (n - 1)}\sum\limits_{u \neq v}\inp{Q_u(a_i)}{Q_v(a_j)} \\
    &=& \frac{1}{n (n - 1)}\sum\limits_{i \neq j}\frac{1}{n^2}\sum\limits_{u \neq v}\inp{Q_u(a_i)}{Q_v(a_j)}.
  \end{eqnarray*}
  Now,
  \begin{eqnarray*}
    \frac{1}{n^2}\sum\limits_{u \neq v}\inp{Q_u(x)}{Q_v(y)} &=& \frac{1}{n^2}\sum\limits_{u = 1}^n\sum\limits_{v = 1}^n \inp{Q_u(x)}{Q_v(y)} - \frac{1}{n^2}\sum\limits_{u = 1}^n\inp{Q_u(x)}{Q_v(y)} \\
    &\stackrel{\text{Condition~\ref{ass3}}}{=}& \inp{x}{y} - \frac{1}{n^2}\sum\limits_{u = 1}^n\inp{Q_u(x)}{Q_u(y)} \\
    &\stackrel{\text{Condition~\ref{ass2}}}{=}& \left(1 - \frac{\theta}{n^2}\right)\inp{x}{y}, \quad \forall x, y \in \R^d.
  \end{eqnarray*}

  Hence,
  $$\frac{1}{n^2}\sum\limits_{i\neq j} \Exp{\inp{Q_{\pi_i}(a_i)}{Q_{\pi_j}(a_j)}} = \frac{1}{n(n - 1)}\left(1 - \frac{\theta}{n^2}\right)\sum\limits_{i \neq j}\inp{a_i}{a_j}.$$
  Finally,
  \begin{align*}
    &\Exp{\norm{\frac{1}{n}\sum\limits_{i = 1}^n\cC_i(a_i) - \frac{1}{n}\sum\limits_{i = 1}^n a_i}^2} \\
    &\leq \frac{(\omega + 1)}{n}\frac{1}{n} \sum\limits_{i = 1}^n \norm{a_i}^2 - \norm{\frac{1}{n}\sum\limits_{i = 1}^na_i}^2 + \frac{1}{n(n - 1)}\left(1 - \frac{\theta}{n^2}\right)\sum\limits_{i \neq j}\inp{a_i}{a_j} \\
    &= \left(\frac{\omega + 1}{n} - \frac{1}{n - 1}\left(1 - \frac{\theta}{n^2}\right) \right)\frac{1}{n}\sum\limits_{i = 1}^n \norm{a_i}^2 
    - \left(1 - \frac{n}{n - 1}\left(1 - \frac{\theta}{n^2}\right)\right)\norm{\frac{1}{n}\sum\limits_{i = 1}^n a^i}^2.
  \end{align*}
  Hence, Assumption~\ref{ass:AB} is fulfilled with $A = \frac{\omega + 1}{n} - \frac{1}{n - 1}\left(1 - \frac{\theta}{n^2}\right)$, $B = 1 - \frac{n}{n - 1}\left(1 - \frac{\theta}{n^2}\right)$.
\end{proof}
\clearpage
\section{Analysis of Complexity Bounds}

In this section, we analyze the complexities bounds of optimization methods, 
and typically these bounds have a structure of a function that we analyze in the following lemma.

\begin{lemma}
  \label{lemma:technical:complexity_bounds}
  Let us consider function
  $$f(x, y) = (x + (1 - x)y)\left(a + b\sqrt{\left(\frac{1}{x} - 1\right)\left(\frac{1}{y} - 1\right)}\right),$$
  where $x \in (0, 1], y \in [y_{\min}, 1],$ $y_{\min} \in (0, 1 / 2],$ $a \geq 0,$ and $b \geq 0,$ then 
  $$f(x, y) \geq \frac{1}{2}\min\{a, ay_{\min} + b\}, \quad \forall x, y \in (0, 1].$$
\end{lemma}

\begin{proof}
  First, let us assume that $x \geq 1 / 2$. Then,
  \begin{eqnarray*}
    f(x, y) \geq \frac{a}{2}.
  \end{eqnarray*}
  Second, let us assume that $y \geq 1 / 2$. Then,
  \begin{eqnarray*}
    f(x, y) \geq a\left(x + \frac{1}{2}(1 - x)\right) \geq\frac{a}{2}.
  \end{eqnarray*}
  Finally, let us assume, that $y \leq 1 / 2$ and $x \leq 1 / 2$. Then,
  \begin{eqnarray*}
    f(x, y) &=& (x + (1 - x)y)\left(a + b\sqrt{\left(\frac{1}{x} - 1\right)\left(\frac{1}{y} - 1\right)}\right) \\
    &\geq& a(1 - x)y + b(x + (1 - x)y)\sqrt{\left(\frac{1}{x} - 1\right)\left(\frac{1}{y} - 1\right)} \\
    &\geq& a(1 - x)y + b(x(1 - y) + (1 - x)y)\sqrt{\left(\frac{1}{x} - 1\right)\left(\frac{1}{y} - 1\right)} \\
    &\geq& a(1 - x)y + bxy\left(\left(\frac{1}{y} - 1\right) + \left(\frac{1}{x} - 1\right)\right)\sqrt{\left(\frac{1}{x} - 1\right)\left(\frac{1}{y} - 1\right)} \\
    &\geq& a(1 - x)y + 2bxy\left(\frac{1}{x} - 1\right)\left(\frac{1}{y} - 1\right) \\
    &\geq& a(1 - x)y + 2b\left(1 - x\right)\left(1 - y\right) \\
    &\geq& \frac{ay_{\min}}{2} + \frac{b}{2}.
  \end{eqnarray*}
\end{proof}

\label{sec:appendix:complexity_bounds}

\subsection{Nonconvex optimization}

\label{sec:appendix:complexity_bound:nonconvex}

\subsubsection{Case $n \leq d$}

We analyze case, when $n \leq d.$ 
For simplicity, we assume that $n \mid d$, $n > 1,$ and $d > 1$. For Perm$K$ from Definition~\ref{def:PermK-1}, constants $A = B = 1$ in AB inequality (see Lemma~\ref{thm:PermK-1}).
We define communication complexity of \algname{MARINA} with Perm$K$ as $N_{\textnormal{Perm$K$}}(p),$
where $p$ is a parameter of \algname{MARINA}, and \algname{MARINA} with Rand$K$ as $N_{\textnormal{Rand$K$}}(p, k)$,
where $k$ is a parameter of Rand$K$.
From Theorem~\ref{theorem:AB}, we have that oracle complexity of \algname{MARINA} with Perm$K$ is equal to
$$\cO\left(\frac{\Delta_0}{\varepsilon}\left(L_- + \sqrt{\frac{1-p}{p}}L_\pm\right)\right).$$
During each iteration of \algname{MARINA}, on average, each node sends the number of bits equal to
$$\cO\left(pd + (1 - p)\frac{d}{n}\right),$$
thus, the communication complexity predicted by theory is
\begin{align}
  \label{eq:complexity_bound:perm_k_with_parameters}
  N_{\textnormal{Perm$K$}}(p) \eqdef \frac{\Delta_0}{\varepsilon}\left(pd + (1 - p)\frac{d}{n}\right)\left(L_- + \sqrt{\frac{1-p}{p}}L_\pm\right)
\end{align}
up to a constant factor.

Analogously, for Rand$K$, the communication complexity predicted by theory is
\begin{align}
  \label{eq:complexity_bound:rand_k_with_parameters}
  N_{\textnormal{Rand$K$}}(p, k) \eqdef \frac{\Delta_0}{\varepsilon}\left(pd + (1 - p)k\right)\left(L_- + \sqrt{\frac{1-p}{p}\frac{\frac{d}{k} - 1}{n}}L_+\right)
\end{align}
up to a constant factor. 
To the best of our knowledge, this is the state-of-the-art theoretical communication complexity bound for the Rand$K$ compressor in the non-convex regime.

Finally, for Top$K$, by Theorem~\ref{thm:main-distrib}, the theoretical communication complexity is

\begin{eqnarray}
  \label{eq:complexity_bound:top_k_with_parameters}
  N_{\textnormal{Top$K$}}(k) &\eqdef& \frac{\Delta_0}{\varepsilon}k\left(L_- + L_+\frac{d - k + \sqrt{d^2 - dk}}{k}\right)
\end{eqnarray}
up to a constant factor. We consider the variant of \algname{EF21}, where $g_i^0$ are initialized with gradients $\nabla f_i(x^0)$, for all $i \in \{1, \dots, n\}$, thus $\Exp{G^0} = 0$ in Theorem~\ref{thm:main-distrib}.

The following lemma will help us to choose the optimal parameters of
$N_{\textnormal{Perm$K$}}(p)$, $N_{\textnormal{Rand$K$}}(p, k)$, and $N_{\textnormal{Top$K$}}(k)$.

\begin{lemma}
  \leavevmode
  For communication complexity $N_{\textnormal{Perm$K$}}(p)$ of \algname{MARINA} with Perm$K$, 
  communication complexity $N_{\textnormal{Rand$K$}}(p, k)$ of \algname{MARINA} with Rand$K$ and 
  communication complexity $N_{\textnormal{Top$K$}}(k)$ of \algname{EF21} with Top$K$
  defined in (\ref{eq:complexity_bound:perm_k_with_parameters}), (\ref{eq:complexity_bound:rand_k_with_parameters}) and (\ref{eq:complexity_bound:top_k_with_parameters})
  the following inequalities hold:
  \label{lemma:optimal_parameters_of_methods}
  \begin{enumerate}
    \item Lower bounds:
    \begin{eqnarray*}
      N_{\textnormal{Perm$K$}}(p) \geq \frac{\Delta_0}{2\varepsilon}\min\left\{dL_-, \frac{d}{n}L_- + \frac{dL_\pm}{\sqrt{n}}\right\}, \quad \forall p \in (0, 1].
    \end{eqnarray*}
    Upper bounds:
    \begin{align}
      \label{eq:appendix:complexity_bounds:perm_k}
      N_{\textnormal{Perm$K$}}\left(\frac{1}{n}\right) &\leq \frac{2\Delta_0}{\varepsilon}\left(\frac{d}{n}L_- + \frac{dL_\pm}{\sqrt{n}}\right), \\
      \label{eq:appendix:complexity_bounds:perm_k_2}
      N_{\textnormal{Perm$K$}}\left(1\right) &= \frac{\Delta_0d L_{-}}{\varepsilon}.
    \end{align}
    \item
    Lower bounds:
    \begin{eqnarray*}
      N_{\textnormal{Rand$K$}}(p, k) \geq \frac{\Delta_0}{2\varepsilon}\min\left\{dL_-, \frac{dL_{+}}{\sqrt{n}}\right\}, \quad \forall p \in (0, 1], \forall k \in \{1, \dots, d\},
    \end{eqnarray*}
    Upper bounds:
    For all $k \in \{1, \dots, d / \sqrt{n}\},$ $p = k / d,$
    \begin{eqnarray}
      \label{eq:appendix:complexity_bounds:rand_k}
      N_{\textnormal{Rand$K$}}\left(p, k\right) \leq \frac{4\Delta_0dL_{+}}{\varepsilon\sqrt{n}}.
    \end{eqnarray}
    Moreover, for all $k \in \{1, \dots, d\}, p = 1,$
    \begin{eqnarray}
      \label{eq:appendix:complexity_bounds:rand_k_2}
      N_{\textnormal{Rand$K$}}\left(1, k\right) = \frac{\Delta_0d L_{-}}{\varepsilon}.
    \end{eqnarray}
    \item
    \begin{eqnarray}
      \label{eq:appendix:complexity_bounds:top_k}
      \min_{k \in \{1, \dots, d\}} N_{\textnormal{Top$K$}}(k) = N_{\textnormal{Top$K$}}(d) = \frac{\Delta_0dL_-}{\varepsilon}
    \end{eqnarray}
  \end{enumerate}
\end{lemma}

\begin{proof}
  \leavevmode
  \begin{enumerate}
  \item
  We start with the first inequality:
  \begin{align*}
    N_{\textnormal{Perm$K$}}(p) &= \frac{\Delta_0}{\varepsilon}\left(pd + (1 - p)\frac{d}{n}\right)\left(L_- + \sqrt{\frac{1-p}{p}}L_\pm\right) \\
    &=\frac{\Delta_0}{\varepsilon}\left(p + (1 - p)\frac{1}{n}\right)\left(dL_- + dL_\pm\sqrt{\frac{1-p}{p}}\right) \\
    &=\frac{\Delta_0}{\varepsilon}\left(p + (1 - p)\frac{1}{n}\right)\left(dL_- + dL_\pm\sqrt{\frac{n}{n}}\sqrt{\frac{1-p}{p}}\right) \\
    &\geq\frac{\Delta_0}{\varepsilon}\left(p + (1 - p)\frac{1}{n}\right)\left(dL_- + 
    \frac{dL_\pm}{\sqrt{n}}\sqrt{\left(\frac{1}{p} - 1\right)\left(n - 1\right)}\right).
  \end{align*}
  Using Lemma~\ref{lemma:technical:complexity_bounds} with $a = dL_-,$ $b = \frac{dL_\pm}{\sqrt{n}},$ and $y_{\min} = 1 / n$, we get
  \begin{align*}
    N_{\textnormal{Perm$K$}}(p) &= \frac{\Delta_0}{\varepsilon}\left(pd + (1 - p)\frac{d}{n}\right)\left(L_- + \sqrt{\frac{1-p}{p}}L_\pm\right) \\
    &\geq \frac{\Delta_0}{2\varepsilon}\min\left\{dL_-, \frac{d}{n}L_- + \frac{dL_\pm}{\sqrt{n}}\right\}. \\
  \end{align*}
  for all $p \in (0, 1].$
  We can obtain the bound \ref{eq:appendix:complexity_bounds:perm_k} if we take $p = 1 / n$:
  \begin{align*}
    N_{\textnormal{Perm$K$}}\left(\frac{1}{n}\right) &= \frac{\Delta_0}{\varepsilon}\left(\frac{d}{n} + \left(1 - \frac{1}{n}\right)\frac{d}{n}\right)\left(L_- + \sqrt{n - 1}L_\pm\right) \\
    &\leq \frac{2\Delta_0}{\varepsilon}\left(\frac{d}{n}L_- + \frac{d}{\sqrt{n}}L_\pm\right).
  \end{align*}
    We obtain the equality \ref{eq:appendix:complexity_bounds:perm_k_2} by taking $p = 1$.
  \item 
  \begin{align*}
    N_{\textnormal{Rand$K$}}(p, k) &= \frac{\Delta_0}{\varepsilon}\left(pd + (1 - p)k\right)\left(L_- + \sqrt{\frac{1-p}{p}\frac{\frac{d}{k} - 1}{n}}L_+\right) \\
    &=\frac{\Delta_0}{\varepsilon}\left(p + (1 - p)\frac{k}{d}\right)\left(dL_- + \frac{dL_+}{\sqrt{n}}\sqrt{\left(\frac{1}{p} - 1\right)\left(\frac{d}{k} - 1\right)}\right).
  \end{align*}
  Using Lemma~\ref{lemma:technical:complexity_bounds} with $a = dL_-,$ $b = \frac{dL_\pm}{\sqrt{n}},$ and $y_{\min} = 1 / d$, we get
  \begin{align*}
    &N_{\textnormal{Rand$K$}}(p, k) \geq \frac{\Delta_0}{2\varepsilon}\min\left\{dL_-, L_{-} + \frac{dL_{+}}{\sqrt{n}}\right\} \geq \frac{\Delta_0}{2\varepsilon}\min\left\{dL_-, \frac{dL_{+}}{\sqrt{n}}\right\},
  \end{align*}
  for all $p \in (0, 1], k \in \{1, \dots, d\}.$ We can obtain the bound \ref{eq:appendix:complexity_bounds:rand_k}
  if we take $k \in \{1, \dots, d / \sqrt{n}\}$ and $p = k / d$:
  \begin{align*}
    N_{\textnormal{Rand$K$}}(p, k) &\leq \frac{2\Delta_0}{\varepsilon}\left(kL_- + k\left(\frac{d}{k} - 1\right)\frac{L_+}{\sqrt{n}}\right) \\
    &\leq\frac{2\Delta_0}{\varepsilon}\left(\frac{dL_{-}}{\sqrt{n}} + \frac{dL_+}{\sqrt{n}}\right) \leq \frac{4\Delta_0dL_+}{\varepsilon\sqrt{n}}.
  \end{align*}
  The equality \ref{eq:appendix:complexity_bounds:rand_k_2} is obtained by taking $p = 1.$
  \item
    This part is easily proved, using $L_{-} \leq L_{+}$ from Lemma~\ref{lem:08hhdf}, and directly minimizing (\ref{eq:appendix:complexity_bounds:top_k}).
  \end{enumerate}
\end{proof}

In Table~\ref{table:communication_complexity}, we summarize bounds 
(\ref{eq:appendix:complexity_bounds:perm_k}), (\ref{eq:appendix:complexity_bounds:perm_k_2}), 
(\ref{eq:appendix:complexity_bounds:rand_k}), (\ref{eq:appendix:complexity_bounds:rand_k_2}), and 
(\ref{eq:appendix:complexity_bounds:top_k}).

\subsubsection{Case $n \geq d$}

Now, we analyze case, when $n \geq d.$  For simplicity, without losing the generality, we assume that $d \mid n,$ $n > 1,$ and $d > 1$.
Then, Perm$K$ from Definition~\ref{def:PermK-2} satisfies the AB inequality with $A = B = \frac{d - 1}{n - 1}$.

In each iteration of \algname{MARINA}, on average, Perm$K$ sends
$$\cO\left(pd + (1 - p)\right)$$
bits, thus the theoretical communication complexity is
\begin{align}
  \label{eq:complexity_bound:perm_k_n_geq_d_with_parameters}
  N_{\textnormal{Perm$K$}}(p) \eqdef \frac{\Delta_0}{\varepsilon}\left(pd + (1 - p)\right)\left(L_- + \sqrt{\frac{1-p}{p}\frac{d - 1}{n - 1}}L_\pm\right)
\end{align}
up to a constant factor.

\begin{lemma}
  \label{lemma:optimal_parameters_of_methods_n_geq_d}
  For communication complexity $N_{\textnormal{Perm$K$}}(p)$ of \algname{MARINA} with Perm$K$, 
  communication complexity $N_{\textnormal{Rand$K$}}(p, k)$ of \algname{MARINA} with Rand$K$ and 
  communication complexity $N_{\textnormal{Top$K$}}(k)$ of \algname{EF21} with Top$K$
  defined in (\ref{eq:complexity_bound:perm_k_n_geq_d_with_parameters}), (\ref{eq:complexity_bound:rand_k_with_parameters}) and (\ref{eq:complexity_bound:top_k_with_parameters})
  the following inequalities hold:
  \leavevmode
  \begin{enumerate}
    \item
    Lower bounds:
    \begin{eqnarray*}
      N_{\textnormal{Perm$K$}}(p) \geq \frac{\Delta_0}{2\varepsilon}\min\left\{dL_-, L_- + \frac{dL_\pm}{\sqrt{n}}\right\}, \quad \forall p \in (0, 1].
    \end{eqnarray*}
    Upper bounds:
    \begin{align}
      \label{eq:appendix:complexity_bounds:perm_k_n_geq_d}
      &N_{\textnormal{Perm$K$}}\left(\frac{1}{d}\right) \leq \frac{4\Delta_0}{\varepsilon}\left(L_- + \frac{dL_\pm}{\sqrt{n}}\right), \\
      \label{eq:appendix:complexity_bounds:perm_k_n_geq_d_2}
      &N_{\textnormal{Perm$K$}}\left(1\right) = \frac{\Delta_0d L_{-}}{\varepsilon}.
    \end{align}
    \item
    Lower bounds:
    \begin{eqnarray*}
      N_{\textnormal{Rand$K$}}(p, k) \geq \frac{\Delta_0}{2\varepsilon}\min\left\{dL_-, L_{-} + \frac{dL_{+}}{\sqrt{n}}\right\}, \quad \forall p \in (0, 1], \forall k \in \{1, \dots, d\},
    \end{eqnarray*}
    Upper bounds:
    \begin{eqnarray}
      \label{eq:appendix:complexity_bounds:rand_k_n_geq_d}
      N_{\textnormal{Rand$K$}}\left(\frac{1}{d}, 1\right) \leq \frac{2\Delta_0}{\varepsilon}\left(L_- + \frac{dL_+}{\sqrt{n}}\right),
    \end{eqnarray}
    Moreover, for all $k \in \{1, \dots, d\}, p = 1,$
    \begin{eqnarray}
      \label{eq:appendix:complexity_bounds:rand_k_n_geq_d_2}
      N_{\textnormal{Rand$K$}}\left(1, k\right) = \frac{\Delta_0d L_{-}}{\varepsilon}.
    \end{eqnarray}
    \item
    \begin{eqnarray}
      \label{eq:appendix:complexity_bounds:top_k_n_geq_d}
      \min_{k \in \{1, \dots, d\}} N_{\textnormal{Top$K$}}(k) = N_{\textnormal{Top$K$}}(d) = \frac{\Delta_0dL_-}{\varepsilon}
    \end{eqnarray}
  \end{enumerate}
\end{lemma}

\begin{proof}
  \leavevmode
  \begin{enumerate}
  \item
  \begin{eqnarray*}
    &N_{\textnormal{Perm$K$}}(p) = \frac{\Delta_0}{\varepsilon}\left(pd + (1 - p)\right)\left(L_- + \sqrt{\frac{1-p}{p}\frac{d - 1}{n - 1}}L_\pm\right) \\
    &\geq \frac{\Delta_0}{\varepsilon}\left(p + (1 - p)\frac{1}{d}\right)\left(dL_- + \frac{dL_\pm}{\sqrt{n}}\sqrt{\left(\frac{1}{p} - 1\right)\left(d - 1\right)}\right)
  \end{eqnarray*}
  Using Lemma~\ref{lemma:technical:complexity_bounds} with $a = dL_-,$ $b = \frac{dL_\pm}{\sqrt{n}},$ and $y_{\min} = 1 / d,$ we get
  \begin{align*}
    &N_{\textnormal{Perm$K$}}(p) \geq \frac{\Delta_0}{2\varepsilon}\min\left\{dL_-, L_- + \frac{dL_\pm}{\sqrt{n}}\right\} \\
  \end{align*}
  for all $p \in (0, 1].$
  We can show the bound \ref{eq:appendix:complexity_bounds:perm_k_n_geq_d} if we take $p = 1 / d$:
  \begin{align*}
    N_{\textnormal{Perm$K$}}\left(\frac{1}{d}\right) = \frac{\Delta_0}{\varepsilon}\left(1 + \left(1 - \frac{1}{d}\right)1\right)\left(L_- + \frac{d - 1}{\sqrt{n - 1}}L_\pm\right) \leq \frac{4\Delta_0}{\varepsilon}\left(L_- + \frac{dL_\pm}{\sqrt{n}}\right)
  \end{align*}
  The bound \ref{eq:appendix:complexity_bounds:perm_k_n_geq_d_2} is obtained by taking $p = 1$.
  \item
  As in Lemma~\ref{lemma:optimal_parameters_of_methods} we can get, that
  \begin{align*}
    N_{\textnormal{Rand$K$}}(p, k) \geq \frac{\Delta_0}{2\varepsilon}\min\left\{dL_-, L_{-} + \frac{dL_{+}}{\sqrt{n}}\right\},
  \end{align*}
  for all $p \in (0, 1], k \in \{1, \dots, d\}.$
  Moreover, if we take $p = 1 / d$ and $k = 1$, we have
  \begin{align*}
    N_{\textnormal{Rand$K$}}\left(\frac{1}{d}, 1\right) \leq \frac{2\Delta_0}{\varepsilon}\left(L_- + \frac{dL_+}{\sqrt{n}}\right).
  \end{align*}
  The bound \ref{eq:appendix:complexity_bounds:rand_k_n_geq_d_2} is obtained by taking $p = 1$.
  \item
  For Top$K$, the reasoning the same as in Lemma~\ref{lemma:optimal_parameters_of_methods}.
  \end{enumerate}
\end{proof}
% For Top$K$, in case of $n \geq d$, communication complexity remains to be defined by (\ref{eq:appendix:complexity_bounds:top_k}).

In Table~\ref{table:communication_complexity}, we summarize bounds 
(\ref{eq:appendix:complexity_bounds:perm_k_n_geq_d}), (\ref{eq:appendix:complexity_bounds:perm_k_n_geq_d_2}),
(\ref{eq:appendix:complexity_bounds:rand_k_n_geq_d}), (\ref{eq:appendix:complexity_bounds:rand_k_n_geq_d_2}), and 
(\ref{eq:appendix:complexity_bounds:top_k_n_geq_d}).

\subsection{P\L\,assumption}

\subsubsection{Case $n \leq d$}

\label{section:sec:appendix:complexity_bound:pl:n_leq_d}

Using the same reasoning as in Appendix~\ref{sec:appendix:complexity_bound:nonconvex}, 
Theorem~\ref{theorem:AB_PL} and Theorem~\ref{thm:PL-main}, we can show 
that communication complexities predicted by theory are equal to
\begin{align}
  \label{section:sec:appendix:complexity_bound:pl:n_leq_d:perm}
  &N_{\textnormal{Perm$K$}}(p) \eqdef \log\frac{\Delta_0}{\varepsilon}\left(pd + (1 - p)\frac{d}{n}\right)\max\left\{\left(\frac{L_-}{\mu} + \sqrt{\frac{2\left(1-p\right)}{p}}\frac{L_\pm}{\mu}\right), \frac{1}{p}\right\},\\
  \label{section:sec:appendix:complexity_bound:pl:n_leq_d:rand}
  &N_{\textnormal{Rand$K$}}(p, k) \eqdef \log\frac{\Delta_0}{\varepsilon}\left(pd + (1 - p)k\right)\max\left\{\left(\frac{L_-}{\mu} + \sqrt{\frac{2\left(1-p\right)}{p}\frac{\frac{d}{k} - 1}{n}}\frac{L_+}{\mu}\right), \frac{1}{p}\right\},\\
  \label{section:sec:appendix:complexity_bound:pl:n_leq_d:top}
  &N_{\textnormal{Top$K$}}(k) \eqdef \log\frac{\Delta_0}{\varepsilon}k\max\left\{\left(\frac{L_-}{\mu} + \frac{L_+}{\mu}\frac{d - k + \sqrt{d^2 - dk}}{k}\right),\frac{1}{1 - \sqrt{1 - \frac{k}{d}}}\right\}.
\end{align}
up to a constant factor. 

\begin{lemma}
  \leavevmode
  For communication complexity $N_{\textnormal{Perm$K$}}(p)$ of \algname{MARINA} with Perm$K$, 
  communication complexity $N_{\textnormal{Rand$K$}}(p, k)$ of \algname{MARINA} with Rand$K$ and 
  communication complexity $N_{\textnormal{Top$K$}}(k)$ of \algname{EF21} with Top$K$
  defined in (\ref{section:sec:appendix:complexity_bound:pl:n_leq_d:perm}), (\ref{section:sec:appendix:complexity_bound:pl:n_leq_d:rand}) and (\ref{section:sec:appendix:complexity_bound:pl:n_leq_d:top})
  the following inequalities hold\footnote{In the lemma, we use ``Big Theta'' notation, which means, that if $f(x) = \Theta(g(x)),$ then $f$ is bounded both above and below by $g$ asymptotically up to a logarithmic factor.}:
  \label{lemma:optimal_parameters_of_methods_pl}
  \begin{enumerate}
    \item
    \begin{eqnarray*}
      \inf_{p \in (0, 1]} N_{\textnormal{Perm$K$}}(p) = \Theta\left(\max\left\{\frac{1}{\mu}\min\left\{dL_{-},\frac{d}{n}L_- + \frac{d}{\sqrt{n}}L_\pm\right\}, d\right\}\right),
    \end{eqnarray*}
    \item
    \begin{eqnarray*}
      \inf_{p \in (0, 1], k \in \{1, \dots, d\}} N_{\textnormal{Rand$K$}}(p, k) = \Theta\left(\max\left\{\frac{1}{\mu}\min\left\{dL_{-},\frac{d}{\sqrt{n}}L_{+}\right\}, d\right\}\right),
    \end{eqnarray*}
    \item
    \begin{eqnarray*}
      \min_{k \in \{1, \dots, d\}} N_{\textnormal{Top$K$}}(k) = \Theta\left(\frac{dL_-}{\mu}\right).
    \end{eqnarray*}
  \end{enumerate}
\end{lemma}

\begin{proof}
Rearranging (\ref{section:sec:appendix:complexity_bound:pl:n_leq_d:perm}), (\ref{section:sec:appendix:complexity_bound:pl:n_leq_d:rand}) and (\ref{section:sec:appendix:complexity_bound:pl:n_leq_d:top}), we get
\begin{align*}
  &N_{\textnormal{Perm$K$}}(p) = \log\frac{\Delta_0}{\varepsilon}\max\left\{\left(pd + (1 - p)\frac{d}{n}\right)\left(\frac{L_-}{\mu} + \sqrt{\frac{2\left(1-p\right)}{p}}\frac{L_\pm}{\mu}\right), d + \frac{(1 - p)d}{pn}\right\},\\
  &N_{\textnormal{Rand$K$}}(p, k) = \log\frac{\Delta_0}{\varepsilon}\max\left\{\left(pd + (1 - p)k\right)\left(\frac{L_-}{\mu} + \sqrt{\frac{2\left(1-p\right)}{p}\frac{\frac{d}{k} - 1}{n}}\frac{L_+}{\mu}\right), d + \frac{(1 - p)k}{p}\right\},\\
  &N_{\textnormal{Top$K$}}(k) = \log\frac{\Delta_0}{\varepsilon}\max\left\{k\left(\frac{L_-}{\mu} + \frac{L_+}{\mu}\frac{d - k + \sqrt{d^2 - dk}}{k}\right),\frac{k}{1 - \sqrt{1 - \frac{k}{d}}}\right\}.
\end{align*}

Note, that 
\begin{align*}
  \frac{k}{1 - \sqrt{1 - \frac{k}{d}}} \geq d, \quad \forall k \in \{1, \dots, d\},
\end{align*}

thus in all complexities, the second terms inside the $\max$ brackets are at least $d$.

Analysis of first terms inside the $\max$ brackets is the same as in Lemma~\ref{lemma:optimal_parameters_of_methods}.
\end{proof}

In Table~\ref{table:communication_complexity_pl}, we provide complexity bounds with optimal parameters of algorithms.

\subsubsection{Case $n \geq d$}

\label{section:sec:appendix:complexity_bound:pl:n_ged_d}

The only difference here is that the communication complexity of Perm$K$ predicted by our theory is the following:

\begin{align}
  \label{section:sec:appendix:complexity_bound:pl:n_geq_d:perm}
  &N_{\textnormal{Perm$K$}}(p) \eqdef \log\frac{\Delta_0}{\varepsilon}\left(pd + (1 - p)\right)\max\left\{\left(\frac{L_-}{\mu} + \sqrt{\frac{2\left(1-p\right)}{p}\frac{d - 1}{n - 1}}\frac{L_\pm}{\mu}\right), \frac{1}{p}\right\}.
\end{align}

\begin{lemma}
  \leavevmode
  For communication complexity $N_{\textnormal{Perm$K$}}(p)$ of \algname{MARINA} with Perm$K$, 
  communication complexity $N_{\textnormal{Rand$K$}}(p, k)$ of \algname{MARINA} with Rand$K$ and 
  communication complexity $N_{\textnormal{Top$K$}}(k)$ of \algname{EF21} with Top$K$
  defined in (\ref{section:sec:appendix:complexity_bound:pl:n_geq_d:perm}), (\ref{section:sec:appendix:complexity_bound:pl:n_leq_d:rand}) and (\ref{section:sec:appendix:complexity_bound:pl:n_leq_d:top})
  the following inequalities hold:
  \label{lemma:optimal_parameters_of_methods_pl:n_geq_d}
  \begin{enumerate}
    \item
    \begin{eqnarray*}
      \inf_{p \in (0, 1]} N_{\textnormal{Perm$K$}}(p) = \Theta\left(\max\left\{\frac{1}{\mu}\min\left\{dL_{-},L_- + \frac{d}{\sqrt{n}}L_\pm\right\}, d\right\}\right),
    \end{eqnarray*}
    \item
    \begin{eqnarray*}
      \inf_{p \in (0, 1], k \in \{1, \dots, d\}} N_{\textnormal{Rand$K$}}(p, k) = \Theta\left(\max\left\{\frac{1}{\mu}\min\left\{dL_{-},L_- + \frac{d}{\sqrt{n}}L_{+}\right\}, d\right\}\right),
    \end{eqnarray*}
    \item
    \begin{eqnarray*}
      \min_{k \in \{1, \dots, d\}} N_{\textnormal{Top$K$}}(k) = \Theta\left(\frac{dL_-}{\mu}\right).
    \end{eqnarray*}
  \end{enumerate}
\end{lemma}

The proof of Lemma~\ref{lemma:optimal_parameters_of_methods_pl:n_geq_d} is the same as in Lemma~\ref{lemma:optimal_parameters_of_methods_pl}.

Using the same reasoning as before, we provide complexity bounds in Table~\ref{table:communication_complexity_pl}.

\newpage
\section{Group Hessian Variance}
%%%%%%%%%%%%%%%%%%%
\label{sec:group_hessian_variance}

We showed the communication complexity improvement of \algname{MARINA} algorithm with Perm$K$ under the assumption that $L_{\pm} \ll L_{-}$. 
In general, $L_{\pm}$ can be large; however, we can still use the notion of $L_{\pm}$ but in a different way, 
by splitting the functions into several groups where $L_{\pm}$ is small. 

We split a set $\{1, \cdots, n\}$ into nonempty sets $\{\mathcal{G}_k\}_{k=1}^{g},$ 
$\bigcup_{k=1}^{g}\mathcal{G}_k = \{1, \cdots, n\},$ $\mathcal{G}_i \bigcap \mathcal{G}_j = \emptyset, $
for all $i \neq j \in \{1, \cdots, g\},$
and $|\mathcal{G}_k| > 0$, for all $k \in \{1, \cdots, g\}.$ Let us fix some set $\mathcal{G}_k$ and 
define functions
\begin{align*}
  &\cL_{-}^{\mathcal{G}_k}(x,y) \eqdef \norm{\frac{1}{|\mathcal{G}_k|} \sum_{i \in |\mathcal{G}_k|} \left(\nabla f_i(x) - \nabla f_i(y)\right)}^2,\\
  & \cL_{+}^{\mathcal{G}_k}(x,y) \eqdef \frac{1}{|\mathcal{G}_k|} \sum_{i \in |\mathcal{G}_k|} \norm{\nabla f_i(x) - \nabla f_i(y)}^2, \\
  &\cL_{\pm}^{\mathcal{G}_k}(x,y) \eqdef \cL_{+}^{\mathcal{G}_k}(x,y) - \cL_{-}^{\mathcal{G}_k}(x,y)
\end{align*}
% analogous to (\ref{eq:nou98d0-9yu9fd}) and (\ref{eq:nou98d0-9yu9fd-98y98fd}) 
and the smallest constants $L^{\mathcal{G}_k}_{-}, L^{\mathcal{G}_k}_{+}, L^{\mathcal{G}_k}_{\pm}$ for functions $\cL_{-}^{\mathcal{G}_k}(x,y), \cL_{+}^{\mathcal{G}_k}(x,y),$ and $\cL_{\pm}^{\mathcal{G}_k}(x,y), $ such that
\begin{align*}
  \cL_{-}^{\mathcal{G}_k}(x,y) \leq \left(L^{\mathcal{G}_k}_{-}\right)^2\norm{x - y}^2,
  \cL_{+}^{\mathcal{G}_k}(x,y) \leq \left(L^{\mathcal{G}_k}_{+}\right)^2\norm{x - y}^2,
  \cL_{\pm}^{\mathcal{G}_k}(x,y) \leq \left(L^{\mathcal{G}_k}_{\pm}\right)^2\norm{x - y}^2,
\end{align*}
for all $k \in \{1, \cdots, g\}, x, y \in \R^d.$

In this section, we have the following assumption about groups.

\begin{assumption} \label{ass:group_ab} 
  Compressors between groups are independent, i.e. $\cC_i$ and $\cC_j$ are independent, 
  for all $i \in \mathcal{G}_k, j \in \mathcal{G}_p, k \neq p.$
  And Assumption~\ref{ass:AB} is satisfied with constants 
  $A_{\mathcal{G}_k}$ and $B_{\mathcal{G}_k}$ inside each group $\mathcal{G}_k$, for $k \in \{1, \cdots, g\}.$
\end{assumption}

Now, we prove group AB inequality.

\begin{lemma}[Group AB inequality]
  Let us assume that Assumptions~\ref{eq:unbiased_compressors} and \ref{ass:group_ab} hold, then
  \begin{multline}
    \label{eq:group_ab_inequality}
    \Exp{\norm{\frac{1}{n}\sum\limits_{i=1}^n\cC_i(a_i) - \frac{1}{n}\sum\limits_{i=1}^n a_i}^2} \\
    \leq \sum\limits_{k=1}^g\frac{A_{\mathcal{G}_k}|\mathcal{G}_k|^2}{n^2} \frac{1}{|\mathcal{G}_k|}\sum\limits_{i \in \mathcal{G}_k} \norm{a_i}^2 - \sum\limits_{k=1}^g \frac{B_{\mathcal{G}_k} |\mathcal{G}_k|^2}{n^2} \norm{ \frac{1}{|\mathcal{G}_k|}\sum\limits_{i \in \mathcal{G}_k} a_i }^2.
  \end{multline}
\end{lemma}

\begin{proof}
  \begin{align*}
    &\Exp{\norm{\frac{1}{n}\sum\limits_{i=1}^n\cC_i(a_i) - \frac{1}{n}\sum\limits_{i=1}^n a_i}^2} \\
    &= \frac{1}{n^2}\sum\limits_{k=1}^g\Exp{\norm{\sum\limits_{i \in \mathcal{G}_k}\cC_i(a_i) - \sum\limits_{i \in \mathcal{G}_k} a_i}^2} \\
    &\quad + \frac{1}{n^2}\sum\limits_{k \neq p}\Exp{\inp{\sum\limits_{i \in \mathcal{G}_k}\cC_i(a_i) - \sum\limits_{i \in \mathcal{G}_k} a_i}{\sum\limits_{i \in \mathcal{G}_p}\cC_i(a_i) - \sum\limits_{i \in \mathcal{G}_p} a_i}}.
  \end{align*}
  Due to independence and unbiasedness, the last term vanishes, and, using AB inequality, we get
  \begin{align*}
    &\Exp{\norm{\frac{1}{n}\sum\limits_{i=1}^n\cC_i(a_i) - \frac{1}{n}\sum\limits_{i=1}^n a_i}^2} \\
    &= \frac{1}{n^2}\sum\limits_{k=1}^g\Exp{\norm{\sum\limits_{i \in \mathcal{G}_k}\cC_i(a_i) - \sum\limits_{i \in \mathcal{G}_k} a_i}^2} \\
    &= \sum\limits_{k=1}^g\frac{|\mathcal{G}_k|^2}{n^2}\Exp{\norm{\frac{1}{|\mathcal{G}_k|}\sum\limits_{i \in \mathcal{G}_k}\cC_i(a_i) - \frac{1}{|\mathcal{G}_k|}\sum\limits_{i \in \mathcal{G}_k} a_i}^2} \\
    &\leq \sum\limits_{k=1}^g\frac{|\mathcal{G}_k|^2}{n^2} \left(A_{\mathcal{G}_k} \left(\frac{1}{|\mathcal{G}_k|}\sum\limits_{i \in \mathcal{G}_k} \norm{a_i}^2\right) - B_{\mathcal{G}_k} \norm{ \frac{1}{|\mathcal{G}_k|}\sum\limits_{i \in \mathcal{G}_k} a_i }^2\right).
  \end{align*}
  From this we can get the result.
\end{proof}

Next, we prove analogous lemma to Lemma~\ref{lemma:upper_bound_variance}.

\newcommand*{\groupconstant}{\left(\sum\limits_{k=1}^g\frac{\left(A_{\mathcal{G}_k} - B_{\mathcal{G}_k}\right)|\mathcal{G}_k|^2}{n^2} 
\left(L_{+}^{\mathcal{G}_k}\right)^2 + \sum\limits_{k=1}^g \frac{B_{\mathcal{G}_k} |\mathcal{G}_k|^2}{n^2} \left(L_{\pm}^{\mathcal{G}_k}\right)^2\right)}

% \peter{I commented the above temporarily since I get an error message during compilation.}
\begin{lemma}
  \label{lemma:group_upper_bound_variance}
  Let us consider $g^{t + 1}$ from Line \ref{alg:gradient_estimate_definition} of Algorithm~\ref{alg:marina} and 
  assume, that Assumptions~\ref{eq:unbiased_compressors} and \ref{ass:group_ab} hold. Moreover, if Assumption~\ref{as:L_+} holds for every group $\mathcal{G}_k,$ for $k \in \{1, \cdots, g\},$ then
  \begin{align}
    &\ExpCond{\norm{g^{t+1}-\nabla f(x^{t+1})}^2}{x^{t+1}} \notag\\ 
    &\leq (1 - p)\groupconstant\norm{x^{t+1} - x^{t}}^2 \notag\\
    &\quad + (1 - p)\norm{g^t - \nabla f(x^{t})}^2.
  \end{align}
\end{lemma}

\begin{proof}
  In the view of definition of $g^{t + 1}$, we get
  \begin{align*}
    &\ExpCond{\norm{g^{t+1}-\nabla f(x^{t+1})}^2}{x^{t+1}} \\
    &= (1-p) \ExpCond{\norm{g^t + \frac{1}{n}\sum\limits_{i=1}^n \cC_i\left(\nabla f_{i}(x^{t+1}) - \nabla f_{i}(x^t)\right) - \nabla f(x^{t+1})}^2}{x^{t+1}}\\
    &= (1-p)\ExpCond{\norm{\frac{1}{n}\sum\limits_{i=1}^n \cC_i\left(\nabla f_{i}(x^{t+1}) - \nabla f_{i}(x^t)\right) - \nabla f(x^{t+1}) + \nabla f(x^t)}^2}{x^{t+1}}\\
    &\quad + (1-p)\norm{g^t - \nabla f(x^t)}^2.\\
  \end{align*}
  In the last inequality we used unbiasedness of $\cC_i.$ Using (\ref{eq:group_ab_inequality}), we get
  \begin{align*}
    &\ExpCond{\norm{g^{t+1}-\nabla f(x^{t+1})}^2}{x^{t+1}} \\
    &\leq (1-p)\ExpCond{\norm{\frac{1}{n}\sum\limits_{i=1}^n \cC_i\left(\nabla f_{i}(x^{t+1}) - \nabla f_{i}(x^t)\right) - \nabla f(x^{t+1}) + \nabla f(x^t)}^2}{x^{t+1}}\\
    &\quad + (1-p)\norm{g^t - \nabla f(x^t)}^2.\\
    &\leq (1 - p)\Bigg(\sum\limits_{k=1}^g\frac{A_{\mathcal{G}_k}|\mathcal{G}_k|^2}{n^2} \frac{1}{|\mathcal{G}_k|}\sum\limits_{i \in \mathcal{G}_k} \norm{\nabla f_{i}(x^{t+1}) - \nabla f_{i}(x^t)}^2 \\
    &\quad - \sum\limits_{k=1}^g \frac{B_{\mathcal{G}_k} |\mathcal{G}_k|^2}{n^2} \norm{ \frac{1}{|\mathcal{G}_k|}\sum\limits_{i \in \mathcal{G}_k} \nabla f_{i}(x^{t+1}) - \nabla f_{i}(x^t) }^2\Bigg)\\
    &\quad + (1 - p)\norm{g^t - \nabla f(x^{t})}^2\\
    &= (1 - p)\Bigg(\sum\limits_{k=1}^g\frac{\left(A_{\mathcal{G}_k} - B_{\mathcal{G}_k}\right)|\mathcal{G}_k|^2}{n^2} 
    \cL_{+}^{\mathcal{G}_k}(x^{t+1}, x^{t}) + \sum\limits_{k=1}^g \frac{B_{\mathcal{G}_k} |\mathcal{G}_k|^2}{n^2} \cL_{\pm}^{\mathcal{G}_k}(x^{t+1}, x^{t})\Bigg)\\
    &\quad + (1 - p)\norm{g^t - \nabla f(x^{t})}^2\\
    &\leq (1 - p)\Bigg(\sum\limits_{k=1}^g\frac{\left(A_{\mathcal{G}_k} - B_{\mathcal{G}_k}\right)|\mathcal{G}_k|^2}{n^2} 
    \left(L_{+}^{\mathcal{G}_k}\right)^2 + \sum\limits_{k=1}^g \frac{B_{\mathcal{G}_k} |\mathcal{G}_k|^2}{n^2} \left(L_{\pm}^{\mathcal{G}_k}\right)^2\Bigg)\norm{x^{t+1} - x^{t}}^2\\
    &\quad + (1 - p)\norm{g^t - \nabla f(x^{t})}^2.\\
  \end{align*}
\end{proof}

Let us define 
$$\widehat{L}^2_{\mathcal{G}} \eqdef \groupconstant.$$

\begin{theorem}
  \label{theorem:AB_group}
  Let Assumptions~\ref{ass:diff}, \ref{as:L_+}, \ref{eq:unbiased_compressors}  and
\ref{ass:group_ab} be satisfied. Let the stepsize in \algname{MARINA} be chosen as
  \begin{align*}
    \gamma \leq \left(L_- + \sqrt{\frac{1-p}{p}\widehat{L}^2_{\mathcal{G}}}\right)^{-1},
  \end{align*}
  then after $T$ iterations, \algname{MARINA} finds point $\hat{x}^T$ for which
  $
  \Exp{\norm{ \nabla f(\hat x^T)}^2} \leq \frac{2 \Delta^0}{\gamma T}.
  $
\end{theorem}

\begin{theorem}
  Let Assumptions~\ref{ass:diff}, \ref{as:L_+}, \ref{eq:unbiased_compressors}, 
  \ref{ass:pl_condition} and \ref{ass:group_ab} be satisfied and
  \begin{align*}
    \gamma \leq \min\left\{\left(L_- + \sqrt{\frac{2\left(1-p\right)}{p}\widehat{L}^2_{\mathcal{G}}}\right)^{-1}, \frac{p}{2\mu}\right\},
  \end{align*}
  then for ${x}^T$ from \algname{MARINA} algorithm the following inequality holds:
  $$
  \Exp{f(x^T) - f^{\star}} \leq \left(1 - \gamma \mu\right)^T \Delta^0.
  $$
\end{theorem}

We omit proofs of this theorems as they repeat proofs from Appendix~\ref{sec:proof_theorem_ab} and \ref{sec:proof_theorem_ab_pl};
the only difference is that we have to take $\widehat{L}^2 = \widehat{L}^2_{\mathcal{G}}.$

% Now, we compare Theorem~\ref{theorem:AB_group} with Theorem~\ref{theorem:AB}.
Let us assume that $n \leq d$, all groups have equal sizes $|\mathcal{G}_k| = G$ and constants $L_{\pm}^{\mathcal{G}_k} = L_{\pm}^{G},$ for all $k \in \{1, \dots, g\},$ and in each group we use Perm$K$ compressor from Definition~\ref{def:PermK-1}, thus communication complexity predicted by our theory is the following:
\begin{align*}
  % \label{eq:complexity_bound:perm_k_with_parameters_groups}
  N_{\textnormal{Perm$K$}}^G(p) \eqdef \frac{\Delta_0}{\varepsilon}\left(pd + (1 - p)\frac{d}{G}\right)\left(L_- + \sqrt{\frac{(1-p)G}{pn}}L_{\pm}^{G}\right).
\end{align*}

Using the same reasoning as in Lemma~\ref{lemma:optimal_parameters_of_methods}, we can take $p = 1$ or $p = 1 / G$ to get that
\begin{align}
  \label{eq:optimal_p_in_groups_complexity}
  \inf_{p \in (0, 1]} N_{\textnormal{Perm$K$}}^G(p) = \cO\left(\frac{2\Delta_0}{\varepsilon}\min\left\{d L_-, \frac{d}{G}L_- + \frac{d}{\sqrt{n}}L_{\pm}^{G}\right\}\right).
\end{align}

For the case when we have one group, we restore the communication complexity
from Lemma~\ref{lemma:optimal_parameters_of_methods}.

Comparing (\ref{eq:appendix:complexity_bounds:perm_k}) 
with (\ref{eq:optimal_p_in_groups_complexity}), we see that $dL_-/n$ from (\ref{eq:appendix:complexity_bounds:perm_k})
is always better than $dL_- / G$ from (\ref{eq:optimal_p_in_groups_complexity}); 
however; if $dL_{\pm} / \sqrt{n}$ is a bottleneck 
% ($n / G$ bigger than $Ld / n$) 
and $L_{\pm}^{G}$ is small, 
then communication complexity (\ref{eq:optimal_p_in_groups_complexity}) can be better.

Let us consider an example of a quadratic optimization task with two groups, wherein one group, 
all matrices are equal to $\mA$, 
and in another one, all matrices are equal to $\mB$, $\mA \neq \mB$, $\mA = \mA^\top \succcurlyeq 0$ and $\mB = \mB^\top \succcurlyeq 0,$
then $G = n / 2$, $L_{\pm}^{G} = 0,$ and $L_{\pm} > 0$ (see Example~\ref{ex:quadratic_functions}).
Hence, we get that
\begin{align*}
  \inf_{p \in (0, 1]} N_{\textnormal{Perm$K$}}^G(p) = \cO\left(\frac{\Delta_0d}{\varepsilon n}L_-\right).
\end{align*}
This bound is better than (\ref{eq:appendix:complexity_bounds:perm_k}) by at least the factor $1 + \nicefrac{\sqrt{n}L_{\pm}}{L_-}$.

\end{document}